\documentclass[10pt,a4paper]{article}
\usepackage[utf8]{inputenc}
\usepackage{amsmath}
\usepackage{amsthm}
\usepackage{amsfonts}
\usepackage{amssymb}
\usepackage[left=2cm,right=2cm,top=2cm,bottom=2cm]{geometry}
\usepackage{amsmath,amsfonts,amssymb,xspace,bm}

\usepackage{verbatim,dsfont,mathtools,algorithm,algorithmic, url, color}
\usepackage{booktabs}
\newcommand{\ra}[1]{\renewcommand{\arraystretch}{#1}}
\usepackage{multirow}
\usepackage{xfrac}
\usepackage{wrapfig}
\usepackage{varwidth}

\usepackage{color}
\usepackage{epstopdf}
\usepackage{appendix}
\usepackage{epstopdf}
\usepackage{enumitem}
\usepackage{pifont}

\usepackage{authblk}

\newtheorem{theorem}{Theorem}[section]
\newtheorem{lemma}[theorem]{Lemma}
\newtheorem{corollary}[theorem]{Corollary}
\newtheorem{definition}[theorem]{Definition}
\newtheorem{remark}{Remark}

\newcommand{\algo}{\textsc{Fgd}\xspace}
\newcommand{\R}{\mathbb{R}}

\DeclareMathOperator{\trace}{Tr}
\DeclareMathOperator*{\argmin}{argmin}

\newcommand{\ip}[2]{\left\langle #1, #2 \right\rangle}

\newcommand{\norm}[1]{\left \Vert #1\right \Vert}
\newcommand{\dist}{{\rm{\textsc{Dist}}}}

\newcommand{\plus}[1]{{#1}^{\scalebox{0.7}{+}}}

\newcommand{\X}{X}
\newcommand{\Xp}{\plus{X}}
\newcommand{\Xo}{X^{\star}}

\newcommand{\U}{U}

\newcommand{\f}{f}

\newcommand{\gradf}{\nabla \f}

\newcommand{\Uo}{U^{\star}}

\newcommand{\uut}{\U \U^{\top}}

\newcommand{\Q}{Q}

\newcommand{\condx}{\tau(\X)}

\newcommand{\rscm}{m}

\newcommand{\Uor}{U^\star R}
\newcommand{\Uorr}{U_r^\star R_U^\star}
\newcommand{\Xor}{X^{\star}_r}
\newcommand{\weta}{\widehat{\eta}}

\def\Up{\plus{U}}

\def\U{U}
\def\V{V}
\def\Q{Q}
\def\Y{Y}
\def\S{\mathbb{S}}

\def\linmap{\mathcal{A}}

\newcommand{\cdiff}{\delta}
\newcommand{\cdiffp}{\plus{\cdiff}}

\title{Dropping Convexity for Faster Semi-definite Optimization}

\author{Srinadh Bhojanapalli\thanks{Toyota Technological Institute at Chicago}, Anastasios Kyrillidis$^\dagger$, Sujay Sanghavi\thanks{The University of Texas at Austin. \\ \quad      Email: srinadh@ttic.edu, anastasios@utexas.edu, sanghavi@mail.utexas.edu.}}
\date{}

\begin{document}

\maketitle
\begin{abstract}
We study the minimization of a convex function $f(X)$ over the set of $n\times n$ positive semi-definite matrices, 
but when the problem is recast as $\min_U g(U) :=  f(UU^\top)$, with $U \in \R^{n \times r}$ and $r\leq n$. 
We study the performance of gradient descent on $g$---which we refer to as Factored Gradient Descent (\algo)---under standard assumptions on the {\em original} function $f$.

We provide a rule for selecting the step size and, with this choice, show that the \emph{local} convergence rate of \algo mirrors that of standard gradient descent on the original $f$: \emph{i.e.}, after $k$ steps, the error is $O(1/k)$ for smooth $f$, and exponentially small in $k$ when $f$ is (restricted) strongly convex. 
In addition, we provide a procedure to initialize \algo for (restricted) strongly convex objectives and when one only has access to $f$ via a first-order oracle; for several problem instances, such proper initialization leads to \emph{global} convergence guarantees.

\algo and similar procedures are widely used in practice for problems that can be posed as matrix factorization.
To the best of our knowledge, this is the first paper to provide precise convergence rate guarantees for general convex functions under standard convex assumptions.
\end{abstract}


\section{Introduction}
Consider the following standard convex semi-definite optimization problem: 
\begin{equation}
\begin{aligned}
\underset{\X \in \R^{n \times n}}{\text{minimize}}
& & f(\X) \quad \text{subject to}
& & \X \succeq 0,
\end{aligned} \label{intro:eq_00}
\end{equation} 
where $f: \R^{n \times n} \rightarrow \R$ is a convex and differentiable function, and $\X \succeq 0$ denotes the convex set over positive semi-definite matrices in $\R^{n \times n}$. 
Let $\Xo$ be an optimum of \eqref{intro:eq_00} with $\text{rank}(\Xo) = r^\star \leq n$.
This problem can be remodeled as a non-convex problem, by writing $X = UU^\top$ where $U$ is an $n\times r$ matrix.
Specifically, define $g(U) := f(UU^\top)$ and\footnote{While $g$ is a non-convex function, we note that it is a very specific kind of non-convexity, arising ``only" due to the recasting of an originally convex function.}
consider direct optimization of the transformed problem, \textit{i.e.},
\begin{equation}
\begin{aligned}
 \underset{\U \in \R^{n \times r}}{\text{minimize}}
& & g(U) \quad \text{where $r\leq n$}.
\end{aligned} \label{intro:eq_01}
\end{equation} 
Problems (\ref{intro:eq_00}) and (\ref{intro:eq_01}) will have the same optimum when $r=r^\star$. 
However, the recast problem is \textit{unconstrained} and leads to computational gains in practice: 
\textit{e.g.}, iterative update schemes, like gradient descent, do not need to do eigen-decompositions to satisfy semi-definite constraints at every iteration.

In this paper, we also consider the case of $r < r^\star$, which often occurs in applications. The reasons of such a choice chould be three-fold: 
$(i)$ it might model better an underlying task (\textit{e.g.}, $f$ may have arisen from a relaxation of a rank constraint in the first place), 
$(ii)$ it leads to computational gains, since smaller $r$ means fewer variables to maintain and optimize, 
$(iii)$ it leads to statistical ``gains'', as it might prevent over-fitting in machine learning or inference problems.

Such recasting of matrix optimization problems is empirically widely popular, especially as the size of problem instances increases.
Some applications in modern machine learning includes matrix completion \cite{candes2009exact, jain2013low, kyrillidis2014matrix, chen2014coherent}, affine rank minimization \cite{recht2010guaranteed, jain2010guaranteed, becker2013randomized},  covariance / inverse covariance selection \cite{hsieh2011sparse, kyrillidis2014scalable}, phase retrieval \cite{netrapalli2013phase, candes2015phase, white2015local, sun2016geometric}, Euclidean distance matrix completion \cite{mishra2011low}, finding the square root of a PSD matrix \cite{jain2015computing}, and sparse PCA \cite{d2007direct}, just to name a few. 
Typically, one can solve \eqref{intro:eq_01} via simple, first-order methods on $\U$ like gradient descent. 
Unfortunately, such procedures have no guarantees on convergence to the optima of the original $f$, or on the rate thereof. 
Our goal in this paper is to provide such analytical guarantees, by using---simply and transparently---standard convexity properties of the original $f$.


\medskip
\noindent \textbf{Overview of our results.} 
In this paper, we prove that updating $U$ via gradient descent in \eqref{intro:eq_01} converges (fast) to optimal (or near-optimal) solutions. 
While there are some recent and very interesting works that consider using such non-convex parametrization \cite{jain2013low, netrapalli2013phase, tu2015low, zheng2015convergent, sun2014guaranteed, zhao2015nonconvex}, 
their results only apply to specific examples. 
To the best of our knowledge, this is the first paper that solves the re-parametrized problem with attractive convergence rate guarantees for \emph{general convex functions $f$} and under common convex assumptions. 
Moreover, we achieve the above by assuming the {\em first order oracle} model: for any matrix $X$, we can only obtain the value $f(X)$ and the gradient $\gradf (X)$.

To achieve the desiderata, we study how gradient descent over $U$ performs in solving  (\ref{intro:eq_01}). 
This leads to the \textit{factored gradient descent} (\algo) algorithm, which applies the simple update rule 
\begin{align}\label{eq:main_recursion}
\Up = \U - \eta \gradf(\U\U^T)  \cdot \U.
\end{align}
We provide a set of sufficient conditions to guarantee convergence. 
We show that given a suitable initialization point, \algo converges to a solution close to the optimal point in sublinear or linear rate, depending on the nature of $f$. 

Our contributions in this work can be summarized as follows: 
\begin{itemize}[leftmargin=0.4cm]
\item [$(i)$] \emph{New step size rule and \algo.} 
Our main algorithmic contribution is a special choice of the step size $\eta$.
Our analysis showcase that $\eta$ needs to depend not only on the convexity parameters of $f$ (as is the case in standard convex optimization) but also on the top singular value of the unknown optimum. 
Section \ref{sec:algo} describes the precise step size rule, and also the intuition behind it. 
Of course, the optimum is not known a priori. 
As a solution in practice, we show that choosing $\eta$ based on a point that is constant relative distance from the optimum also provably works. 

\item [$(ii)$]  \emph{Convergence of \algo under common convex assumptions.} 
We consider two cases: $(i)$ when $f$ is just a $M$-smooth convex function, and 
$(ii)$ when $f$ satisfies also {\em restricted strong convexity (RSC)}, \textit{i.e.}, $f$ satisfies strong-convexity-like conditions, but only over low rank matrices; see next section for definitions. 
Both cases are based on now-standard notions, common for the analysis of convex optimization algorithms. 
Given a good initial point, we show that, when $f$ is $M$-smooth, \algo converges sublinearly to an optimal point $\Xo$.
For the case where $f$ has RSC, \algo converges linearly to the unique $\Xo$, matching analogous result for classic gradient descent schemes, under smoothness and strong convexity assumptions. 

Furthermore, for the case of smooth and strongly convex $f$, our analysis extends to the case $r < r^\star$, where \algo converges to a point  close to the best rank-$r$ approximation of $\Xo$.\footnote{In this case, we require $\|\Xo - \Xo_r\|_F$ to be small enough, such that the rank-constrained optimum be close to the best rank-$r$ approximation of $\Xo$. 
This assumption naturally applies in applications, where \emph{e.g.}, $\X^\star$ is a superposition of a low rank latent matrix, plus a small perturbation term \cite{javanmard2013localization, yu2014large}.
In Section \ref{sec:no_tail}, we show how this assumption can be dropped by using a different step size $\eta$, where spectral norm computation of two $n \times r$ matrices is required per iteration.}

Both results hold when \algo is initialized at a point with constant relative distance from optimum. 
Interestingly, the linear convergence rate factor depends not only on the convexity parameters of $f$, but also on the spectral characteristics of the optimum; a phenomenon borne out in our experiments. 
Section \ref{sec:main} formally states these results.

\item [$(iii)$]  \emph{Initialization:} 
For specific problem settings, various initialization schemes are possible (see \cite{jain2013low, netrapalli2013phase, chen2015fast}). 
In this paper, we extend such results to the case where we only have access to $f$ via the first-order oracle: specifically, we initialize based on the gradient at zero, \textit{i.e.}, $\gradf(0)$. 
We show that, for certain condition numbers of $f$, this yields a constant relative error initialization (Section \ref{sec:init}).
Moreover, Section \ref{sec:init} lists alternative procedures that lead to good initialization points and comply with our theory.
\end{itemize}

\paragraph{Roadmap.} The rest of the paper is organized as follows. Section \ref{sec:prelim} contains basic notation and standard convex definitions. Section \ref{sec:algo} presents the \algo algorithm and the step size $\eta$ used, along with some intuition for its selection. Section \ref{sec:main} contains the convergence guarantees of \algo; the main supporting lemmas and proofs of the main theorems are provided in Section \ref{sec:proofs}. In Section \ref{sec:init}, we discuss some initialization procedures that guarantee a ``decent'' starting point for \algo. This paper concludes with discussion on related work (Section \ref{sec:related}).

\section{Preliminaries}\label{sec:prelim}

\noindent {\bf Notation.} For matrices $\X, \Y \in \R^{n \times n}$, their inner product is $\ip{\X}{\Y} = \trace\left(\X^\top \Y \right)$. Also, $\X \succeq 0$ denotes $\X$ is a positive semi-definite (PSD) matrix, while the convex set of PSD matrices is denoted $\S_{+}^n$. We use $\norm{\X}_F$ and $\norm{X}_2$ for the Frobenius and spectral norms of a matrix, respectively. Given a matrix $\X$, we use $\sigma_{\min}\left(\X\right)$ and $\sigma_{\max}\left(\X\right)$ to denote the smallest and largest \emph{strictly positive} singular values of $\X$ and define $\condx = \frac{\sigma_{\max}\left(\X\right)}{\sigma_{\min}\left(\X\right)}$; with a slight abuse of notation, we also use $\sigma_1\left(\X\right) \equiv \sigma_{\max}\left(\X\right) \equiv \norm{\X}_2$. $\X_r$ denotes the rank-$r$ approximation of $\X$ via its truncated singular value decomposition. Let $\tau(\Xor) =\frac{\sigma_{1}(\Xo)}{\sigma_{r}(\Xo)}$ denote the condition number of $\Xo_r$; again, observe $\sigma_r\left(\X_r\right) \equiv \sigma_{\min}\left(\X_r\right)$. $Q_A$ denotes the basis of the column space of matrix $A$. $\text{\texttt{srank}}\left(\X\right) := \sfrac{\norm{\X}_F^2}{\norm{\X}_2^2}$ represents the stable rank of matrix $\X$. We use $e_i \in \R^n$ to denote the standard basis vector with 1 at the $i$-th position and zeros elsewhere.

Without loss of generality, $f$ is a symmetric convex function, \emph{i.e.}, $f(\X) = f(\X^\top)$. Let $\gradf(X)$ denote the gradient matrix, \emph{i.e.}, its $(i,j)^{th}$ element is $[\gradf(X)]_{ij} = \frac{\partial f(\X)}{\partial x_{ij}}$. For $\X = \U\U^\top$, the gradient of $f$ with respect to $\U$ is $\left(\gradf(\U \U^\top) + \gradf(\U \U^\top)^\top\right)\U = 2\gradf(\X) \cdot \U$, due to symmetry of $f$. Finally, let $\Xo$ be the optimum of $f(\X)$ over $\S_{+}^n$ with factorization $ \Xo = \Uo (\Uo)^T$. 

For any general symmetric matrix $\X$, let the matrix $\mathcal{P}_+(\X)$ be its projection onto the set of PSD matrices. This can be done by finding all the strictly positive eigenvalues and corresponding eigenvectors $(\lambda_i,v_i:\lambda_i>0)$ and then forming $\mathcal{P}_+(\X) = \sum_{i:\lambda_i>0} \lambda_i v_i v_i^\top$. 

In algorithmic descriptions, $\U$ and $\Up$ denote the putative solution of current and next iteration, respectively. 
An important issue in optimizing $f$ over the $\U$ space is the existence of non-unique possible factorizations $\U\U^\top$ for any feasible point $\X$. To see this, given factorization $\X = \U\U^\top$ where $\U \in \R^{n \times r}$, one can define an class of \emph{equivalent} factorizations $\U R^\top R\U^\top = \U\U^\top$, where $R$ belongs to the set $\{ R \in \R^{r \times r} ~:~ R^\top R = I\}$ of rotational matrices. 
So we use a rotation invariant distance metric in the factored space that is equivalent to distance in the matrix $X$ space, which is defined below.
\begin{definition}{\label{prelim:def_04}}
Let matrices $\U, \V \in \R^{n \times r}$. Define:
\begin{align*}
\dist\left(U, V\right) :=\min_{R: R \in \mathcal{O}} \norm{U - V R}_F.
\end{align*}
$\mathcal{O}$ is the set of $r \times r$ orthonormal matrices $R$, such that $R^\top R = I_{r\times r}$. The optimal $R$ satisfies $PQ^\top$ where $P\Sigma \Q^\top$ is the singular value decomposition of $\V^\top \U$.
\end{definition}

\medskip
\noindent {\bf Assumptions.} We will investigate the performance of non-convex gradient descent for functions $f$ that satisfy standard smoothness conditions only, as well as the case where $f$ further is (restricted) strongly convex. We state these standard definitions below.

\begin{definition}{\label{prelim:def_00}}
Let $f: \S_{+}^n \rightarrow \R$ be convex and differentiable. Then, $f$ is $m$-strongly convex if: 
\begin{equation}\label{eq:sc}
f(\Y) \geq f(\X) + \ip{\gradf\left(\X\right)}{\Y - \X} + \tfrac{m}{2} \norm{Y - \X}_F^2, \quad \forall \X, \Y \in  \S_{+}^n.
\end{equation}
\end{definition}

\begin{definition}{\label{prelim:def_01}}
Let $f: \S_{+}^n \rightarrow \R$ be a convex differentiable function. Then, $f$ is $M$-smooth if: 
\begin{equation}
\norm{\gradf\left(\X\right) - \gradf\left(\Y\right)}_F \leq M \cdot \norm{\X - \Y}_F, \quad \X, \Y \in  \S_{+}^n.
\end{equation} This further implies the following upper bound:
\begin{equation}\label{eq:lip}
f(\Y) \leq \f(X) + \ip{\gradf\left(\X\right)}{\Y - \X} + \tfrac{M}{2} \norm{\Y - \X}_F^2.
\end{equation}
\end{definition} Given the above definitions, we define $\kappa = \frac{M}{m}$ as the condition number of function $f$.

Finally, in high dimensional settings, often loss function $f$ does not satisfy strong convexity globally, but only on a restricted set of directions; see \cite{negahban2012restricted, agarwal2010fast} and Section~\ref{sec:sensing} for a more detailed discussion.

\begin{definition}{\label{prelim:def_02}}
A convex function $\f$ is $(\rscm, r)$-restricted strongly convex if:
\begin{equation}\label{eq:rsc}
\f(Y) \geq \f(X) + \ip{\gradf\left(\X\right)}{Y-X} + \tfrac{\rscm}{2} \norm{Y-X}_F^2,  \quad \text{for any rank-$r$  matrices $X, Y \in \S_{+}^n$}.
\end{equation}
\end{definition}

\section{Factored gradient descent}{\label{sec:algo}}

We solve the non-convex problem (\ref{intro:eq_01}) via \emph{Factored Gradient Descent} (\algo) with update rule\footnote{The true gradient of $f$ with respect to $\U$ is $2 \gradf(\U\U^\top) \cdot \U$. However, for simplicity and clarity of exposition, in our algorithm and its theoretical guarantees, we absorb the 2-factor in the step size $\eta$.}:
\begin{align*}
\Up = \U - \eta \gradf(\U\U^\top)  \cdot \U.
\end{align*}
\algo does this, but with two key innovations: a careful initialization and a special step size $\eta$. 
The discussion on the initialization is deferred until Section \ref{sec:init}. 

\paragraph{Step size $\eta$.} Even though $f$ is a convex function over $X\succeq 0$, the fact that we operate with the non-convex $\U\U^\top$ parametrization means that we need to be careful about the step size $\eta$; 
\textit{e.g.}, our \emph{constant} $\eta$ selection should be such that, when we are close to $\Xo$, we do not ``overshoot'' the optimum $\Xo$. 

In this work, we pick the step size parameter, according to the following closed-form\footnote{Constant $16$ in the expression \eqref{eq:step_size} appears due to our analysis, where we do not optimize over the constants. One can use another constant in order to be more aggressive; nevertheless, we observed that our setting works well in practice. }:
\begin{wrapfigure}{r}{0.5\textwidth} 
\begin{minipage}{0.52\textwidth}
   \vspace{-1em}
	\begin{algorithm}[H]\label{algo:altgrad}
		\caption{Factored gradient descent (\algo)}
		\begin{algorithmic}[1]		
			\INPUT Function $f$, target rank $r$, \# iterations $K$. 
			\STATE Compute $\X^0$ as in \eqref{eq:X0}.
			\STATE Set $\U \in \R^{n \times r}$ such that $\X^0 = \U\U^\top$.
			\STATE Set step size $\eta$ as in \eqref{eq:step_size}.
			\FOR {$k=0$ to $K-1$}
				\STATE $\Up = \U - \eta \gradf(\U\U^T)  \cdot \U$.
		    		\STATE $\U= \Up$.
			\ENDFOR
			\OUTPUT $\X = \U\U^\top$. 	
		\end{algorithmic}
	\end{algorithm} \vspace{-0.5cm}
\end{minipage}
\end{wrapfigure} 
\begin{align}\label{eq:step_size}
\eta =\frac{1}{16 ( M\norm{\X^0}_2 + \norm{\gradf(X^0) }_2) }.
\end{align} 
Recall that, if we were just doing standard gradient descent on $f$, we would choose a step size of $\sfrac{1}{M}$, where $M$ is a uniform upper bound on  the largest eigenvalue of the Hessian $\nabla^2 f(\cdot)$.

To motivate our step size selection, let us consider a simple setting where $U \in \R^{n \times r}$ with $r = 1$; 
\textit{i.e.}, $U$ is a vector. 
For clarity, denote it as $u$.
Let $f$ be a separable function with $f(\X) = \sum_{ij} f_{ij}(X_{ij})$. 
Furthermore, define the function $g: \R^n \rightarrow \R$ such that $f(uu^\top) \equiv g(u)$. 
It is easy to compute (see Lemma~\ref{lem:hessian}):
\begin{small}
\begin{align*}
\nabla g(u) = \gradf(uu^\top)  \cdot u \in \R^{n} \quad \text{and} \quad \nabla^2 g(u) = \text{\texttt{mat}}\left( \texttt{diag}(\nabla^2 f(uu^\top)) \cdot \text{\texttt{vec}}\left(uu^\top\right)\right) + \gradf(uu^\top) \in \R^{n \times n},
\end{align*}
\end{small} 
where $\text{\texttt{mat}}: \R^{n^2} \rightarrow \R^{n \times n}, ~\text{\texttt{vec}}: \R^{n \times n} \rightarrow \R^{n^2}$ and, $\text{\texttt{diag}}: \R^{n^2 \times n^2} \rightarrow \R^{n^2 \times n^2}$ are the matricization, vectorization and diagonalization operations, respectively; 
for the last case, $\text{\texttt{diag}}$ generates a diagonal matrix from the input, discarding its off-diagonal elements. 
We remind that $\nabla f(uu^\top) \in \R^{n \times n}$ and $\nabla^2 f(uu^\top) \in \R^{n^2 \times n^2}$. 
Note also that $\nabla^2 f(\X)$ is diagonal for separable $f$. 

Standard convex optimization suggests that $\eta$ should be chosen such that $\eta < \sfrac{1}{\|\nabla^2 g(\cdot)\|_2}$. 
The above suggest the following step size selection rule for $M$-smooth $f$:
\begin{align*}
\eta < \tfrac{1}{\|\nabla^2 g(\cdot)\|_2} \propto \tfrac{1}{M\norm{X}_2+\|\gradf(\X)\|_2}.
\end{align*} 
In stark contrast with classic convex optimization where $\eta \propto \tfrac{1}{M}$, the step size selection further depends on the spectral information of the current iterate and the gradient.  
Since computing $\norm{X}_2, \|\gradf(\X)\|_2$ per iteration could be computational inefficient, we use the spectral norm of $\X^0$ and its gradient $\gradf(X^0)$ as surrogate, where $X^0$ is the initialization point\footnote{However, as we show in Section \ref{sec:no_tail}, one could compute $\|\X\|_2$ and $\|\gradf(X)\|_2$ per iteration in order to relax some of the requirements of our approach.}.

To clarify $\eta$ selection further, we next describe a toy example, in order to illustrate the necessity of such a scaling of the step size.
Consider the following minimization problem. 
\begin{equation*}
\begin{aligned}
 \underset{u \in \R^{n \times 1}}{\text{minimize}}
& & f(uu^\top) := ||uu^\top - Y||_F^2,
\end{aligned} 
\end{equation*} 
where $u \equiv \U \in \R^{n \times 1}$---and thus, $\X = uu^\top$, \textit{i.e.}, we are interested in rank-1 solutions---and $Y$ is a given rank-2 matrix such that $Y = \alpha^2 v_1 v_1^\top - \beta^2 v_2 v_2^\top$, for $\alpha > \beta \in \R$ and $v_1, v_2$ orthonormal vectors. 
Observe that $f$ is a strongly convex function with rank-1 minimizer $\Xo = \alpha^2 v_1 v_1^\top$; let $u^\star = \alpha v_1$. 
It is easy to verify that $(i)$ $\|\Xo\|_2 = \alpha^2$, $(ii)$ $\|\gradf(\Xo)||_2 = \|2\cdot \left(\Xo - Y\right)\|_2 = 2 \beta^2$, and $(iii)$ $\|\gradf(\X_1) - \gradf(\X_2)\|_F \leq M \cdot \|\X_1 - \X_2\|_F$, where $M = 2$.

Consider the case where $u = \tfrac{\alpha}{2} v_1 + \tfrac{\beta}{10} v_2$ is the current estimate. 
Then, the gradient of $f$ at $u$ is evaluated as: 
\begin{align*}
\gradf(uu^\top) \cdot u = 2\left(-\tfrac{3\alpha^2}{8}v_1v_1^\top + \tfrac{101 \beta^2}{10^3} v_2 v_2^\top \right) \cdot \left( \tfrac{\alpha}{2} v_1 + \tfrac{\beta}{10} v_2\right) = -\tfrac{3\alpha^3}{4}v_1 + \tfrac{101 \beta^3}{500} v_2. 
\end{align*} 
Hence, according to the update rule $u^{+} = u - 2\eta\gradf(uu^\top) \cdot u$, the next iterate satisfies:
\begin{align*}
u^+= u - 2\eta \left(-\tfrac{3\alpha^3}{4}v_1 + \tfrac{101 \beta^3}{500} v_2\right) = \left( \tfrac{\alpha}{2} + \eta \tfrac{3 \alpha^3}{2} \right) v_1  + \left( \tfrac{\beta}{10} + \eta \tfrac{202 \beta^3}{500} \right) v_2 .
\end{align*} 
Observe that coefficients of both $v_1, ~v_2$ in $u^{+}$ include $O(\alpha^3)$ and $O(\beta^3)$ quantities.

The quality of $u^{+}$ clearly depends on how $\eta$ is chosen. 
In the case $\eta = \tfrac{1}{M} = \tfrac{1}{2}$, such step size can result in divergence$/$``overshooting'', as $\|\Xo\|_2 = O(\alpha^2)$ and $\| \gradf(\Xo)\|_2 = O(\beta^2)$ can be arbitrarily large (independent of $M$). 
Therefore, it could be the case that $\dist(u^+, u^\star) > \dist(u, u^\star)$.

In contrast, consider the step size\footnote{For illustration purposes, we consider a step size that depends on the unknown $\X^\star$; in practice, our step size selection is a surrogate of this choice and our results automatically carry over, with appropriate scaling.} $\eta = \tfrac{1}{16(M\|\Xo\|_2 + \|\gradf(\Xo)\|_2)} \propto \tfrac{1}{C(\alpha^2 +\beta^2)}$. 
Then, with appropriate scaling $C$, we observe that $\eta$ lessens the effect of $O(\alpha^3)$ and $O(\beta^3)$ terms in $v_1$ and $v_2$ terms, that lead to overshooting for the case $\eta = \tfrac{1}{2}$. 
This most possibly result in $\dist(u^+, u^\star) \leq  \dist(u, u^\star)$.

\paragraph{Computational complexity.} 
The per iteration complexity of \algo is dominated by the gradient computation. 
This computation is required in any first order algorithm and the complexity of this operation depends on the function $f$. 
Apart from $\gradf(X)$, the additional computation required in \algo is matrix-matrix additions and multiplications, with time complexity upper bounded by $\text{\texttt{nnz}}(\gradf(\cdot)) \cdot r$, where $\text{\texttt{nnz}}(\gradf(\cdot))$ denotes the number of non zeros in the gradient at the current point.\footnote{It could also occur that gradient $\nabla f(X)$ is low-rank, or low-rank + sparse, depending on the problem at hand; it could also happen that the structure of $\nabla f(X)$ leads to ``cheap" matrix-vector calculations, when applied to vectors. Here, we state a more generic --and maybe pessimistic-- scenario where $\gradf(\X)$ is unstructured.}
Hence, the per iteration complexity of \algo is much lower than traditional convex methods like projected gradient descent~\cite{nesterov2004introductory} or classic interior point methods~\cite{nesterov1988general, nesterov1989self}, as they often require a full eigenvalue decomposition per step. 

Note that, for $r = O(n)$, \algo and projected gradient descent have same per iteration complexity of $O(n^3)$. 
However, \algo performs only a single matrix-matrix multiplication operation, which is much ``cheaper" than a SVD calculation. 
Moreover, matrix multiplication is an easier-to-parallelize operation, as opposed to eigen decomposition operation which is inherently sequential. 
We notice this behavior in practice; see Sections~\ref{sec:sensing}, \ref{sec:QST} and \ref{sec:high_rank} for applications in matrix sensing and quantum state tomography.

\section{Local convergence of {\rm \algo}}{\label{sec:main}}

In this section, we present our main theoretical results on the performance of \algo. 
We present convergence rates for the settings where 
$(i)$ $f$ is a $M$-smooth convex function, 
and $(ii)$ $f$ is a $M$-smooth \emph{and} $(m, r)$-restricted strongly convex function. 
These assumptions are now standard in convex optimization. 
Note that, since the $UU^\top$ factorization makes the problem non-convex, it is hard to guarantee convergence of gradient descent schemes in general, without any additional assumptions. 

We now state the main assumptions required by \algo for convergence: \vspace{0.3cm}

\quad \quad \quad \quad \quad \quad \quad \quad \quad \quad \quad \quad \quad \quad \algo \textsc{Assumptions}
\begin{itemize}[leftmargin=0.8cm] 
	\item {\it Initialization:} We assume that \algo is initialized with a ``good'' starting point $X^0=U^0(U^0)^\top$ that has constant relative error to $\Xo_r =\Uo_r(\Uo_r)^\top$.\footnote{If $r = r^\star$, then one can drop the subscript. For completeness and in order to accommodate the approximate rank-$r$ case, described below, we will keep the subscript in our discussion.}
	In particular, we assume
	\begin{itemize}
	\item [$(A1)$]  \quad \quad \quad $ \dist(\U^0, \Uo_r) \leq  \rho \sigma_{r}(\Uo_r)$ \quad \text{for } $\rho := \frac{1}{100} \frac{\sigma_r(\Xo)}{\sigma_1(\Xo)} $ \quad \quad~~ (\text{Smooth } $f$) \vspace{-0.1cm}
	\item [$(A2)$]   \quad \quad \quad$ \dist(\U^0, \Uo_r) \leq \rho' \sigma_{r}(\Uo_r)$ ~~~\text{for } $\rho' := \frac{1}{100 \kappa} \frac{\sigma_r(\Xo)}{\sigma_1(\Xo)}  $ \quad \quad (\text{Strongly convex }$f$),
	\end{itemize}
	for the smooth and restricted strongly convex setting, respectively. 	
	This assumption helps in avoiding saddle points, introduced by the $\U$ parametrization\footnote{To illustrate this consider the following example, \begin{equation*}
\begin{aligned}
 \underset{U \in \R^{n \times r}}{\text{minimize}}
& & f(UU^\top) := ||UU^\top - \Uo_r(\Uo_r)^\top||_F^2.
\end{aligned} 
\end{equation*}  
Now it is easy to see that $\dist(\Uo_{r-1}, \Uo_r) = \sigma_r(\Uo_r)$ and $\Uo_{r-1}$ is a stationary point of the function considered $\left( \gradf(\Uo_{r-1}(\Uo_{r-1})^\top) \cdot \Uo_{r-1} =0 \right)$. We need the initial error to be further smaller than $\sigma_r(\Uo)$ by a factor of condition number of $\Xo_r$.}.

In many applications, an initial point $U^0$ with this type of guarantees is easy to obtain, often with just one eigenvalue decomposition; we refer the reader to the works \cite{jain2013low, netrapalli2013phase, chen2015fast, zheng2015convergent, tu2015low} for specific initialization procedures for different problem settings. See also Section~\ref{sec:init} for a more detailed discussion. 
Note that the problem is still non-trivial after the initialization, as this only gives a constant error approximation. 

	\item {\it Approximate rank-$r$ optimum:} 
	In many learning applications, such as localization~\cite{javanmard2013localization} and multilabel learning~\cite{yu2014large}, the true $\X^\star$ emerges as the superposition of a low rank latent matrix plus a small perturbation term, such that $\|\X^\star - \X^\star_r\|_F$ is small.
	While, in practice, it might be the case $\text{rank}(\X^\star) = n$---due to the presence of noise---often we are more interested in revealing the latent low-rank part.
	As already mentioned, we might as well set $r < \text{rank}(\Xo) $ for computational or statistical reasons. 
	 In all these cases, further assumptions w.r.t. the quality of approximation have to be made. 
	  In particular, let $\Xo$ be the optimum of \eqref{intro:eq_00} and $f$ is $M$-smooth and $(m, r)$-strongly convex. In our analysis, we assume:
		\begin{itemize}
	\item [$(A3)$]  \quad \quad \quad \quad \quad$ \|\Xo -\Xo_r\|_F \leq  \frac{1}{200 \kappa^{1.5}} \frac{\sigma_r(\Xo)}{\sigma_1(\Xo)} \sigma_{r}(\Xo)$ \quad ~~(\text{Strongly convex } $f$), 
	\end{itemize}
	This assumption intuitively requires the noise magnitude to be smaller than the optimum and constrains the rank constrained optimum to be closer to $\Xo_r$.\footnote{Note that the assumption $(A3)$ can be dropped by using a different step size $\eta$ (see Theorem~\ref{thm:convergence_main_new} in Section \ref{sec:no_tail}). However, this requires two additional spectral norm computations per iteration.}
\end{itemize}

\medskip
We note that, in the results presented below, we have not attempted to optimize over the constants appearing in the assumptions and any intermediate steps of our analysis. Finding such tight constants could strengthen our arguments for fast convergence; however, it does not change our claims for sublinear or linear convergence rates. 
Moreover, we consider the case $r \leq \text{rank}(\Xo)$; we believe the analysis can be extended to the setting $r > \text{rank}(\Xo)$ and leave it for future work.\footnote{Experimental results on synthetic matrix sensing settings have shown that, if we overshoot $r$, \emph{i.e.}, $r > \text{rank}(\Xo)$, \algo still performs well, finding an $\varepsilon$-accurate solution with linear rate.}

\subsection{$1/k$ convergence rate for smooth $f$} \label{sec:sublin_convg}
Next, we state our first main result under smoothness condition, as in Definition \ref{prelim:def_01}. 
In particular, we prove that \algo makes progress per iteration with sublinear rate. 
Here, we assume only the case where $r = r^\star$; for consistency reasons, we denote $\Xo = \Xo_r$. 
Key lemmas and their proofs for this case are provided in Section \ref{sec:smooth_case_proof}.

\begin{theorem}[Convergence performance for smooth $f$]\label{thm:smooth_inexact}
Let $\Xo_r = \Uo_r \U_r^{\star^\top}$ denote an optimum of $M$-smooth $f$ over the PSD cone. 
Let $f(\X^0) > f(\Xor)$. 
Then, under assumption $(A1)$, after $k$ iterations, the \textsc{FGD} algorithm 
finds solution $\X^k$ such that
\begin{align}\label{eq:smooth_exact}
f(\X^k) - f(\Xo_r) \leq \frac{\tfrac{5}{\eta} \cdot \dist(\U^0, \Uo_r)^2}{k + \tfrac{5}{\eta} \cdot \tfrac{\dist(\U^0, \Uo_r)^2}{f(\X^0) - f(\Xo_r)}}.
\end{align}
\end{theorem}

The theorem states that provided $(i)$ we choose the step size $\eta$, based on a starting point that has constant relative distance to $\Uo_r$, and $(ii)$ we start from such a point, gradient descent on $\U$ will converge sublinearly to a point $\Xo_r$. 
In other words, Theorem \ref{thm:smooth_inexact} shows that \algo computes a sequence of estimates in the $U$-factor space such that the function values decrease with $O\left(\tfrac{1}{k}\right)$ rate, towards a global minimum of $f$ function. 
Recall that, even in the standard convex setting, classic gradient descent schemes over $X$ achieve the same $O\left(\tfrac{1}{k}\right)$ convergence rate for smooth convex functions \cite{nesterov2004introductory}.  
Hence, \algo matches the rate of convex gradient descent, under the assumptions of Theorem~\ref{thm:smooth_inexact}. 
The above	 are abstractly illustrated in Figure \ref{fig:smooth_illustration}. 

\begin{figure}[!ht]
	\centering
	\includegraphics[width=1\textwidth]{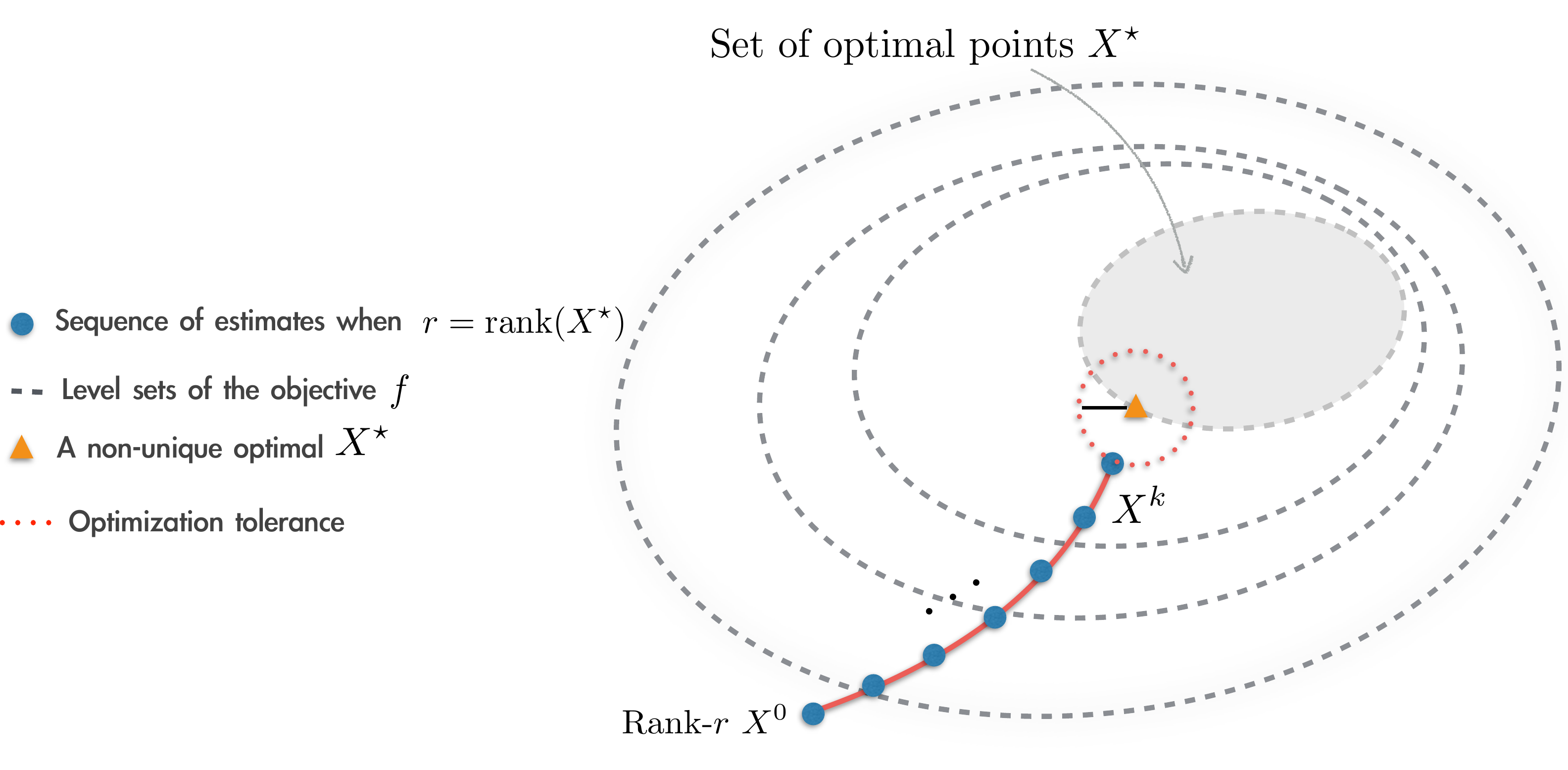}
	\caption{Abstract illustration of Theorem \ref{thm:smooth_inexact} and the behavior of \algo in the case where $f$ is \emph{just $M$-smooth}. 
	The grey-shaded area represents the set of optimum solutions $\X^\star = \Xo_r$. Let the orange triangle denote the optimum, close to which \algo converges; the dashed red circle denotes the \emph{optimization tolerance/error}. 
	} \label{fig:smooth_illustration}
\end{figure}

\subsection{Linear convergence rate under strong convexity assumption} \label{sec:lin_convg}
Here, we show that, with the additional assumption that $f$ satisfies the $(m, r)$-restricted strong convexity over $ \S_{+}^n $, \algo achieves linear convergence rate. 
The proof is provided in Section \ref{sec:strong_case_proof}.

\begin{theorem}[Convergence rate for restricted strongly convex $f$]\label{thm:convergence_main}
Let the current iterate be $\U$ and $\X = \U\U^\top$. Assume $ \dist(\U, \Uo_r) \leq \rho' \sigma_{r}(\Uo_r)$ and let the step size be $\eta =\frac{1}{16 \, (M \norm{X^0 }_2 + \norm{\gradf(X^0)}_2) }$. Then under assumptions $(A2), (A3)$, the new estimate $\Up= U -\eta \gradf(X) \cdot U$ satisfies
\begin{equation}
\dist(\Up, \Uo_r)^2 \leq \alpha \cdot \dist(\U, \Uo_r)^2+ \beta \cdot \|\Xo -\Xo_r\|_F^2, \label{conv:eq_00}
\end{equation}
where $\alpha = 1 -\frac{m \sigma_{r}(\Xo)}{64 (M\|\Xo\|_2 + \|\gradf(\Xo)\|_2)}$ and $\beta=  \frac{M}{28(M\|\Xo\|_2 + \|\gradf(\Xo)\|_2)}$. Furthermore, $\Up$ satisfies $ \dist(\Up, \Uo_r) \leq \rho' \sigma_{r}(\Uo_r). $
 \end{theorem}

The theorem states that provided $(i)$ we choose the step size based on a point that has constant relative distance to $\Uo_r$, and $(ii)$ we start from such a point, gradient descent on $\U$ will converge linearly to a neighborhood of $\Uo_r$.
The above theorem immediately implies linear convergence rate for the setting where $f$ is standard strongly convex, with parameter $m$. This follows by observing that standard strong convexity implies restricted strong convexity for all values of rank $r$.

Last, we present results for the special case where $r = r^\star$; in this case, \algo finds an optimal point $\Uo_r$ with linear rate, within the equivalent class of orthonormal matrices in $\mathcal{O}$.

\begin{corollary}[Exact recovery of $\Xo$]\label{cor:exact}
Let $\Xo$ be the optimal point of $f$, over the set of PSD matrices, such that $rank(\Xo) =r$.  
Consider $X$ as in Theorem \ref{thm:convergence_main}. 
Then, under the same assumptions and with the same convergence factor $\alpha$ as in Theorem \ref{thm:convergence_main}, we have 
\begin{equation*}
\dist(\Up, \Uo)^2 \leq \alpha \cdot \dist(\U, \Uo)^2.
\end{equation*}
 \end{corollary}

Further, for $r =n$ we recover the exact case of semi-definite optimization. In plain words, the above corollary suggests that, given an accuracy parameter $\varepsilon$, \algo requires $K = O\left(\log \left(\sfrac{1}{\varepsilon}\right)\right)$ iterations in order to achieve $\dist(\U^{K}, \Uo)^2 \leq \varepsilon$; recall the analogous result for classic gradient schemes for $M$-smooth and strongly convex functions $f$, where similar rates can be achieved in $\X$ space \cite{nesterov2004introductory}. The above are abstractly illustrated in Figure \ref{fig:stronglycvx_illustration}.

\begin{figure}[!ht]
	\centering
	\includegraphics[width=1\textwidth]{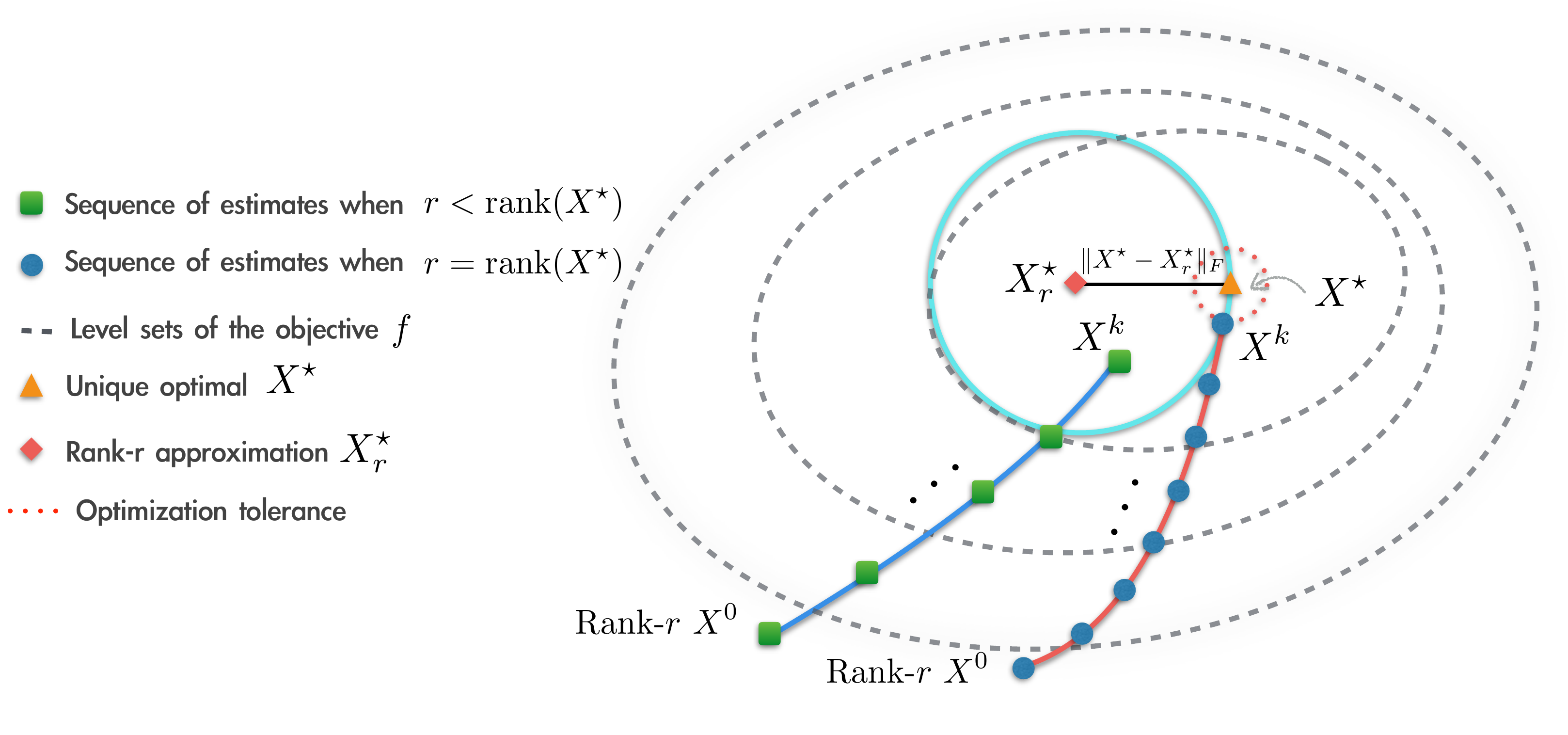}
	\caption{Abstract illustration of Theorem \ref{thm:convergence_main} and Corollary \ref{cor:exact}. The two curves denote the two cases: $(i)$ $r = \text{rank}(X^\star)$ and, $(ii)$ $r < \text{rank}(\X^\star)$. $(i)$ In the first case, the triangle marker denotes the unique optimum $\X^\star$ and the dashed red circle denotes the \emph{optimization tolerance/error}. $(ii)$ In the case where $r < \text{rank}(\X^\star)$, let the cyan circle with radius $c\|\X^\star - \X^\star_r\|_F$ (set $c = 1$ for simplicity) denote a \emph{neighborhood} around $\X^\star$. In this case, $\algo$ converges to a rank-$r$ approximation in the vicinity of $\X^\star$ in sublinear rate, according to Theorem \ref{thm:convergence_main}.} \label{fig:stronglycvx_illustration}
\end{figure}

\begin{remark}
\textit{By the results above, one can easily observe that the convergence rate factor $\alpha$, in contrast to standard convex gradient descent results, depends both on the condition number of $\Xo_r$ and $\|\gradf(\Xo)\|_2$, in addition to $\kappa$. 
This dependence is a result of the step size selection, which is different from standard step sizes, \textit{i.e.}, $\sfrac{1}{M}$ for standard gradient descent schemes. 
We also refer the reader to Section~\ref{sec:discuss} for some discussion.
}
\end{remark}

As a ramification of the above, notice that $\alpha$ depends only on the condition number of $\Xo_r$ and not that of $\Xo$. This suggests that, in settings where the optimum $\Xo$ has bad condition number (and thus leads to slower convergence), it is indeed beneficial to restrict $U$ to be a $n \times r$ matrix and only search for a rank-$r$ approximation of the optimal solution, which leads to faster convergence rate in practice; see Figure \ref{fig:03} in our experimental findings at the end of Section \ref{sec:sims}.

\begin{remark}
\textit{In the setting where the optimum $\Xo$ is 0, directly applying the above theorems requires an initialization that is exactly at the optimum 0. On the contrary, this is actually an easy setting and the \algo converges from any initial point to the optimum. 
}
\end{remark}

\section{Initialization}\label{sec:init}

In the previous section, we show that gradient descent over $U$ achieves sublinear$/$linear convergence, once the iterates are closer to $\Uo_r$. Since the overall problem is non-convex, intuition suggests that we need to start from a ``decent" initial point, in order to get provable convergence to $\Uo_r$.

One way to satisfy this condition for general convex $f$ is to use one of the standard convex algorithms and obtain $\U$ within constant error to $\Uo$ (or $\Uo_r$); then, switch to \algo to get the high precision solution. 
See \cite{tu2015low} for a specific implementation of this idea on matrix sensing.
Such initialization procedure comes with the following guarantees; the proof can be found in Section \ref{sec:init_proofs}:

\begin{lemma}\label{lem:switch}
Let $f$ be a $M$-smooth and $(m, r)$-restricted strongly convex function over PSD matrices and let $\Xo$ be the minimum of $f$ with $rank(\Xo)=r$. Let $\Xp =\mathcal{P}_{+} (X -\frac{1}{M}\gradf(X))$ be the projected gradient descent update. 
Then, $\norm{\Xp -\X}_F \leq \frac{c}{\kappa\sqrt{r} \tau(X_r)} \sigma_r(X)$ implies,
$$\dist(\U_r, \Uo_r) \leq \tfrac{c'}{ \tau(\Xo_r)} \sigma_r(\Uo_r), \quad \text{ for constants } c,~c' > 0.$$
\end{lemma}

Next, we present a generic initialization scheme for general smooth and strongly convex $f$. 
We use only the \emph{first-order oracle}: we only have access to---at most---gradient information of $f$. 
Our initialization comes with theoretical guarantees w.r.t. distance from optimum. 
Nevertheless, in order to show small relative distance in the form of $\dist(\U^0, \Uo_r) \leq \rho \sigma_r(\Uo_r)$, one requires certain condition numbers of $f$ and further assumptions on the spectrum of optimal solution $\Xo$ and rank $r$. 
However, empirical findings in Section \ref{sec:sims} show that our initialization performs well in practice.

Let $\gradf(0) \in \R^{n \times n}$. 
Since the initial point should be in the PSD cone, we further consider the projection $\mathcal{P}_+(-\gradf(0))$. 
By strong convexity and smoothness of $f$, one can observe that the point $\sfrac{1}{M} \cdot \mathcal{P}_{+}\left(- \gradf(0)\right)$ is a good initialization point, within some radius from the vicinity of $\Xo$; \textit{i.e.}, 
\begin{align}\label{eq:init_point}
\norm{\tfrac{1}{M}\mathcal{P}_+(-\gradf(0)) - \Xo}_F ~ \leq ~ 2\left(1-\tfrac{m}{M}\right) \|\Xo\|_F;
\end{align} 
see also Theorem \ref{thm:scinit}. 
Thus, a scaling of $\mathcal{P}_+(-\gradf(0))$ by $M$ could serve as a decent initialization. 
In many recent works~\cite{jain2013low, netrapalli2013phase, candes2015phase1, zheng2015convergent, chen2015fast} this initialization has been used for specific applications.\footnote{To see this, consider the case of least-squares objective $f(\X) := \tfrac{1}{2} \|\mathcal{A}(\X) - y\|_2^2$, where $y$ denote the set of observations and, $\mathcal{A}$ is a properly designed sensing mechanism, depending on the problem at hand. For example, in the affine rank minimization case \cite{zheng2015convergent, chen2015fast}, $\left(\mathcal{A}(X)\right)_i$ represents the linear system mechanism where $\trace(A_i\cdot X) = b_i$. Under this setting, computing the gradient $\nabla f(\cdot)$ at zero point, we have: $-\nabla f(0) = \mathcal{A}^*(y)$, where $\mathcal{A}^*$ is the adjoint operator of $\mathcal{A}$. Then, it is obvious that the operation $\mathcal{P}_{+}\left(-\nabla f(0)\right)$  is very similar to the spectral methods, proposed for initialization in the references above.} 
Here, we note that the point $\sfrac{1}{M} \cdot \mathcal{P}_{+}\left(- \gradf(0)\right)$ can be used as initialization point for generic smooth and strongly convex $f$. 

The smoothness parameter $M$ is not always easy to compute exactly; in such cases, one can use the surrogate $m \leq \| \gradf(0)-\gradf(e_1 e_1^\top)\|_F \leq M$. Finally, our initial point $U^0 \in \R^{n \times r}$ is a rank-$r$ matrix such that $\X^0_r = \U^0 \U^{0 \top}$. 

We now present guarantees for the initialization discussed. The proof is provided in Section~\ref{sec:init_prof}. 

\begin{theorem}[Initialization]\label{thm:scinit}
Let $\f$ be a $M$-smooth and $m$-strongly convex function, with condition number $\kappa = \frac{M}{m}$, and let $\Xo$ be its minimum over PSD matrices. 
Let $\X^0$ be defined as:
\begin{equation}
\X^0 := \tfrac{1}{\| \gradf(0)-\gradf(e_1 e_1\top)\|_F}\mathcal{P}_+ \left( -\gradf(0)\right ),
\label{eq:X0}
\end{equation} and $\X^0_r$ is its rank-$r$ approximation. Let $\norm{\Xo -\Xo_r}_F \leq \tilde{\rho} \norm{\Xo_r}_2$ for some $\tilde{\rho}$. Then, 
$\dist(U^0, \Uo_r) \leq \gamma \sigma_r(\Uo_r)$,
where $\gamma =4 \tau(\Xo_r) \sqrt{2r}  \cdot \left( \sqrt{\kappa^2 -\sfrac{2}{\kappa} +1}\left(\text{\texttt{srank}}^{\sfrac{1}{2}}\left(\Xo_r\right) +\tilde{\rho}\right) +\tilde{\rho} \right)$ and $\text{\texttt{srank}}\left(\Xo_r\right) = \frac{||\Xo_r||_F^2}{||\Xo_r||_2^2}$.
\end{theorem} 

\begin{figure}[!t]
\centering
\includegraphics[width=0.33\textwidth]{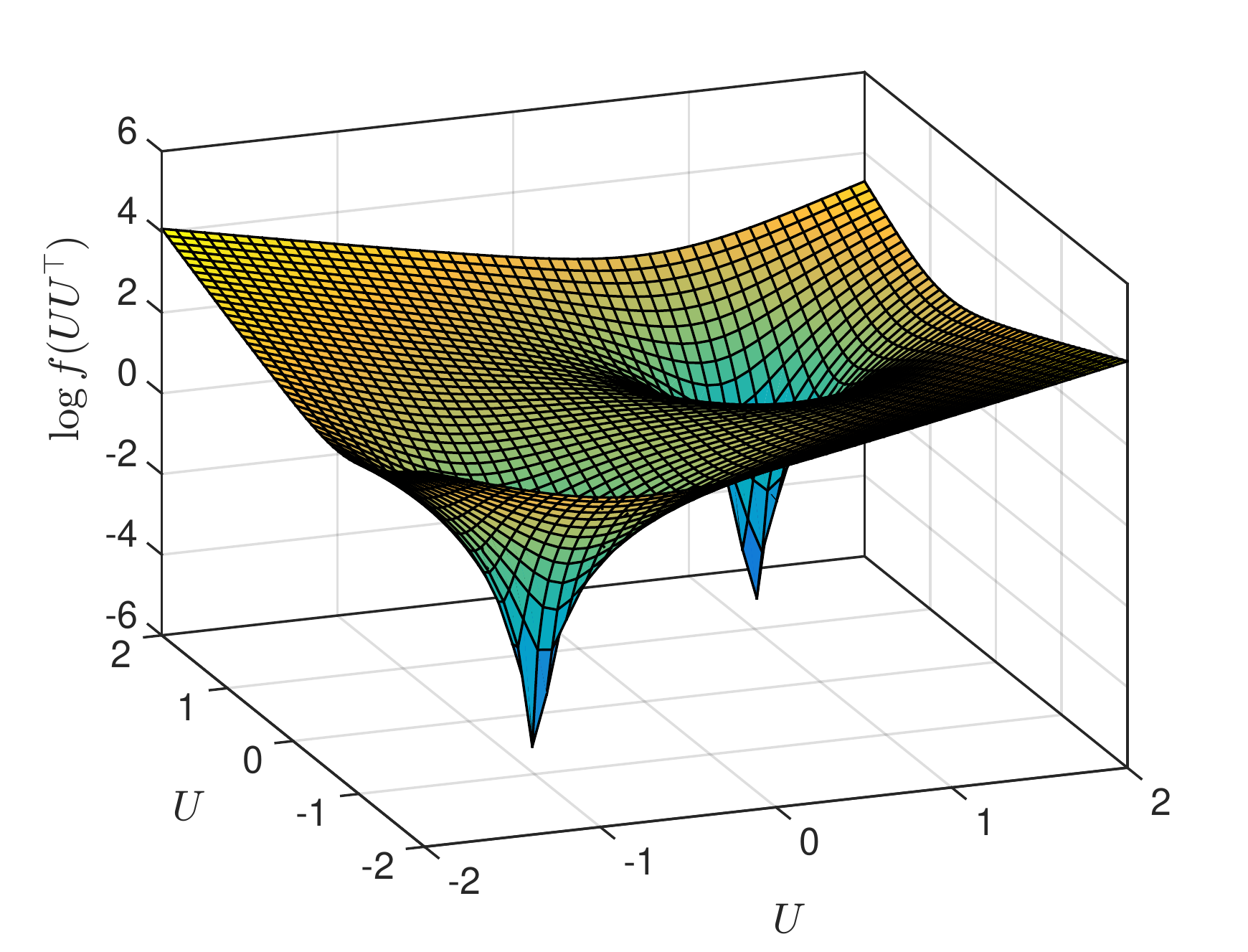}
\includegraphics[width=0.33\textwidth]{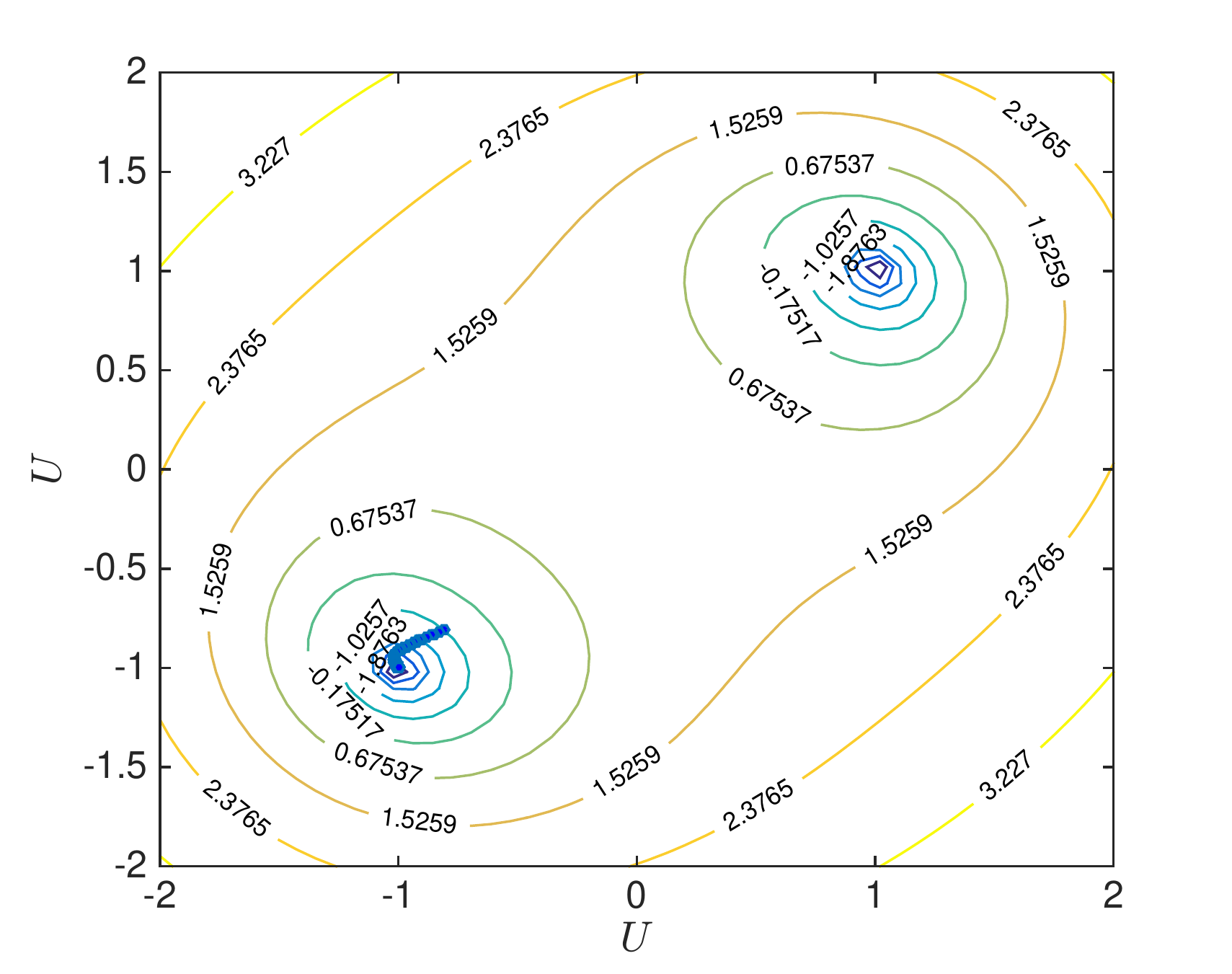}
\includegraphics[width=0.32\textwidth]{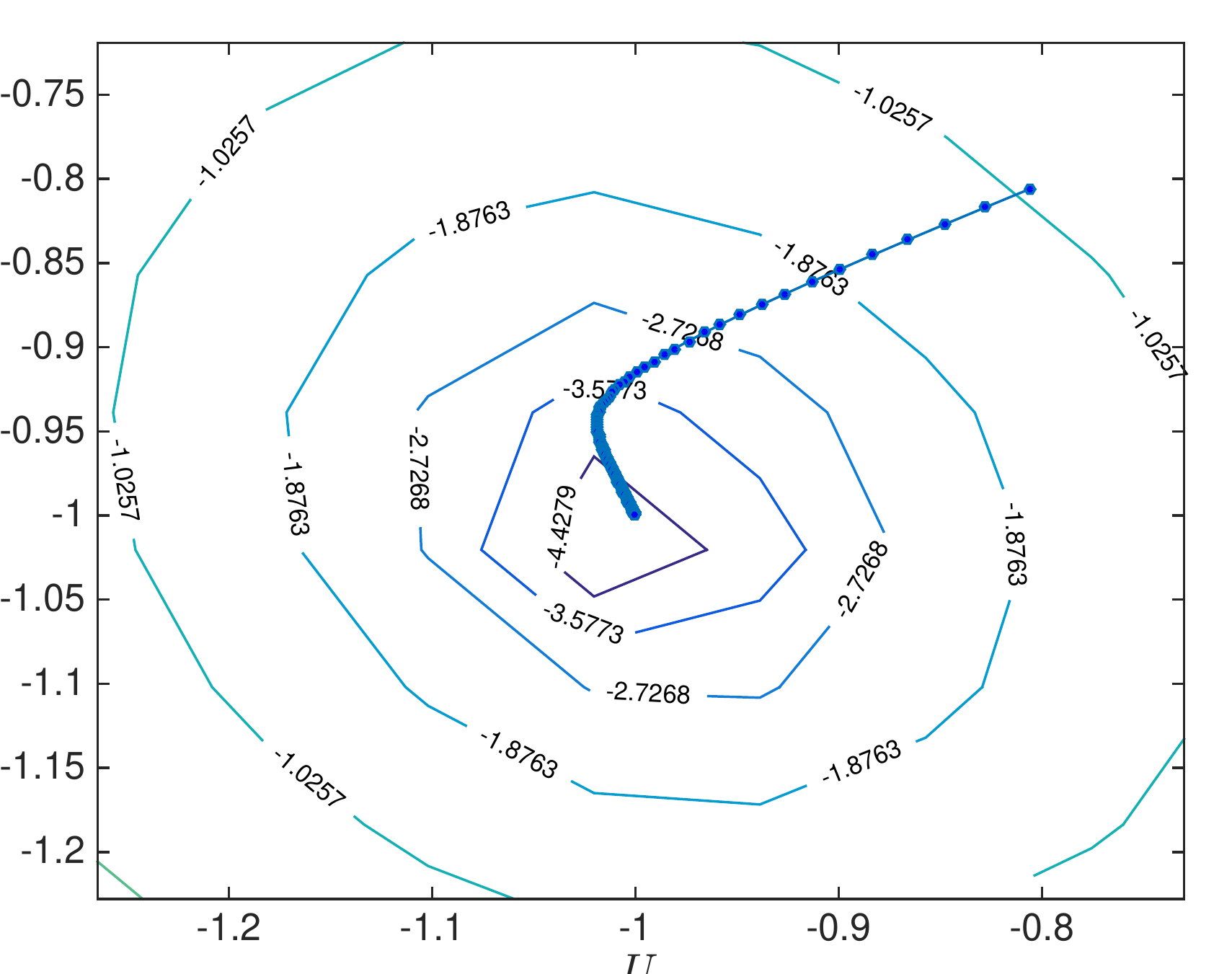}
\caption{Abstract illustration of initialization effect on a toy example. In this experiment, we design $\Xo = \Uo \U^{\star \top}$ where $\Uo = [1 ~~1]^\top$ (or  $\Uo = -[1 ~~1]^\top$---these are equivalent). We observe $\Xo$ via $y = \text{vec}\left(A\cdot \Xo\right)$ where $A \in \R^{3 \times 2}$ is randomly generated. We consider the loss function $f(\U\U^\top) = \tfrac{1}{2} \|y - \text{vec}\left(A\cdot \U\U^\top\right)\|_2^2$. Left panel: $f$ values in logarithimic scale for various values of variable $\U \in \R^{2 \times 1}$. Center panel: Contour lines of $f$ and the bahavior of $\algo$ using our initialization scheme. Right panel: zoom-in plot of center plot.} \label{fig:init_example}
\end{figure}

While the above result guarantees a good initialization for only small values of $\kappa$, in many applications ~\cite{jain2013low, netrapalli2013phase, chen2015fast}, this is indeed the case and $X^0$ has constant relative error to the optimum.

To understand this result, notice that in the extreme case, when $f$ is the $\ell_2$ loss function $\| X -\Xo\|_F^2$, which has condition number $\kappa =1$ and $\text{rank}(\Xo) =r$,  $\X^0$ indeed is the optimum. 
More generally as the condition number $\kappa$ increases, the optimum moves away from $\X^0$ and the above theorem characterizes this error as a function of condition number of the function. 
See also Figure \ref{fig:init_example}.

Now for the setting when the optimum is exactly rank-$r$ we get the following result.
\begin{corollary}[Initialization, exact]\label{cor:exactinit}
Let $\Xo$ be rank-$r$ for some $r\leq n$. Then, under the conditions of Theorem~\ref{thm:scinit}, we get
\begin{equation*} 
\dist(U^0, \Uo_r) \leq  4\sqrt{2} r \tau(\Xo_r)   \cdot \sqrt{\kappa^2 -\sfrac{2}{\kappa} +1}  \cdot \sigma_r(\Uo_r). \\
\end{equation*} 
\end{corollary} 

Finally, for the setting when the function satisfies $(m, r)$-restricted strong convexity, the above corollary still holds as the optimum is a rank-$r$ matrix.

\section{Convergence proofs for the {\rm \algo} algorithm}{\label{sec:proofs}}
In this section, we first present the key techniques required for analyzing the convergence of \algo. 
Later, we present proofs for both Theorems~\ref{thm:smooth_inexact} and \ref{thm:convergence_main}. 
Throughout the proofs we use the following notation. $\Xo$ is the optimum of problem~\eqref{intro:eq_00} and $\Xo_r =\Uorr (\Uorr)^\top$ is the rank-$r$ approximation; 
for the just smooth case, $\Xo = \Xo_r$, as we consider only the rank-$r^\star$ case and $r = r^\star$.
Let $R_U^\star := \argmin_{R:R\in \mathcal{O}} \|U - \Uo_r R\|_F$ and $\Delta =\U - \Uo_r R_U^\star$.

A key property that assists classic gradient descent to converge to the optimum $\Xo$ is the fact that $\langle \Xp -X, ~\Xo -X \rangle \geq 0$ for a smooth convex function $f$; 
in the case of strongly convex $f$, the inner product is further lower bounded by $\frac{m}{2}\|X-\Xo\|_F^2$ (see Theorem 2.2.7 of~\cite{nesterov2004introductory}). 
Classical proofs mainly use such lower bounds to show convergence (see Theorems 2.1.13 and 2.2.8 of~\cite{nesterov2004introductory}). 

We follow broadly similar steps in order to show convergence of \algo. In particular,
\begin{itemize}
\item In section~\ref{sec:proofs_1}, we show a lower bound for the inner product $\ip{U -\Up}{\U -\Uorr}$ (Lemma~\ref{lem:gradU,U-U_r_ bound}), even though the function is not convex in $U$. The initialization and rank-$r$ approximate optimum assumptions play a crucial role in proving this, along with the fact that $f$ is convex in $X$.
\item In sections~\ref{sec:proofs_2} and~\ref{sec:proofs_3}, we use the above lower bound to show convergence for $(i)$ smooth and strongly $f$, and $(ii)$ just smooth $f$, respectively, similar to the convex setting.
\end{itemize}

\subsection{Rudiments of our analysis}\label{sec:proofs_1}
Next, we present the main descent lemma that is used for both sublinear and linear convergence rate guarantees of \algo. 

\begin{lemma}[Descent lemma]\label{lem:gradU,U-U_r_ bound}
For $f$ being a $M$-smooth and $(m, r)$-strongly convex function and, under assumptions $(A2)$ and $(A3)$, the following inequality holds true:
\begin{align*}
\tfrac{1}{\eta}\ip{U -\Up}{\U -\Uorr} \geq  \tfrac{2}{3} \eta  \|\gradf(X) U\|_F^2 + \tfrac{3m}{20} \cdot  \sigma_{r}(\Xo)  \dist(U, \Uo_r)^2 - \tfrac{M}{4}\norm{\Xo -\Xor}_F^2.
\end{align*}
Further, when $f$ is just $M$-smooth convex function and, under the assumptions $f(\Xp) \geq f(\Xor)$ and $(A1)$, we have: 
\begin{align*}
\tfrac{1}{\eta}\ip{U -\Up}{\U -\Uorr} \geq  \tfrac{1}{2} \eta  \|\gradf(X) U\|_F^2.
\end{align*}
\end{lemma}

\begin{proof}
First, we rewrite the inner product as shown below.
\begin{align}
 \ip{\gradf(X) \U}{\U -\Uorr} &= \ip{\gradf(X) }{\X -\Uorr \U^\top} \nonumber \\
&=\frac{1}{2}\ip{\gradf(X) }{\X -\Xo_r}  + \ip{\gradf(X) }{\frac{1}{2}(\X + \Xo_r) -\Uorr \U^\top} \nonumber \\
&=\frac{1}{2}\ip{\gradf(X) }{\X -\Xo_r}  + \frac{1}{2} \ip{\gradf(X) }{\Delta \Delta^\top}, \label{proofsr1:eq_09_main}
\end{align} which follows by adding and subtracting $\tfrac{1}{2}\Xo_r$. 

\begin{itemize}
\item \textsc{Strongly convex $f$ setting.} For this case, the next 3 steps apply.
\end{itemize}

\noindent \textit{Step I: Bounding $\ip{\gradf(X) }{\X -\Xo_r}$.}
The first term in the above expression can be lower bounded using smoothness and strong convexity of $f$ and, involves a construction of a feasible point $\X$. 
We construct such a feasible point by modifying the current update to one with bigger step size $\weta$. 

\begin{lemma}\label{lem:gradX,X-X_r_ bound_sc}
Let $f$ be a $M$-smooth and $(m, r)$-restricted strongly convex function with optimum point $\Xo$. Moreover, let $\Xo_r$ be the best rank-$r$ approximation of $\Xo$. Let $\X = \U\U^\top$. Then,
\begin{align*}   
\ip{\gradf(\X)}{\X-\Xor}  \geq \tfrac{18\weta}{10} \|\gradf(X) U\|_F^2 +\tfrac{m}{2}\norm{\X - \Xor}_F^2 -\tfrac{M}{2}\|\Xo -\Xor\|_F^2,
\end{align*}
where  $\weta =\frac{1}{16 (M\|\X\|_2 + \|\gradf(\X)Q_U Q_U^\top\|_2)} \geq \tfrac{5\eta}{6}$, by Lemma \ref{lem:diff_eta}.
\end{lemma} 
Proof of this lemma is provided in Section~\ref{sec:supp_sc1}.

\medskip
\noindent \textit{Step II: Bounding $\ip{\gradf(X) }{\Delta \Delta^\top}$.}
The second term in equation~\eqref{proofsr1:eq_09_main} can actually be negative. 
Hence, we lower bound it using our initialization assumptions. Intuitively, the second term is smaller than the first one as it scales as $\dist(\U, \Uo_r)^2$, while the first term scales as $\dist(\U, \Uo_r)$. 
\begin{lemma}\label{lem:DD_bound_sc}
Let $f$ be $M$-smooth and $(m, r)$-restricted strongly convex. Then, under assumptions $(A2)$ and $(A4)$, 
the following bound holds true:
\begin{align*}
\ip{\gradf(X) }{ \Delta \Delta^\top} \geq - \tfrac{2\weta}{25 } \|\gradf(\X) U \|_F^2 - \left( \tfrac{m\sigma_{r}(\Xo)}{20} + M\|\Xo -\Xo_r\|_F\right) \cdot \dist(\U, \Uo_r)^2.
\end{align*} 
\end{lemma} 
Proof of this lemma can be found in Section~\ref{sec:supp_sc2}.

\medskip
\noindent \textit{Step III: Combining the bounds in equation~\eqref{proofsr1:eq_09_main}.} For a detailed description, see Section~\ref{sec:supp_sc3}.

\begin{itemize}
\item \textsc{Smooth $f$ setting.} For this case, the next 3 steps apply.
\end{itemize}

\medskip
\noindent \textit{Step I: Bounding $\ip{\gradf(X) }{\X -\Xo_r}$.} Similar to the strongly convex case, one can obtain a lower bound on $\ip{\gradf(X) }{\X -\Xo_r}$, according to the following Lemma:
\begin{lemma}
Let $f$ be a $M$-smooth convex function with optimum point $\Xo_r$. 
 Then, under the assumption that $f(\Xp) \geq f(\Xor)$, the following holds:
\begin{align*}
\ip{\gradf(\X)}{\X-\Xor}  \geq \tfrac{18\weta}{10} \|\gradf(X) U\|_F^2. 
\end{align*}
\end{lemma}
The proof of this lemma can be found in Appendix \ref{sec:smooth_case_proof}.

\medskip
\noindent \textit{Step II: Bounding $\ip{\gradf(X) }{\Delta \Delta^\top}$.} Here, we follow a different path in providing a lower bound for $\ip{\gradf(X) }{\Delta \Delta^\top}$. The following lemma provides such a lower bound.

\begin{lemma}
Let $\X = \U\U^\top$ and define $\Delta := \U - \Uorr$. Let $f(\Xp) \geq f(\Xor)$. 
Then, for $\dist(\U, \Uo_r) \leq \rho \sigma_r\left(\Uo_r\right)$, where $\rho=\tfrac{1}{100} \frac{\sigma_r(\Xo)}{\sigma_1(\Xo)}$, and $f$ being a $M$-smooth convex function, the following lower bound holds:
\begin{align*}
\ip{\gradf(X)}{\Delta \Delta^\top} \geq - \tfrac{\sqrt{2}}{\sqrt{2} - \tfrac{1}{100}} \cdot \tfrac{1}{100} \cdot \left|\ip{\gradf(X)}{\X -\Xor}\right|.
\end{align*}
\end{lemma}
The proof of this lemma can be found in Appendix \ref{sec:smooth_case_proof}.

\medskip
\noindent \textit{Step III: Combining the bounds in equation~\eqref{proofsr1:eq_09_main}.} For a detailed description, see Section \ref{sec:smooth_case_proof} and Lemma~\ref{lem:ipbound_case1}.

%

\end{proof}

\subsection{Proof of linear convergence (Theorem \ref{thm:convergence_main})}\label{sec:proofs_2}

The proof of this theorem involves showing that the potential function $\dist(U, \Uo_r)$ is decreasing per iteration (up to approximation error $||\Xo -\Xor||_F$), using the descent Lemma \ref{lem:gradU,U-U_r_ bound}. 
Using the algorithm's update rule, we obtain
\begin{align}
\dist(\Up, \Uo_r)^2 &= \min_{R:~R \in \mathcal{O}} \|\U - \Uo_rR\|_F^2 \nonumber \\ &\leq \norm{\Up -\Uorr }_F^2 \nonumber \\&= \norm{\Up -\U + \U -\Uorr }_F^2 \nonumber \\
&= \norm{\Up -\U}_F^2 + \norm{\U -\Uorr}_F^2 -2\ip{\Up -\U}{\Uorr -\U}, \label{proofs:eq_04}
\end{align}
which follows by adding and subtracting $\U$ and then expanding the squared term. 

\medskip
\noindent \textit{Step I: Bounding term  $\ip{\U - \Up }{\U - \Uor}$  in \eqref{proofs:eq_04}. }
By Lemma \ref{lem:gradU,U-U_r_ bound}, we can bound the last term on the right hand side as:
\begin{align*}
\ip{\gradf(X) \U}{\U -\Uorr} \geq  \tfrac{2}{3} \eta  \|\gradf(X) U\|_F^2 + \tfrac{3m}{20} \cdot  \sigma_{r}(\Xo)  \dist(\U, \Uo_r)^2- \tfrac{M}{4}\norm{\Xo -\Xor}_F^2.
\end{align*}
Furthermore, we can substitute $\Up$ in the first term to obtain $\|\Up - \U\|_F^2 = \eta^2 \|\gradf(\X)\U\|_F^2$. 

\medskip
\noindent \textit{Step II: Combining bounds into \eqref{proofs:eq_04}.}
Combining the above two equations \eqref{proofs:eq_04} becomes:
\begin{small}
\begin{align}
\dist(\Up, \Uo_r)^2 &\leq \eta^2 \|\gradf(X) \U \|_F^2 + \norm{\U -\Uorr}_F^2 \nonumber \\
							  &\quad \quad \quad \quad \quad \quad \quad \quad - 2 \eta \left( \tfrac{2}{3} \eta  \|\gradf(X) U\|_F^2 + \tfrac{3m}{20} \cdot  \sigma_{r}(\Xo)  \dist(\U, \Uo_r)^2 - \tfrac{M}{4}\norm{\Xo -\Xor}_F^2
 \right) \nonumber \\
&= \norm{\U -\Uorr}_F^2 +  \tfrac{\eta M}{2}\norm{\Xo -\Xor}_F^2 +\eta^2 \underbrace{\left(\|\gradf(X) \U\|_F^2 - \tfrac{4}{3} \norm{\gradf(X) \U}_F^2 \right)}_{\leq 0} \nonumber \\ 
							  &\quad \quad \quad \quad \quad \quad \quad \quad -\tfrac{3m \eta}{10} \cdot  \sigma_{r}(\Xo)  \dist(\U, \Uo_r)^2  \nonumber \\
&\stackrel{(i)}{\leq} \norm{\U -\Uorr}_F^2 + \tfrac{\eta M}{2}\norm{\Xo -\Xor}_F^2  -\tfrac{3m \eta}{10} \cdot  \sigma_{r}(\Xo) \dist(\U, \Uo_r)^2 \nonumber \\
&= \left(1 - \tfrac{3m \eta}{10}\cdot \sigma_r(\Xo)\right)\cdot \dist(\U, \Uo_r)^2 + \tfrac{\eta M}{2}\norm{\Xo -\Xor}_F^2  \nonumber \\
&\stackrel{(ii)}{\leq} \left(1 - \tfrac{3m}{10} \cdot \tfrac{10 \eta^\star}{11} \cdot \sigma_r(\Xo)\right)\cdot \dist(\U, \Uo_r)^2 + M\cdot \tfrac{11 \eta^\star}{20} \norm{\Xo -\Xor}_F^2  \nonumber \\
&\leq  \left(1 -\tfrac{m \eta^\star}{4} \cdot \sigma_{r}(\Xo) \right) \dist(\U, \Uo_r)^2 +  \tfrac{11M\eta^\star}{20}\norm{\Xo -\Xor}_F^2 \nonumber \\
&\stackrel{(iii)}{\leq}\left( 1 -\tfrac{m \sigma_{r}(\Xo)}{64 (M\|\Xo\|_2 + \|\gradf(\Xo)\|_2)}\right)  \dist(\U , \Uo_r)^2 \nonumber \\ 
								&\quad \quad \quad \quad \quad \quad \quad \quad+  \tfrac{ M}{28(M\|\Xo\|_2 + \|\gradf(\Xo)\|_2)}\norm{\Xo -\Xor}_F^2, \nonumber
\end{align}
\end{small}
where $(i)$ is due to removing the negative part from the right hand side, $(ii)$ is due to $\tfrac{10}{11}\eta^\star \leq \eta \leq \tfrac{11}{10} \eta^\star$ by Lemma \ref{lem:diff_eta}, $(iii)$ follows from substituting $\eta^\star =\tfrac{1}{16 (M\|\Xo\|_2 +  \|\gradf(\Xo)\|_2)}$. 
This proves the first part of the theorem.

\medskip
\noindent \textit{Step III: $\Up$ satisfies the initial condition.}
Now we will prove the second part. By the above equation, we have:
\begin{align*}
\dist(\Up, \Uo_r)^2 &\leq \left(1 -\tfrac{m \eta^\star}{4} \cdot \sigma_{r}(\Xo) \right) \dist(\U, \Uo_r)^2 +  \tfrac{11M\eta^\star}{20}\norm{\Xo -\Xor}_F^2 \nonumber \\
&\stackrel{(i)}{\leq}  \left(1 -\tfrac{m \eta^\star}{4} \cdot \sigma_{r}(\Xo) \right) (\rho')^2 \sigma_r(\Xo) +  \tfrac{11M\eta^\star}{20} \tfrac{(\rho')^2}{4\kappa}\sigma_r^2(\Xo)   \nonumber \\
&=( \rho')^2 \sigma_r(\Xo) \left( 1 -  \tfrac{m \eta^\star}{4} \cdot \sigma_{r}(\Xo) + \tfrac{11M\eta^\star}{80\kappa } \cdot \sigma_r(\Xo) \right) \\
&\leq (\rho')^2 \sigma_r(\Xo) \left( 1 -  \tfrac{m \eta^\star}{4} \cdot \sigma_{r}(\Xo) + \tfrac{m\eta^\star}{7} \sigma_r(\Xo)\right) \\
&\leq (\rho')^2 \sigma_r(\Xo).
\end{align*}
$(i)$ follows from substituting the assumptions on $\dist(\U, \Uo_r)$ and $\|\Xo -\Xo_r\|_F$ and the last inequality is due to the term in the parenthesis being less than one.

\subsection{Proof of sublinear convergence (Theorem \ref{thm:smooth_inexact})}\label{sec:proofs_3}

Here, we show convergence of \algo when $f$ is only a $M$-smooth convex function.  
At iterate $k$, we assume $f(X^k) > f(\Xo_r)$; in the opposite case, the bound follows trivially. 
Recall the updates of \algo over the $\U$-space satisfy $$\Up =\U - \eta \gradf(\X) \U.$$ 
It is easy to verify that $\Xp = \Up\left(\Up\right)^\top = \X - \eta \gradf(\X) \X \Lambda - \eta \Lambda^\top \X \gradf(\X)$, where $\Lambda = I - \tfrac{\eta}{2} Q_U Q_U^\top \gradf(X) \in \R^{n \times n}$. Notice that for step size $\eta$, using Lemma ~A.5 we get, 
\begin{align}{\label{eq:Lambda_stuff}}
\Lambda \succ 0, \quad  \|\Lambda\|_2 \leq  1+ \sfrac{1}{32} \quad \text{and} \quad \sigma_{n}(\Lambda) \geq  1- \sfrac{1}{32}.
\end{align} 
Our proof proceeds using the smoothness condition on $f$, at point $\Xp$. In particular,
\begin{align*}
f(\Xp) &\leq f(\X) +\ip{\gradf(\X)}{\Xp -\X} + \tfrac{M}{2} \|\Xp -\X\|_F^2 \\
&\stackrel{(i)}{\leq}  f(\X) - 2 \eta \cdot \sigma_{n} (\Lambda) \cdot \|\gradf(\X) U\|_F^2 + 2M\eta^2 \cdot \|\gradf(\X) U\|_F^2 \cdot \|\X\|_2 \cdot \|\Lambda\|_2^2\\ 
&\stackrel{(ii)}{\leq} f(\X) - \tfrac{\eta \cdot 62}{32} \cdot \|\gradf(\X) U\|_F^2 + \tfrac{\eta}{7} \cdot  \left(\tfrac{33}{32}\right)^2 \cdot \|\gradf(\X) U\|_F^2\\
 &\leq f(\X)  -  \tfrac{17\eta}{10} \|\gradf(\X) U\|_F^2,
\end{align*} 
where $(i)$ follows from symmetry of $\gradf(X)$, $\X$ and 
\begin{align}
\trace(\gradf(X)\gradf(X)X\Lambda) &= \trace(\gradf(X)\gradf(X)UU^\top) -\frac{\eta}{2}\trace(\gradf(X)\gradf(X)UU^\top \gradf(X)) \nonumber \\ &\geq (1 -\frac{\eta}{2} \|Q_U Q_U^\top\gradf(X)\|_2) \|\gradf(\X) U\|_F^2  \nonumber \\ &\geq  (1- \sfrac{1}{32}) \|\gradf(\X) U\|_F^2,
\end{align} and $(ii)$ is due to \eqref{eq:Lambda_stuff} and the fact that  $\eta \leq \tfrac{1}{16M\|\X^0\|_2} \leq \tfrac{1}{14M\|\X\|_2}$ (see Lemma~ A.5). Hence,
\begin{equation}\label{proofs_eq:jc1} 
f(\Xp) - f(\Xo_r) \leq f(\X) - f(\Xo_r) - \tfrac{18\eta}{10} \|\gradf(\X) U\|_F^2.
\end{equation}

To bound the term $f(\X) - f(\Xo_r) $ on the right hand side of \eqref{proofs_eq:jc1}, we use standard convexity as follows: \begin{align}
f(X) &- \f(\Xor) \leq  \ip{\gradf(\X)}{\X -\Xor} \nonumber \\
&\stackrel{(i)}{=} 2\ip{\gradf(\X)}{UU^\top - \Uorr U^\top} - \ip{\gradf(\X)}{UU^\top + \Uorr (\Uorr)^\top- 2\Uorr U^\top}  \nonumber \\
&= 2\ip{\gradf(\X) U}{U - \Uorr} - \ip{\gradf(\X)}{(U -\Uorr) (U -\Uorr)^\top} \nonumber \\
&\stackrel{(ii)}{=} 2\ip{\gradf(\X) U}{\Delta} -\ip{\gradf(\X)}{\Delta \Delta^\top} \nonumber \\ 
&\leq 2\ip{\gradf(\X) U}{\Delta} + \left|\ip{\gradf(\X)}{\Delta \Delta^\top}\right| \nonumber \\ 
&\stackrel{(iii)}{\leq} 2\cdot\|\gradf(\X) U\|_F \cdot \dist(\U, \Uo_r) + \tfrac{1}{40} \|\gradf(X)U\|_2 \cdot \dist(\U, \Uo_r) \nonumber \\ 
&\stackrel{(iv)}{\leq} \tfrac{5}{2} \|\gradf(\X) U\|_F \cdot \dist(\U, \Uo_r), \label{proofs_eq:jc2}
\end{align} 
where $(i)$ is due to $\X = \U\U^\top$ and $\Xor = \Uorr \left(\Uorr\right)^\top$ for orthonormal matrix $R_{\U}^\star \in \R^{r \times r}$, 
$(ii)$ is by $\Delta := U - \Uorr$, 
$(iii)$ is due to Cauchy-Schwarz inequality and Lemma \ref{lem:DD_bound} and, 
$(iv)$ is due to norm ordering $\|\cdot\|_2 \leq \|\cdot\|_F$. 

From \eqref{proofs_eq:jc2}, we obtain to the following bound:
\begin{align}\label{proofs_eq:jc3}
\|\gradf(\X) U\|_F \geq \tfrac{2}{5} \cdot \tfrac{f(\X) - f(\Xor)}{\dist(U, \Uo_r)}.
\end{align} 
Define $\cdiff = f(X) - f(\Xor)$ and $\cdiffp =f(\Xp) -\f(\Xor)$. 
Moreover, by Lemma~\ref{lem:gradf,D_bound1}, we know that $\dist(\U, \Uo_r) \leq \dist(U^0, \Uo_r)$ for all iterations of \textsc{FGD}; 
thus, we have $\tfrac{1}{\dist(U, \Uo_r)} \geq \tfrac{1}{\dist(U^0, \Uo_r)}$ for every update $U$. Using the above definitions and substituting~\eqref{proofs_eq:jc3} in~\eqref{proofs_eq:jc1}, we obtain the following recursion:
\begin{align*}
\cdiffp \leq \cdiff - \tfrac{17{\eta}}{10} \cdot \left(\tfrac{2}{5}\right)^2 \cdot \left(\frac{\cdiff}{\|\Delta\|_F} \right)^2 \leq \cdiff - \tfrac{{\eta}}{5\cdot \dist(U^0, \Uo_r)^2} \cdot \cdiff^2 \Longrightarrow \cdiffp \leq \cdiff\left(1 - \tfrac{{\eta}}{5\cdot \dist(U^0, \Uo_r)^2} \cdot \cdiff\right), \nonumber
\end{align*} which can be further transformed as:
\begin{align*}
\frac{\left(1 - \tfrac{{\eta}}{5\cdot \dist(U^0, \Uo_r)^2} \cdot \cdiff\right)}{\cdiffp} \geq \frac{1}{\cdiff} \Longrightarrow \frac{1}{\cdiffp} \geq \frac{1}{\cdiff} + \tfrac{{\eta}}{5\cdot\dist(U^0, \Uo_r)^2} \cdot \frac{\delta}{\cdiffp} \geq \frac{1}{\cdiff} + \tfrac{{\eta}}{5\cdot \dist(U^0, \Uo_r)^2} 
\end{align*} since $\cdiffp \leq \cdiff$ from equation~\eqref{proofs_eq:jc1}.
Since each $\cdiff$ and $\cdiffp$ correspond to previous and new estimate in \textsc{FGD} per iteration, we can sum up the above inequalities over $k$ iterations to obtain
\begin{align*}
\frac{1}{\cdiff^k} \geq \frac{1}{\cdiff^0} + \tfrac{{\eta}}{5\cdot \dist(U^0, \Uo_r)^2} \cdot k;
\end{align*} here, $\cdiff^k :=  f(\X^k) - f(\Xor)$ and $\cdiff^0 :=  f(\X^0) - f(\Xor)$. 
After simple transformations, we finally obtain\footnote{One can further obtain a bound on the right hand side that depends on $\eta^\star =\frac{1}{16 \, (M \norm{\Xo }_2 + \norm{\gradf(\Xo)}_2) }$. By Lemma \ref{lem:diff_eta}, we know $\eta \geq \tfrac{10}{11}\eta^\star$. Thus, the current proof leads to the bound:
\begin{align*}
f(\X^k) - f(\Xor) \leq \frac{\tfrac{6}{ \eta^\star} \cdot \dist(\U^0, \Uo_r)^2}{k + \tfrac{6}{ \eta^\star} \cdot \tfrac{\dist(\U^0, \Uo_r)^2}{f(\X^0) - f(\Xor)}}.
\end{align*} 
}:
\begin{align*}
f(\X^k) - f(\Xor) \leq \frac{\tfrac{5}{ \eta} \cdot \dist(\U^0, \Uo_r)^2}{k + \tfrac{5}{\eta} \cdot \tfrac{\dist(\U^0, \Uo_r)^2}{f(\X^0) - f(\Xor)}}.
\end{align*} 
This finishes the proof.


\section{Related work}{\label{sec:related}}

\noindent \textbf{Convex approaches.} 
A significant volume of work has focused on solving the classic Semi-Definite Programming (SDP) formulation, where the {\it objective $f$ (as well as any additional convex constraints) is assumed to be  linear}. There, interior point methods (IPMs) constitute a popular choice for small- and moderate-sized problems; see \cite{karmarkar1984new, alizadeh1995interior}. For a comprehensive treatment of this subject, see the excellent survey in \cite{monteiro2003first}.

Large scale SDPs pointed research towards first-order approaches, which are more computationally appealing. 
For linear $f$, we note among others the work of \cite{wen2010alternating}, a provably convergent alternating direction augmented Lagrangian algorithm, and that of Helmberg and Rendl \cite{helmberg2000spectral}, where they develop an efficient first-order spectral bundle method for SDPs with the constant trace property; see also \cite{helmberg2014spectral} for extensions on this line of work. 
In both cases, no convergence rate guarantees are provided;
see also \cite{monteiro2003first}. For completeness, we also mention the work of \cite{burer2003semidefinite, fukuda2001exploiting, nakata2003exploiting, toh2004solving} on \emph{second-order} methods, that take advantage of data sparsity in order to handle large SDPs in a more efficient way. However, it turns out that the amount of computations required per iteration is comparable to that of log-barrier IPMs \cite{monteiro2003first}. 

Standard SDPs have also found application in the field of combinatorial optimization; there, in most cases, even a rough approximation to the discrete problem, via SDP, is sufficiently accurate and computationally affordable, than exhaustive combinatorial algorithms. Goemans and Williamson \cite{goemans1995improved} were the first to propose the use of SDPs in approximating graph \textsc{Max Cut}, where a near-optimum solution can be found in polynomial time. \cite{klein1996efficient} propose an alternative approach for solving \textsc{Max Cut} and \textsc{Graph Coloring} instances, where SDPs are transformed into eigenvalue problems. Then, power method iterations lead to $\varepsilon$-approximate solutions; however, the resulting running-time dependence on $\varepsilon$ is worse, compared to standard IPMs. Arora, Hazan and Kale in \cite{arora2005fast} derive an algorithm to approximate SDPs, as a hybrid of the Multiplicative Weights Update method and of ideas originating from an ellipsoid variant \cite{vaidya1989new}, improving upon existing algorithms for graph partitioning, computational biology and metric embedding problems.\footnote{The algorithm in \cite{arora2005fast} shows significant computational gains over standard IPMs per iteration, due to requiring only a power method calculation per iteration (versus a Cholesky factorization per iteration, in the latter case). However, the polynomial dependence on the accuracy parameter $\tfrac{1}{\varepsilon}$ is worse, compared to IPMs. Improvements upon this matter can be found in \cite{arora2007combinatorial} where a primal-dual Multiplicative Weights Update scheme is proposed.}

Extending to {\it non-linear convex $f$} cases, \cite{nesterov1988general, nesterov1989self} have shown how IPMs can be generalized to solve instances of \eqref{intro:eq_00}, via the notion of self-concordance; see also \cite{lee2012proximal, trandihn15a} for a more recent line of work. 
Within the class of first-order methods, approaches for nonlinear convex $f$ include, among others, projected and proximal gradient descent methods \cite{nesterov2004introductory, trandihn15a, jiang2012inexact}, (smoothed) dual ascent methods \cite{nesterov2007smoothing}, as well as Frank-Wolfe algorithm variants \cite{jaggi2011convex}. Note that all these schemes, often require heavy calculations, such as eigenvalue decompositions, to compute the updates (often, to remain within the feasible set).\\

\noindent \textbf{Burer \& Monteiro factorization and related work.}
Burer and Monteiro \cite{burer2003nonlinear, burer2005local} popularized the idea of solving classic SDPs by representing the solution as a product of two factor matrices. The main idea in such representation is to remove the positive semi-definite constraint by directly embedding it into the objective. While the problem becomes non-convex, Burer and Monteiro propose a method-of-multiplier type of algorithm which iteratively updates the factors in an alternating fashion. For linear objective $f$, they establish convergence guarantees to the optimum but do not provide convergence rates. 

For generic smooth convex functions, Hazan in \cite{hazan2008sparse} proposes \textsc{SparseApproxSDP} algorithm,\footnote{Sparsity here corresponds to low-rankness of the solution, as in the Cholesky factorization representation. Moreover, inspired by Quantum State Tomography applications \cite{aaronson2007learnability}, \textsc{SparseApproxSDP} can also handle constant trace constraints, in addition to PSD ones. } a generalization of the Frank-Wolfe algorithm for the vector case \cite{clarkson2010coresets}, where putative solutions are refined by rank-1 approximations of the gradient. At the $r$-th iteration, \textsc{SparseApproxSDP} is guaranteed to compute a $\tfrac{1}{r}$-approximate solution, with rank at most $r$, \textit{i.e.}, achieves a sublinear $O\left(\tfrac{1}{\varepsilon}\right)$ convergence rate. However, depending on $\varepsilon$, \textsc{SparseApproxSDP} is not guaranteed to return a low rank solution unlike \algo. Application of these ideas in machine learning tasks can be found in \cite{shalev2011large}. Based on \textsc{SparseApproxSDP} algorithm, \cite{laue2012hybrid} further introduces ``de-bias'' steps in order to optimize parameters in \textsc{SparseApproxSDP} and do local refinements of putative solutions via L-BFGS steps. Nevertheless, the resulting convergence rate is still sublinear.\footnote{For running time comparisons with \algo see Sections~\ref{sec:QST} and \ref{sec:high_rank}.} 

Specialized algorithms -- for objectives beyond the linear case -- that utilize such factorization include matrix completion /sensing solvers \cite{jain2013low, sun2014guaranteed, zheng2015convergent, tu2015low}, non-negative matrix factorization schemes \cite{lee2001algorithms}, phase retrieval methods \cite{netrapalli2013phase, candes2015phase1} and sparse PCA algorithms \cite{laue2012hybrid}. Most of these results guarantee linear convergence for various algorithms on the factored space starting from a ``good'' initialization. They also present a simple spectral method to compute such an initialization. For the matrix completion /sensing setting, \cite{sa2015global} have shown that stochastic gradient descent achieves global convergence at a sublinear rate. Note that these results only apply to quadratic loss objectives and not to generic convex functions $f$.\footnote{We recently became aware of the extension of the work \cite{tu2015low} for the non-square case $\X = \U\V^\top$.} 
\cite{jain2015computing} consider the problem of computing the matrix square-root of a PSD matrix via gradient descent on the factored space: in this case, the objective $f$ boils down to minimizing the standard squared Euclidean norm distance between two matrices. 
Surprisingly, the authors show that, given an initial point that is well-conditioned, the proposed scheme is guaranteed to find an $\varepsilon$-accurate solution with linear convergence rate.

\cite{chen2015fast} propose a first-order optimization framework for the problem \eqref{intro:eq_00}, where the same parametrization technique is used to efficiently accommodate the PSD constraint.\footnote{In this work, the authors further assume orthogonality of columns in $\U$.} Moreover, the proposed algorithmic solution can accommodate extra constraints on $\X$.\footnote{Though, additional constraints should satisfy the $\Xo$-\textit{faithfulness} property: a constraint set on $\U$, say $\mathcal{U}$, is faithful if for each $\U \in \mathcal{U}$, that is within some bounded radius from optimal point, we are guaranteed that the closest (in the Euclidean sense) rotation of $\Uo$ lies within $\mathcal{U}$.} 
The set of assumptions listed in \cite{chen2015fast} include---apart from $\Xo$-faithfulness---local descent, local Lipschitz and local smoothness conditions \emph{in the factored space}. \emph{E.g.}, the local descent condition can be established if $g(\U) := f(\U\U^\top)$ is locally strongly convex and $\nabla g(\cdot)$ at an optimum point vanishes. They also require bounded gradients as their step size doesn't account for the modified curvature of $f(UU^\top)$. \footnote{One can define non-trivially conditions on the original space; we defer the reader to \cite{chen2015fast}}
These conditions are less standard than the global assumptions of the current work and one needs to validate that they are satisfied for each problem, separately. 
\cite{chen2015fast} presents some applications where these conditions are indeed satisfied.
Their results are of the same flavor with ours: under such proper assumptions, one can prove local convergence with $O(1/\varepsilon)$ or  $O(\log(1/\varepsilon))$ rate and for $f$ instances that even fail to be locally convex.

Finally, for completeness, we also mention optimization over the Grassmannian manifold that admits tailored solvers~\cite{edelman1998geometry}; see \cite{keshavan2010matrix, boumal2014optimization, boumal2015riemannian, zhang2015global, uschmajewgreedy} for applications in matrix completion and references therein. \cite{journee2010low} presents a second-order method for \eqref{intro:eq_00}, based on manifold optimization over the set of all equivalence class $\mathcal{O}$. The proposed algorithm can additionally accommodate constraints and enjoys monotonic decrease of the objective function (in contrast to \cite{burer2003nonlinear, burer2005local}), featuring quadratic local convergence. 
In practice, the per iteration complexity is dominated by the extraction of the eigenvector, corresponding to the smallest eigenvalue, of a $n \times n$ matrix---and only when the current estimate of rank satisfies some conditions.

Table \ref{table:algo_comp_summary} summarizes the comparison of the most relevant work to ours, for the case of matrix factorization techniques.

\begin{table*}[!htb]
\centering
\begin{footnotesize}
\begin{tabular}{c c c c c c c}
  \toprule
  Reference & & Conv. rate & & Initialization & & Output rank \\ 
  \cmidrule{1-1} \cmidrule{3-3} \cmidrule{5-5} \cmidrule{7-7} 
      \cite{hazan2008sparse} & & $\sfrac{1}{\varepsilon}$ (Smooth $f$) & & $\X^0 = 0$ & & $\sfrac{1}{\varepsilon}$ \\      
      \cite{laue2012hybrid} & & $\sfrac{1}{\varepsilon}$ (Smooth $f$) & & $\X^0 = 0$ & & $\sfrac{1}{\varepsilon}$ \\
      \cite{chen2015fast} & & $\sfrac{1}{\varepsilon}$ (Local Asm.) & & Application dependent & & $r$ \\ 
      \cite{chen2015fast} & & $\log(\sfrac{1}{\varepsilon})$ (Local Asm.)& & Application dependent  & &  $r$ \\ 
  \midrule
        This work & & $\sfrac{1}{\varepsilon}$ (Smooth $f$) & & SVD $/$ top-$r$ & & $r$ \\      
        This work & & $\log(\sfrac{1}{\varepsilon})$ (Smooth, RSC $f$)& & SVD $/$ top-$r$ & & $r$ \\      
  \bottomrule
\end{tabular}
\end{footnotesize}
\caption{Summary of selected results on solving variants of \eqref{intro:eq_00} via matrix factorization. ``Conv. rate" describes the number of iterations required to achieve $\varepsilon$ accuracy. ``Initialization'' describes the process for starting point computation. ``SVD" stands for singular value decomposition and ``top-$r$" denotes that a rank-$r$ decomposition is computed. For the case of \cite{chen2015fast}, ``Local Asm.'' refer to specific assumptions made on the $U$-space; we refer the reader to the footnote for a short description. ``Output rank'' denotes the maximum rank of solution returned for $\varepsilon$-accuracy. 
} \label{table:algo_comp_summary} 
\end{table*}



%

%
%
%
%
%
%
\section{Conclusion}
In this paper, we focus on how to efficiently minimize a convex function $f$ over the positive semi-definite cone. Inspired by the seminal work \cite{burer2003nonlinear, burer2005local}, we drop convexity by factorizing the optimization variable $\X = \U\U^\top$ and show that \emph{factored gradient descent} with a non-trivial step size selection results in linear convergence when $f$ is smooth and (restricted) strongly convex, even though the problem is now non-convex. In the case where $f$ is only smooth, only sublinear rate is guaranteed. In addition, we present initialization schemes that use only first order information and guarantee to find a starting point with small relative distance from optimum. 

There are many possible directions for future work, extending the idea of using non-convex formulation for semi-definite optimization. Showing convergence under weaker initialization condition or without any initialization requirement is definitely of great interest. Another interesting direction is to improve the convergence rates presented in this work, by using acceleration techniques and thus, extend ideas used in the case of convex gradient descent \cite{nesterov2004introductory}. Finally, it would be valuable to see how the techniques presented in this paper can be generalized to other standard algorithms like stochastic gradient descent and coordinate descent. 

Furthermore, we identify applications, such as sparse PCA \cite{vu2013minimax, asteris2015sparse}, that require non-smooth constraints on the factors $U$. That being said, an extension of this work to proximal techniques for the non-convex case is a very interesting future research direction.

\bibliographystyle{plain}
\bibliography{gradient_matrix}

\begin{thebibliography}{10}

\bibitem{aaronson2007learnability}
Scott Aaronson.
\newblock The learnability of quantum states.
\newblock In {\em Proceedings of the Royal Society of London A: Mathematical,
  Physical and Engineering Sciences}, volume 463, pages 3089--3114. The Royal
  Society, 2007.

\bibitem{agarwal2010fast}
Alekh Agarwal, Sahand Negahban, and Martin~J Wainwright.
\newblock Fast global convergence rates of gradient methods for
  high-dimensional statistical recovery.
\newblock In {\em Advances in Neural Information Processing Systems}, pages
  37--45, 2010.

\bibitem{alizadeh1995interior}
Farid Alizadeh.
\newblock Interior point methods in semidefinite programming with applications
  to combinatorial optimization.
\newblock {\em SIAM Journal on Optimization}, 5(1):13--51, 1995.

\bibitem{arora2005fast}
Sanjeev Arora, Elad Hazan, and Satyen Kale.
\newblock Fast algorithms for approximate semidefinite programming using the
  multiplicative weights update method.
\newblock In {\em Foundations of Computer Science, 2005. FOCS 2005. 46th Annual
  IEEE Symposium on}, pages 339--348. IEEE, 2005.

\bibitem{arora2007combinatorial}
Sanjeev Arora and Satyen Kale.
\newblock A combinatorial, primal-dual approach to semidefinite programs.
\newblock In {\em Proceedings of the thirty-ninth annual ACM symposium on
  Theory of computing}, pages 227--236. ACM, 2007.

\bibitem{asteris2015sparse}
Megasthenis Asteris, Dimitris Papailiopoulos, Anastasios Kyrillidis, and
  Alexandros~G Dimakis.
\newblock Sparse {PCA} via bipartite matchings.
\newblock {\em arXiv preprint arXiv:1508.00625}, 2015.

\bibitem{becker2011templates}
S.~Becker, E.~Cand{\`e}s, and M.~Grant.
\newblock Templates for convex cone problems with applications to sparse signal
  recovery.
\newblock {\em Mathematical Programming Computation}, 3(3):165--218, 2011.

\bibitem{becker2013randomized}
Stephen Becker, Volkan Cevher, and Anastasios Kyrillidis.
\newblock Randomized low-memory singular value projection.
\newblock In {\em 10th International Conference on Sampling Theory and
  Applications (Sampta)}, 2013.

\bibitem{bhatia1987perturbation}
Rajendra Bhatia.
\newblock {\em Perturbation bounds for matrix eigenvalues}, volume~53.
\newblock SIAM, 1987.

\bibitem{boumal2014optimization}
Nicolas Boumal.
\newblock {\em Optimization and estimation on manifolds}.
\newblock PhD thesis, UC Louvain, Belgium, 2014.

\bibitem{boumal2015riemannian}
Nicolas Boumal.
\newblock A riemannian low-rank method for optimization over semidefinite
  matrices with block-diagonal constraints.
\newblock {\em arXiv preprint arXiv:1506.00575}, 2015.

\bibitem{boyd2004convex}
Stephen Boyd and Lieven Vandenberghe.
\newblock {\em Convex optimization}.
\newblock Cambridge university press, 2004.

\bibitem{bubeck2014theory}
S{\'e}bastien Bubeck.
\newblock Theory of convex optimization for machine learning.
\newblock {\em arXiv preprint arXiv:1405.4980}, 2014.

\bibitem{burer2003semidefinite}
Samuel Burer.
\newblock Semidefinite programming in the space of partial positive
  semidefinite matrices.
\newblock {\em SIAM Journal on Optimization}, 14(1):139--172, 2003.

\bibitem{burer2003nonlinear}
Samuel Burer and Renato~DC Monteiro.
\newblock A nonlinear programming algorithm for solving semidefinite programs
  via low-rank factorization.
\newblock {\em Mathematical Programming}, 95(2):329--357, 2003.

\bibitem{burer2005local}
Samuel Burer and Renato~DC Monteiro.
\newblock Local minima and convergence in low-rank semidefinite programming.
\newblock {\em Mathematical Programming}, 103(3):427--444, 2005.

\bibitem{cai2010singular}
J.~Cai, E.~Cand{\`e}s, and Z.~Shen.
\newblock A singular value thresholding algorithm for matrix completion.
\newblock {\em SIAM Journal on Optimization}, 20(4):1956--1982, 2010.

\bibitem{candes2015phase}
Emmanuel~J Candes, Yonina~C Eldar, Thomas Strohmer, and Vladislav Voroninski.
\newblock Phase retrieval via matrix completion.
\newblock {\em SIAM Review}, 57(2):225--251, 2015.

\bibitem{candes2015phase1}
Emmanuel~J Candes, Xiaodong Li, and Mahdi Soltanolkotabi.
\newblock Phase retrieval via wirtinger flow: Theory and algorithms.
\newblock {\em Information Theory, IEEE Transactions on}, 61(4):1985--2007,
  2015.

\bibitem{candes2011tight}
Emmanuel~J Candes and Yaniv Plan.
\newblock Tight oracle inequalities for low-rank matrix recovery from a minimal
  number of noisy random measurements.
\newblock {\em Information Theory, IEEE Transactions on}, 57(4):2342--2359,
  2011.

\bibitem{candes2009exact}
Emmanuel~J Cand{\`e}s and Benjamin Recht.
\newblock Exact matrix completion via convex optimization.
\newblock {\em Foundations of Computational mathematics}, 9(6):717--772, 2009.

\bibitem{chen2014coherent}
Yudong Chen, Srinadh Bhojanapalli, Sujay Sanghavi, and Rachel Ward.
\newblock Coherent matrix completion.
\newblock In {\em Proceedings of The 31st International Conference on Machine
  Learning}, pages 674--682, 2014.

\bibitem{chen2015fast}
Yudong Chen and Martin~J Wainwright.
\newblock Fast low-rank estimation by projected gradient descent: General
  statistical and algorithmic guarantees.
\newblock {\em arXiv preprint arXiv:1509.03025}, 2015.

\bibitem{chen2010general}
Yuxin Chen and Sujay Sanghavi.
\newblock A general framework for high-dimensional estimation in the presence
  of incoherence.
\newblock In {\em Communication, Control, and Computing (Allerton), 2010 48th
  Annual Allerton Conference on}, pages 1570--1576. IEEE, 2010.

\bibitem{clarkson2010coresets}
Kenneth~L Clarkson.
\newblock Coresets, sparse greedy approximation, and the {F}rank-{W}olfe
  algorithm.
\newblock {\em ACM Transactions on Algorithms (TALG)}, 6(4):63, 2010.

\bibitem{d2007direct}
Alexandre d'Aspremont, Laurent El~Ghaoui, Michael~I Jordan, and Gert~RG
  Lanckriet.
\newblock A direct formulation for sparse {PCA} using semidefinite programming.
\newblock {\em SIAM review}, 49(3):434--448, 2007.

\bibitem{trandihn15a}
Quoc~Tran Dinh, Anastasios Kyrillidis, and Volkan Cevher.
\newblock Composite self-concordant minimization.
\newblock {\em Journal of Machine Learning Research}, 16:371--416, 2015.

\bibitem{edelman1998geometry}
Alan Edelman, Tom{\'a}s~A Arias, and Steven~T Smith.
\newblock The geometry of algorithms with orthogonality constraints.
\newblock {\em SIAM journal on Matrix Analysis and Applications},
  20(2):303--353, 1998.

\bibitem{fazel2002matrix}
M.~Fazel.
\newblock {\em Matrix rank minimization with applications}.
\newblock PhD thesis, PhD thesis, Stanford University, 2002.

\bibitem{flammia2012quantum}
S.~Flammia, D.~Gross, Y.-K. Liu, and J.~Eisert.
\newblock Quantum tomography via compressed sensing: {E}rror bounds, sample
  complexity and efficient estimators.
\newblock {\em New Journal of Physics}, 14(9):095022, 2012.

\bibitem{fukuda2001exploiting}
Mituhiro Fukuda, Masakazu Kojima, Kazuo Murota, and Kazuhide Nakata.
\newblock Exploiting sparsity in semidefinite programming via matrix completion
  {I}: General framework.
\newblock {\em SIAM Journal on Optimization}, 11(3):647--674, 2001.

\bibitem{goemans1995improved}
Michel~X Goemans and David~P Williamson.
\newblock Improved approximation algorithms for maximum cut and satisfiability
  problems using semidefinite programming.
\newblock {\em Journal of the ACM (JACM)}, 42(6):1115--1145, 1995.

\bibitem{gross2010quantum}
David Gross, Yi-Kai Liu, Steven~T Flammia, Stephen Becker, and Jens Eisert.
\newblock Quantum state tomography via compressed sensing.
\newblock {\em Physical review letters}, 105(15):150401, 2010.

\bibitem{hazan2008sparse}
Elad Hazan.
\newblock Sparse approximate solutions to semidefinite programs.
\newblock In {\em LATIN 2008: Theoretical Informatics}, pages 306--316.
  Springer, 2008.

\bibitem{helmberg2014spectral}
Christoph Helmberg, Michael~L Overton, and Franz Rendl.
\newblock The spectral bundle method with second-order information.
\newblock {\em Optimization Methods and Software}, 29(4):855--876, 2014.

\bibitem{helmberg2000spectral}
Christoph Helmberg and Franz Rendl.
\newblock A spectral bundle method for semidefinite programming.
\newblock {\em SIAM Journal on Optimization}, 10(3):673--696, 2000.

\bibitem{horn37topics}
Roger~A Horn and Charles~R Johnson.
\newblock Topics in matrix analysis.
\newblock {\em Cambridge University Presss, Cambridge}, 37:39, 1991.

\bibitem{hsieh2011sparse}
Cho-Jui Hsieh, Inderjit~S Dhillon, Pradeep~K Ravikumar, and M{\'a}ty{\'a}s~A
  Sustik.
\newblock Sparse inverse covariance matrix estimation using quadratic
  approximation.
\newblock In {\em Advances in Neural Information Processing Systems}, pages
  2330--2338, 2011.

\bibitem{jaggi2011convex}
Martin Jaggi.
\newblock Convex optimization without projection steps.
\newblock {\em arXiv preprint arXiv:1108.1170}, 2011.

\bibitem{jain2015computing}
Prateek Jain, Chi Jin, Sham~M Kakade, and Praneeth Netrapalli.
\newblock Computing matrix squareroot via non convex local search.
\newblock {\em arXiv preprint arXiv:1507.05854}, 2015.

\bibitem{jain2010guaranteed}
Prateek Jain, Raghu Meka, and Inderjit~S Dhillon.
\newblock Guaranteed rank minimization via singular value projection.
\newblock In {\em Advances in Neural Information Processing Systems}, pages
  937--945, 2010.

\bibitem{jain2013low}
Prateek Jain, Praneeth Netrapalli, and Sujay Sanghavi.
\newblock Low-rank matrix completion using alternating minimization.
\newblock In {\em Proceedings of the 45th annual ACM symposium on Symposium on
  theory of computing}, pages 665--674. ACM, 2013.

\bibitem{javanmard2013localization}
Adel Javanmard and Andrea Montanari.
\newblock Localization from incomplete noisy distance measurements.
\newblock {\em Foundations of Computational Mathematics}, 13(3):297--345, 2013.

\bibitem{jiang2012inexact}
Kaifeng Jiang, Defeng Sun, and Kim-Chuan Toh.
\newblock An inexact accelerated proximal gradient method for large scale
  linearly constrained convex {SDP}.
\newblock {\em SIAM Journal on Optimization}, 22(3):1042--1064, 2012.

\bibitem{journee2010low}
Michel Journ{\'e}e, Francis Bach, P-A Absil, and Rodolphe Sepulchre.
\newblock Low-rank optimization on the cone of positive semidefinite matrices.
\newblock {\em SIAM Journal on Optimization}, 20(5):2327--2351, 2010.

\bibitem{karmarkar1984new}
Narendra Karmarkar.
\newblock A new polynomial-time algorithm for linear programming.
\newblock In {\em Proceedings of the sixteenth annual ACM symposium on Theory
  of computing}, pages 302--311. ACM, 1984.

\bibitem{keshavan2010matrix}
Raghunandan~H Keshavan, Andrea Montanari, and Sewoong Oh.
\newblock Matrix completion from a few entries.
\newblock {\em Information Theory, IEEE Transactions on}, 56(6):2980--2998,
  2010.

\bibitem{klein1996efficient}
Philip Klein and Hsueh-I Lu.
\newblock Efficient approximation algorithms for semidefinite programs arising
  from {MAX CUT} and {COLORING}.
\newblock In {\em Proceedings of the twenty-eighth annual ACM symposium on
  Theory of computing}, pages 338--347. ACM, 1996.

\bibitem{kyrillidis2014matrix}
Anastasios Kyrillidis and Volkan Cevher.
\newblock Matrix recipes for hard thresholding methods.
\newblock {\em Journal of mathematical imaging and vision}, 48(2):235--265,
  2014.

\bibitem{kyrillidis2014scalable}
Anastasios Kyrillidis, Rabeeh Karimi, Quoc~Tran Dinh, and Volkan Cevher.
\newblock Scalable sparse covariance estimation via self-concordance.
\newblock In {\em Twenty-Eighth AAAI Conference on Artificial Intelligence},
  2014.

\bibitem{laue2012hybrid}
Soeren Laue.
\newblock A hybrid algorithm for convex semidefinite optimization.
\newblock In {\em Proceedings of the 29th International Conference on Machine
  Learning (ICML-12)}, pages 177--184, 2012.

\bibitem{lee2001algorithms}
Daniel~D Lee and H~Sebastian Seung.
\newblock Algorithms for non-negative matrix factorization.
\newblock In {\em Advances in neural information processing systems}, pages
  556--562, 2001.

\bibitem{lee2012proximal}
Jason Lee, Yuekai Sun, and Michael Saunders.
\newblock Proximal newton-type methods for convex optimization.
\newblock In {\em Advances in Neural Information Processing Systems}, pages
  836--844, 2012.

\bibitem{liu2011universal}
Yi-Kai Liu.
\newblock Universal low-rank matrix recovery from {P}auli measurements.
\newblock In {\em Advances in Neural Information Processing Systems}, pages
  1638--1646, 2011.

\bibitem{mirsky1975trace}
Leon Mirsky.
\newblock A trace inequality of {J}ohn von {N}eumann.
\newblock {\em Monatshefte f{\"u}r Mathematik}, 79(4):303--306, 1975.

\bibitem{mishra2011low}
Bamdev Mishra, Gilles Meyer, and Rodolphe Sepulchre.
\newblock Low-rank optimization for distance matrix completion.
\newblock In {\em Decision and control and European control conference
  (CDC-ECC), 2011 50th IEEE conference on}, pages 4455--4460. IEEE, 2011.

\bibitem{monteiro2003first}
Renato~DC Monteiro.
\newblock First-and second-order methods for semidefinite programming.
\newblock {\em Mathematical Programming}, 97(1-2):209--244, 2003.

\bibitem{nakata2003exploiting}
Kazuhide Nakata, Katsuki Fujisawa, Mituhiro Fukuda, Masakazu Kojima, and Kazuo
  Murota.
\newblock Exploiting sparsity in semidefinite programming via matrix completion
  {II}: Implementation and numerical results.
\newblock {\em Mathematical Programming}, 95(2):303--327, 2003.

\bibitem{negahban2012restricted}
Sahand Negahban and Martin~J Wainwright.
\newblock Restricted strong convexity and weighted matrix completion: Optimal
  bounds with noise.
\newblock {\em The Journal of Machine Learning Research}, 13(1):1665--1697,
  2012.

\bibitem{nesterov2004introductory}
Yurii Nesterov.
\newblock {\em Introductory lectures on convex optimization}, volume~87.
\newblock Springer Science \& Business Media, 2004.

\bibitem{nesterov2007smoothing}
Yurii Nesterov.
\newblock Smoothing technique and its applications in semidefinite
  optimization.
\newblock {\em Mathematical Programming}, 110(2):245--259, 2007.

\bibitem{nesterov1988general}
Yurii Nesterov and Arkadi Nemirovski.
\newblock A general approach to polynomial-time algorithms design for convex
  programming.
\newblock {\em Report, Central Economical and Mathematical Institute, USSR
  Academy of Sciences, Moscow}, 1988.

\bibitem{nesterov1989self}
Yurii Nesterov and Arkadi Nemirovski.
\newblock {\em Self-concordant functions and polynomial-time methods in convex
  programming}.
\newblock USSR Academy of Sciences, Central Economic \& Mathematic Institute,
  1989.

\bibitem{netrapalli2013phase}
Praneeth Netrapalli, Prateek Jain, and Sujay Sanghavi.
\newblock Phase retrieval using alternating minimization.
\newblock In {\em Advances in Neural Information Processing Systems}, pages
  2796--2804, 2013.

\bibitem{recht2010guaranteed}
Benjamin Recht, Maryam Fazel, and Pablo~A Parrilo.
\newblock Guaranteed minimum-rank solutions of linear matrix equations via
  nuclear norm minimization.
\newblock {\em SIAM review}, 52(3):471--501, 2010.

\bibitem{sa2015global}
Christopher~D Sa, Christopher Re, and Kunle Olukotun.
\newblock Global convergence of stochastic gradient descent for some non-convex
  matrix problems.
\newblock In {\em Proceedings of the 32nd International Conference on Machine
  Learning (ICML-15)}, pages 2332--2341, 2015.

\bibitem{shalev2011large}
Shai Shalev-shwartz, Alon Gonen, and Ohad Shamir.
\newblock Large-scale convex minimization with a low-rank constraint.
\newblock In {\em Proceedings of the 28th International Conference on Machine
  Learning (ICML-11)}, pages 329--336, 2011.

\bibitem{sun2016geometric}
Ju~Sun, Qing Qu, and John Wright.
\newblock A geometric analysis of phase retrieval.
\newblock {\em arXiv preprint arXiv:1602.06664}, 2016.

\bibitem{sun2014guaranteed}
Ruoyu Sun and Zhi-Quan Luo.
\newblock Guaranteed matrix completion via non-convex factorization.
\newblock {\em arXiv preprint arXiv:1411.8003}, 2014.

\bibitem{toh2004solving}
Kim-Chuan Toh.
\newblock Solving large scale semidefinite programs via an iterative solver on
  the augmented systems.
\newblock {\em SIAM Journal on Optimization}, 14(3):670--698, 2004.

\bibitem{tu2015low}
Stephen Tu, Ross Boczar, Mahdi Soltanolkotabi, and Benjamin Recht.
\newblock Low-rank solutions of linear matrix equations via {P}rocrustes flow.
\newblock {\em arXiv preprint arXiv:1507.03566}, 2015.

\bibitem{uschmajewgreedy}
Andre Uschmajew and Bart Vandereycken.
\newblock Greedy rank updates combined with riemannian descent methods for
  low-rank optimization.
\newblock In {\em 12th International Conference on Sampling Theory and
  Applications (Sampta)}, 2015.

\bibitem{vaidya1989new}
Pravin~M Vaidya.
\newblock A new algorithm for minimizing convex functions over convex sets.
\newblock In {\em Foundations of Computer Science, 1989., 30th Annual Symposium
  on}, pages 338--343. IEEE, 1989.

\bibitem{vu2013minimax}
Vincent~Q Vu, Jing Lei, et~al.
\newblock Minimax sparse principal subspace estimation in high dimensions.
\newblock {\em The Annals of Statistics}, 41(6):2905--2947, 2013.

\bibitem{waters2011sparcs}
Andrew~E Waters, Aswin~C Sankaranarayanan, and Richard Baraniuk.
\newblock Sparcs: Recovering low-rank and sparse matrices from compressive
  measurements.
\newblock In {\em Advances in neural information processing systems}, pages
  1089--1097, 2011.

\bibitem{wen2010alternating}
Zaiwen Wen, Donald Goldfarb, and Wotao Yin.
\newblock Alternating direction augmented {L}agrangian methods for semidefinite
  programming.
\newblock {\em Mathematical Programming Computation}, 2(3-4):203--230, 2010.

\bibitem{white2015local}
Chris~D White, Sujay Sanghavi, and Rachel Ward.
\newblock The local convexity of solving systems of quadratic equations.
\newblock {\em arXiv preprint arXiv:1506.07868}, 2015.

\bibitem{yu2014large}
Hsiang-Fu Yu, Prateek Jain, Purushottam Kar, and Inderjit Dhillon.
\newblock Large-scale multi-label learning with missing labels.
\newblock In {\em Proceedings of The 31st International Conference on Machine
  Learning}, pages 593--601, 2014.

\bibitem{yurtsever2015universal}
A.~Yurtsever, Q.~Tran-Dinh, and V.~Cevher.
\newblock A universal primal-dual convex optimization framework.
\newblock In {\em Advances in Neural Information Processing Systems 28}, pages
  3132--3140. 2015.

\bibitem{zhang2015global}
Dejiao Zhang and Laura Balzano.
\newblock Global convergence of a grassmannian gradient descent algorithm for
  subspace estimation.
\newblock {\em arXiv preprint arXiv:1506.07405}, 2015.

\bibitem{zhao2015nonconvex}
Tuo Zhao, Zhaoran Wang, and Han Liu.
\newblock A nonconvex optimization framework for low rank matrix estimation.
\newblock In {\em Advances in Neural Information Processing Systems}, pages
  559--567, 2015.

\bibitem{zheng2015convergent}
Qinqing Zheng and John Lafferty.
\newblock A convergent gradient descent algorithm for rank minimization and
  semidefinite programming from random linear measurements.
\newblock {\em arXiv preprint arXiv:1506.06081}, 2015.

\end{thebibliography}

\appendix
\section{Supporting lemmata}

\begin{lemma}[Hoffman, Wielandt~\cite{bhatia1987perturbation}]
\label{lem:lowtrace}
Let $A$ and $B$ be two PSD $n \times n$ matrices. Also let $A$ be full rank. 
Then,
\begin{equation}\label{eq:lowtrace} 
\trace(AB) \geq \sigma_{\min}(A) \trace(B). 
\end{equation}
\end{lemma}


The following lemma shows that $\dist$, in the factor $U$ space, upper bounds the Frobenius norm distance in the matrix $X$ space.
\begin{lemma}\label{lem:X_Xr_bound}
Let $\X = \U\U^\top$ and $\Xor = \Uo_r \U_r^{\star\top}$ be two $n \times n$ rank-$r$ PSD matrices. 
Let $\dist(\U, \Uo_r)  \leq \rho \sigma_{r}(\Uo_r)$, for some rotation matrix $R_U^\star$ and constant $\rho > 0$. 
Then,
\begin{align*}
\|\X -\Xor\|_F \leq (2 + \rho) \rho \cdot  \|\Uo_r\|_2 \cdot \sigma_r(\Uo_r).
\end{align*}
\end{lemma} 

\begin{proof}
By substituting $\X = \U\U^\top$ and $\Xor = \Uo_r\U_r^{\star\top}$ in $\|\X -\Xor\|_F$, we have:
\begin{align*}
\|\X -\Xor\|_F &= \|\U\U^\top - \Uo_r\U_r^{\star\top}\|_F \\ 
				   &\stackrel{(i)}{=} \| \U \U^\top - \Uorr\U^\top + \Uorr \U^\top -\Uorr (\Uorr)^\top\|_F\\
				   &\stackrel{(ii)}{\leq} \dist(\U, \Uo_r) \cdot \|\U\|_2 + \dist(\U, \Uo_r) \cdot \|\Uo_r\|_2 \\
				   &\stackrel{(iii)}{\leq}  (1 + \rho) \|\Uo\|_2 \cdot \dist(\U, \Uo_r) + \dist(\U, \Uo_r) \cdot \|\Uo_r\|_2 \\
				   &= (2+\rho) \cdot \dist(\U, \Uo_r) \cdot \|\Uo_r\|_2 \\
				   &\stackrel{(iv)}{\leq} (2+\rho)\rho \cdot \|\Uo_r\|_2 \cdot \sigma_r(\Uo_r)
\end{align*} where $(i)$ is due to the orthogonality $R_U^{\star\top} R_U^\star = I_{r \times r}$, 
$(ii)$ is due to the triangle inequality, the Cauchy-Schwarz inequality and the fact that spectral norm is invariant w.r.t. orthogonal transformations and, $(iii)$ is due to the following sequence of inequalities, based on the hypothesis of the lemma:
\begin{align*}
\|\U\|_2	- \|\Uo_r\|_2 \leq \|\U - \Uorr\|_2 \leq \dist(\U, \Uo_r) \leq \rho \sigma_{r}(\Uo_r)
\end{align*} and thus $\|\U\|_2 \leq (1 + \rho)\cdot \|\Uo_r\|_2$. The final inequality $(iv)$ follows from the hypothesis of the lemma.
\end{proof}


The following lemma connects the spectrum of $U$ to $\Uo_r$ under the initialization assumptions.

\begin{lemma}\label{lem:sigma_bounds}
Let $\U$ and $\Uo_r$ be $n \times r$ matrices such that $\dist(\U, \Uo_r) \leq  \rho \sigma_r\left(\Uo_r\right)$, for $\rho=\tfrac{1}{100} \tfrac{\sigma_r(\Xo)}{\sigma_1(\Xo)}$. 
Withal, define $\Xo_r = \Uo_r\U_r^{\star\top}$. 
Then, the following bounds hold true: 
\begin{equation*}
\left(1-\sfrac{1}{100}\right) \sigma_1(\Uo_r) \leq \sigma_1(\U) \leq \left(1+\sfrac{1}{100}\right) \sigma_1(\Uo_r), 
\end{equation*}
\begin{equation*}
\left(1-\sfrac{1}{100}\right) \sigma_r(\Uo_r) \leq \sigma_r(\U) \leq \left(1+\sfrac{1}{100}\right) \sigma_r(\Uo_r). 
\end{equation*} 
Moreover, by definition of $\tau(\V) := \tfrac{\sigma_r(\V)}{\sigma_1(\V)}$ for some $\V$ matrix, we also observe:
\begin{equation*}
\tau(\U) \leq \tfrac{101}{99} \cdot \tau(\Uo_r) \quad \text{and} \quad \tau(\X) \leq \left(\tfrac{101}{99}\right)^2 \cdot \tau(\Xo_r).
\end{equation*} 
\end{lemma}

\begin{proof}
Using the norm ordering $\|\cdot\|_2 \leq \|\cdot\|_F $ and the Weyl's inequality for perturbation of singular values (Theorem 3.3.16 \cite{horn37topics}) we get, 
$$|\sigma_i(U) -\sigma_i(\Uo_r)| \leq \tfrac{1}{100  \tau(\Xo)}  \sigma_{r}(\Uo), ~1\leq i \leq r.$$ 
Then, the first two inequalities of the lemma follow by using triangle inequality and the above bound. 
For the last two inequalities, it is easy to derive bounds on condition numbers by combining the first two inequalities. 
\textit{Viz.}, 
\begin{align*}
\tau\left(\U\right) = \tfrac{\sigma_1\left(\U\right)}{\sigma_r\left(\U\right)} \leq \tfrac{1 + \sfrac{1}{100}}{1 - \sfrac{1}{100}} \cdot \tfrac{\sigma_1\left(\Uo_r\right)}{\sigma_r\left(\Uo_r\right)} \leq \tfrac{101}{99} \cdot \tau\left(\Uo_r\right),
\end{align*} 
while the last bound can be easily derived since $\tau\left(\Uo_r\right) = \sqrt{\tau\left(\Xo_r\right)}$. 
\end{proof}


The following lemma shows that $\dist$, in the factor $U$ space, lower bounds the Frobenius norm distance in the matrix $X$ space.

\begin{lemma}\label{lem:lower_bound_X_Xr}
Let $\X = \U\U^\top$ and $\Xor = \Uo_r \U_r^{\star\top}$ be two rank-$r$ PSD matrices. 
Let $\dist(\U , \Uo_r)  \leq  \rho \sigma_r\left(\Uo_r\right)$, for $\rho=\tfrac{1}{100} \tfrac{\sigma_r(\Xo)}{\sigma_1(\Xo)}$. Then,
\begin{align*}
\norm{\X -\Xor}_F^2 \geq \tfrac{3\sigma_{r}(\Xo)}{4}  \dist(\U , \Uo_r)^2.
\end{align*}
\end{lemma}

\begin{proof}
This proof largely follows the arguments for Lemma 5.4 in~\cite{tu2015low}, 
from which we know that 
\begin{align} 
||\X -\Xor||_F^2 \geq 2(\sqrt{2} -1) \sigma_{r}(\Xo)\dist(\U , \Uo_r)^2 . \label{proofs1:eq_12}
\end{align}
Hence, $\norm{\X -\Xor}_F^2 \geq \tfrac{3\sigma_{r}(\Xo)}{4}  \dist(\U , \Uo_r)^2$, for the given value of $\rho$.
\end{proof}

The following lemma shows equivalence between various step sizes used in the proofs.

\begin{lemma}\label{lem:diff_eta}
Let $\X^0 = \U^0 \U^{0 \top}$ and $X = \U\U^\top$ be two $n \times n$ rank-$r$ PSD matrices such that 
$\dist(\U, \Uo_r) \leq \dist(\U^0, \Uo_r)  \leq \rho \sigma_{r}(\Uo_r)$, where $\rho= \tfrac{1}{100} \cdot \tfrac{\sigma_r(\Xo)}{\sigma_1(\Xo)}$. 
Define the following step sizes: 
\begin{itemize}
\item[$(i)$] $\eta = \tfrac{1}{16 (M \|\X^0\|_2 + \|\gradf(\X^0)\|_2)}$, 
\item[$(ii)$] $\weta = \tfrac{1}{16 (M \|\X\|_2 + \|\gradf(\X)Q_U Q_U^\top\|_2)}$, and
\item[$(iii)$] $\eta^\star = \tfrac{1}{16 (M\|\Xo\|_2 + \|\gradf(\Xo)\|_2)}$. 
\end{itemize} 
Then, $  \weta \geq \tfrac{5}{6}\eta $ holds. Moreover, assuming  $\|\Xo -\Xo_r\|_F \leq \frac{\sigma_r(\Xo)}{100} \sqrt{\tfrac{\sigma_r(\Xo)}{\sigma_1(\Xo)}}$, the following inequalities hold:
\begin{align*} 
\tfrac{10}{11} \eta^\star \leq \eta \leq \tfrac{11}{10}\eta^\star 
\end{align*}
\end{lemma}

\begin{proof}
By the assumptions of this lemma and based on Lemma~\ref{lem:sigma_bounds}, we have, $\sfrac{98}{100} \norm{\Xo}_2 \leq\norm{\X^0}_2  \leq \sfrac{103}{100}\norm{\Xo}_2$; similarly $\sfrac{98}{100} \norm{\Xo}_2 \leq \norm{\X}_2  \leq \sfrac{103}{100}\norm{\Xo}_2$. Hence, we can combine these two set of inequalities to obtain bounds between $\X^0$ and $\X$, as follows:
\begin{align*}
\tfrac{98}{103}\norm{\X^0}_2 \leq \norm{\X}_2  \leq  \tfrac{103}{98}\norm{\X^0}_2.
\end{align*}

To prove the desiderata, we show the relationship between the gradient terms $\| \gradf (X) Q_{U} Q_{U}^\top \|_2$, $\| \gradf(X^0)\|_2$ and $\| \gradf(\Xo_r)\|_2$. In particular, for the case $\weta \geq \tfrac{5}{6}\eta$, we have:
\begin{align*}
\| \gradf (X) Q_{U} Q_{U}^\top \|_2 \leq \| \gradf(X)\|_2 &\stackrel{(i)}{\leq}  \| \gradf(X) -\gradf(\X^0)\|_2 + \| \gradf(X^0)\|_2 \\
&\stackrel{(ii)}{\leq} M\|\X -\X^0\|_F +  \| \gradf(X^0)\|_2 \\
&\stackrel{(iii)}{\leq} M \|\X -\Xo_r\|_F + M \|\X^0 -\Xo_r\|_F+  \| \gradf(X^0)\|_2   \\
&\stackrel{(iv)}{\leq} 2 M (2 + \rho)\rho \|\Uo_r\|_2 \cdot \sigma_{r}(\Uo_r) +  \| \gradf(X^0)\|_2 \\
&\stackrel{(v)}{\leq} 2M\cdot(2 + \tfrac{1}{100}) \cdot \tfrac{1}{100} \|\Xo\|_2 + \| \gradf(X^0)\|_2 \\
&\leq \tfrac{M}{20} \|\X^0\|_2 + \| \gradf(X^0)\|_2
\end{align*} where $(i)$ follows from the triangle inequality, $(ii)$ is due to the smoothness assumption, $(iii)$ is due to the triangle inequality, $(iv)$ follows by applying Lemma~\ref{lem:X_Xr_bound} on the first two terms on the right hand side and, $(v)$ is due to the fact $\|\Uo_r\|_2 \cdot \sigma_{r}(\Uo_r) \leq \|\Xo\|_2$ and by substituting $\rho = \tfrac{1}{100} \cdot \tfrac{\sigma_r(\Xo)}{\sigma_1(\Xo)} \leq \tfrac{1}{100}$. Last inequality follows from $\sfrac{98}{100} \norm{\Xo}_2 \leq\norm{\X^0}_2$. Hence, using the above bounds in step size selection, we get
\begin{align*}
\weta = \tfrac{1}{16 (M \|\X\|_2 + \|\gradf(\X)Q_U Q_U^\top\|_2)} \stackrel{(i)}{\geq} \tfrac{1}{16 \left(\tfrac{6M}{5} \|\X^0\|_2 + \| \gradf(X^0)\|_2\right)} \geq \tfrac{5}{6} \eta,
\end{align*} where $(i)$ is based also on the bound $\|\X\|_2 \leq \tfrac{103}{98}\|\X^0\|_2$.


Similarly we show the bound $\tfrac{10}{11} \eta^\star \leq \eta \leq \tfrac{11}{10}\eta^\star$. First observe that,
\begin{align*}
\| \gradf(X^0)\|_2 &\leq  \| \gradf(\Xo_r) -\gradf(\X^0)\|_2 + \| \gradf(\Xo_r)\|_2 \\
&\leq M\|\Xo_r -\X^0\|_F +  \| \gradf(\Xo_r)\|_2 \\
&\stackrel{(i)}{\leq} M\|\Xo_r -\X^0\|_F  + M\|\Xo -\Xo_r \|_F+  \| \gradf(\Xo)\|_2 \\
&\leq  M (2+\rho)\rho \cdot \|\Uo_r\|_2 \cdot  \sigma_{r}(\Uo_r) +  \frac{1}{100} M \sigma_{r}(\Xo_r)+ \| \gradf(\Xo)\|_2\\
&\leq \frac{4}{100}M\|\Xo\|_2 +\| \gradf(\Xo)\|_2.
\end{align*}  Combining the above bound with $\sfrac{98}{100} \norm{\Xo_r}_2 \leq\norm{\X^0}_2  \leq \sfrac{103}{100}\norm{\Xo_r}_2$ gives, $\eta \geq \frac{10}{11} \eta^\star$. Similarly we can show the other bounds.
\end{proof}

\section{Main lemmas for the restricted strong convex case}{\label{sec:strong_case_proof}}
In this section, we present proofs for the main lemmas used in the proof of Theorem~\ref{thm:convergence_main}, in Section~\ref{sec:proofs}. 

\subsection{Proof of Lemma \ref{lem:gradX,X-X_r_ bound_sc}}\label{sec:supp_sc1}
Here, we prove the existence of a non-trivial lower bound for $\ip{\gradf(\X)}{\X -\Xo_r}$. 
Our proof differs from the standard convex gradient descent proof (see ~\cite{nesterov2004introductory}), as we need to analyze updates without any projections. 
Our proof technique constructs a pseudo-iterate to obtain a bigger lower bound than the error term in Lemma~\ref{lem:DD_bound_sc}. 
Here, the nature of the step size plays a key role in achieving the bound. 

Let us abuse our notation and define $\Up = U - \weta \gradf(\X) U$ and $\Xp =\Up \U^{\scalebox{0.7}{+} \top}$.  
Observe that we use the surrogate step size $\weta$, where according to Lemma \ref{lem:diff_eta} satisfies $\weta \geq \tfrac{5}{6} \eta$. 
By smoothness of $\f$, we get:
\begin{align}
\f(\X) &\geq \f(\Xp) -\ip{\gradf(\X)}{\Xp -\X} - \tfrac{M}{2} \norm{\Xp -\X}_F^2 \nonumber \\ 
&\stackrel{(i)}{\geq} \f(\Xo) -\ip{\gradf(\X)}{\Xp -\X} - \tfrac{M}{2} \norm{\Xp -\X}_F^2 , \label{eq:gradX,X-X_r_00}
 \end{align}
where $(i)$ follows from optimality of $\Xo$ and since $\Xp$ is a feasible point $(\Xp \succeq 0)$ for problem~\eqref{intro:eq_00}. 
Further, note that $\Xor$ is a PSD feasible point. 
By smoothness of $f$, we also get
\begin{align}
f(\Xor) &\leq f(\Xo) + \ip{\gradf(\Xo)}{\Xor -\Xo} + \tfrac{M}{2}\|\Xor -\Xo\|_F^2 \nonumber \\
&\stackrel{(i)}{=} f(\Xo)+ \tfrac{M}{2}\|\Xor -\Xo\|_F^2, \label{eq:gradX,X-X_r_01}
\end{align} 
where $(i)$ is due to KKT conditions \cite{boyd2004convex}: 
since $\gradf(\Xo)$ is orthogonal to $\Xo$, it is also orthogonal to the $n-r$ bottom eigenvectors of $\Xo$. 
\textit{Viz.}, $\ip{\gradf(\Xo)}{\Xor -\Xo} = 0$. 
Finally, since $\text{rank}(\Xor) = r$, by the $(\rscm, r)$-restricted strong convexity of $\f$, we get, 
\begin{align} 
\f(\Xor) \geq \f(\X) +\ip{\gradf(\X)}{\Xor-\X} +\tfrac{m}{2}\norm{\Xor -\X}_F^2 .\label{eq:gradX,X-X_r_02}
\end{align}
Combining equations~\eqref{eq:gradX,X-X_r_00},~\eqref{eq:gradX,X-X_r_01}, and~\eqref{eq:gradX,X-X_r_02}, we obtain: 
\begin{small}
\begin{align} 
\ip{\gradf(\X)}{\X-\Xor} \geq  \ip{\gradf(\X)}{\X -\Xp} -\tfrac{M}{2} \norm{\Xp -\X}_F^2+\tfrac{m}{2}\norm{\Xor -\X}_F^2 -\tfrac{M}{2}\|\Xor -\Xo\|_F^2. \label{eq:gradX,X-X_r_03} 
\end{align}
\end{small}
It is easy to verify that $\Xp = \X - \weta \gradf(\X) \X \Lambda - \weta \Lambda^\top \X \gradf(\X)$, where $\Lambda = I - \tfrac{\weta}{2} Q_U Q_U^\top \gradf(X) \in \R^{n \times n}$. 
Notice that, for step size $\widehat{\eta}$, we have 
\begin{align*}
\Lambda \succ 0, \quad  \|\Lambda\|_2 \leq  1+ \sfrac{1}{32}, \quad \text{and} \quad \sigma_{n}(\Lambda) \geq  1- \sfrac{1}{32}.
\end{align*} 

Substituting the above in~\eqref{eq:gradX,X-X_r_03}, we obtain:
\begin{align*}
 \ip{\gradf(\X)}{\X-\Xor} &- \tfrac{m}{2}\norm{\Xor -\X}_F^2 + \tfrac{M}{2}\|\Xor -\Xo\|_F^2  \\
 								 &\stackrel{(i)}{\geq}  2\weta \ip{\gradf(\X)}{\gradf(X)X \Lambda} - \tfrac{M}{2} \norm{ 2 \weta \gradf(X)X\Lambda}_F^2 \\
 								 &= 2\weta \trace( \gradf(X) \gradf(X)X \Lambda) - 2M\weta^2 \norm{  \gradf(X)X\Lambda }_F^2 \\
 								 &\stackrel{(ii)}{\geq} 2 \weta \trace( \gradf(X) \gradf(X)X) \cdot \sigma_{n}(\Lambda) - 2 M \weta^2 \norm{  \gradf(X)U }_F^2 \|U\|_2^2 \|\Lambda\|_2^2 \\
 								 &\geq \tfrac{31\cdot\weta}{16} \|\gradf(X) U\|_F^2 - 2 M \weta^2 \cdot \left(\tfrac{33}{32}\right)^2 \cdot \norm{  \gradf(X)U }_F^2 \|U\|_2^2  \\  
 								 &=  \tfrac{31\cdot\weta}{16} \weta \|\gradf(X) U\|_F^2 \left(1  - 2 M \weta \left(\tfrac{33}{32}\right)^2 \cdot \tfrac{16}{31} \cdot \|X\|_2 \right) \\
 								 &\stackrel{(iii)}{\geq} \tfrac{18 \weta}{10} \|\gradf(X) U\|_F^2,
 \end{align*} 
where $(i)$ follows from symmetry of $\gradf(X)$ and $\X$, and $(ii)$ follows from 
 \begin{align*}
 \trace(\gradf(X)\gradf(X)X\Lambda) &= \trace(\gradf(X)\gradf(X)UU^\top) -\frac{\eta}{2}\trace(\gradf(X)\gradf(X)UU^\top \gradf(X)) \\ &\geq (1 -\frac{\eta}{2} \|Q_U Q_U^\top\gradf(X)\|_2) \|\gradf(\X) U\|_F^2  \\ &\geq  (1- \sfrac{1}{32}) \|\gradf(\X) U\|_F^2.
 \end{align*}  
Finally, $(iii)$ follows by observing that $\weta \leq \tfrac{1}{16M \|\X\|_2}$. 
Thus, we achieve the desiderata:
\begin{align*}   
\ip{\gradf(\X)}{\X-\Xor}  \geq \tfrac{18 \weta}{10} \|\gradf(X) U\|_F^2 +\tfrac{m}{2}\norm{\Xor -\X}_F^2 -\tfrac{M}{2}\|\Xo -\Xor\|_F^2.
\end{align*} 
This completes the proof.


\subsection{Proof of Lemma \ref{lem:DD_bound_sc}}\label{sec:supp_sc2}

We lower bound $\ip{\gradf(X) }{ \Delta \Delta^\top}$ as follows:
 \begin{align}
\ip{\gradf(X) }{ \Delta \Delta^\top} &\stackrel{(i)}{=} \ip{Q_{\Delta} Q_{\Delta}^\top \gradf(\X) }{ \Delta \Delta^\top} \nonumber \\
&\geq -  \left|\trace\left(Q_{\Delta} Q_{\Delta}^\top \gradf(\X) \Delta \Delta^\top\right)\right| \nonumber \\ 
&\stackrel{(ii)}{\geq} - \| Q_{\Delta} Q_{\Delta}^\top \gradf(\X) \|_2 \trace( \Delta \Delta^\top)  \nonumber \\
&\stackrel{(iii)}{\geq} - \left( \|Q_{U} Q_{U}^\top\gradf(\X)\|_2+  \|Q_{\Uo_r} Q_{\Uo_r}^\top\gradf(X)\|_2 \right)  \dist(\U, \Uo_r)^2 .   \label{proofsr1:eq_11}
\end{align}
Note that  $(i)$ follows from the fact $\Delta =Q_{\Delta} Q_{\Delta}^\top \Delta$ and $(ii)$ follows from $|\trace(AB)| \leq \|A\|_2 \trace(B)$, for PSD matrix $B$ (Von Neumann's trace inequality~\cite{mirsky1975trace}). For the transformation in $(iii)$, we use that fact that the column space of $\Delta$, $\text{\textsc{Span}}(\Delta)$, is a subset of $\text{\textsc{Span}}(\U \cup \Uo_r)$, as $\Delta$ is a linear combination of $\U$ and $\Uorr$. 

To bound the first term in equation~\eqref{proofsr1:eq_11}, we observe:
\begin{small}
\begin{align}
&\|Q_{U} Q_{U}^\top\gradf(\X)\|_2 \cdot \dist(\U, \Uo_r)^2 \\ 
&\quad \quad \stackrel{(i)}{=} \weta \cdot 16 \left(M\|X\|_2 +\|Q_{U} Q_{U}^\top \gradf(X)\|_2 \right) \cdot \|Q_{U} Q_{U}^\top\gradf(\X)\|_2 \cdot \dist(\U, \Uo_r)^2  \nonumber \\
 &\quad \quad = \weta \left( 16 \underbrace{M\|X\|_2 \|Q_{U} Q_{U}^\top\gradf(\X)\|_2 \cdot \dist(\U, \Uo_r)^2}_{:=A} + 16 \|Q_{U} Q_{U}^\top\gradf(\X)\|_2^2 \cdot  \dist(\U, \Uo_r)^2  \right) \nonumber
\end{align} 
\end{small}
At this point, we desire to introduce strong convexity parameter $m$ and condition number $\kappa$ in our bound. In particular, to bound term $A$, we observe that $\|Q_{U} Q_{U}^\top\gradf(\X)\|_2 \leq \tfrac{m \sigma_r(\X)}{40}$ or $\|Q_{U} Q_{U}^\top\gradf(\X)\|_2 \geq \tfrac{m \sigma_r(\X)}{40}$.
This results into bounding $A$ as follows:
\begin{small}
\begin{align*}
M\|X\|_2 &\|Q_{U} Q_{U}^\top\gradf(\X)\|_2 \cdot \dist(\U, \Uo_r)^2 \\ &\leq \max \left\{\tfrac{16\cdot \weta \cdot M \|X\|_2 \cdot m \sigma_r(\X)}{40} \cdot \dist(\U, \Uo_r)^2 ,~ \weta \cdot 16 \cdot 40 \kappa \tau(\X) \|Q_{U} Q_{U}^\top\gradf(\X)\|_2^2 \cdot  \dist(\U, \Uo_r)^2 \right\} \\
																				 &\leq \tfrac{16\cdot \weta \cdot M \|X\|_2 \cdot m \sigma_r(\X)}{40} \cdot \dist(\U, \Uo_r)^2 + \weta \cdot 16 \cdot 40 \kappa \tau(\X) \|Q_{U} Q_{U}^\top\gradf(\X)\|_2^2 \cdot  \dist(\U, \Uo_r)^2.
\end{align*}
\end{small}
Combining the above inequalities, we obtain:
\begin{small}
\begin{align}
\|Q_{U} Q_{U}^\top\gradf(\X)\|_2 &\cdot \dist(\U, \Uo_r)^2 \\ 
											   &\stackrel{(i)}{\leq}  \tfrac{m \sigma_{r}(\X)}{40} \cdot \dist(\U, \Uo_r)^2 + (40 \kappa \tau(\X)+1) \cdot  16 \cdot \weta \|Q_{U} Q_{U}^\top\gradf(\X)\|_2^2 \cdot \dist(\U, \Uo_r)^2  \nonumber \\
											   &\stackrel{(ii)}{\leq}  \tfrac{m \sigma_{r}(\X)}{40} \cdot \dist(\U, \Uo_r)^2 + (41 \kappa \tau(\Xo_r)+1) \cdot  16 \cdot \weta \|Q_{U} Q_{U}^\top\gradf(\X)\|_2^2 \cdot (\rho')^2 \sigma_{r}(\Xo_r)  \nonumber \\
 											   &\stackrel{(iii)}{\leq}  \tfrac{m \sigma_{r}(\X)}{40} \cdot \dist(\U, \Uo_r)^2 + 16 \cdot 42 \cdot \weta \cdot \kappa \tau(\Xo_r) \cdot \|\gradf(\X)U\|_F^2 \cdot \tfrac{11(\rho')^2}{10}   \nonumber \\ 
 											   &\stackrel{(iv)}{\leq} \tfrac{m \sigma_{r}(\X)}{40} \cdot \dist(\U, \Uo_r)^2 + \tfrac{2\weta}{25} \cdot \|\gradf(\X)U\|_F^2,   \label{proofsr1:eq_12}
\end{align}
\end{small}
where $(i)$ follows from $\weta \leq \tfrac{1}{16 M \|\X\|_2}$, 
$(ii)$ is due to Lemma~\ref{lem:sigma_bounds} and bounding $\dist(\U, \Uo_r) \leq \rho' \sigma_{r}(\Uo_r)$ by the hypothesis of the lemma, 
$(iii)$ is due to $\sigma_{r}(\Xo) \leq 1.1 \sigma_{r}(\X)$ by Lemma~\ref{lem:sigma_bounds} and due to the facts $\sigma_{r}(\X)\|Q_{U} Q_{U}^\top\gradf(X)\|_2^2 \leq  \|U^\top\gradf(X)\|_F^2$ and $(41 \kappa \tau(\Xo_r)+1) \leq 42 \kappa \tau(\Xo_r)$. 
Finally, $(iv)$ follows from substituting $\rho'$ and using Lemma~\ref{lem:sigma_bounds}.

Next, we bound the second term in equation~\eqref{proofsr1:eq_11}:
\begin{align}
&\|Q_{\Uor} Q_{\Uor}^\top\gradf(X)\|_2 \cdot \dist(\U, \Uo_r)^2 \\ &\stackrel{(i)}{\leq}  \|\gradf(\X) -\gradf( \Xo) \|_2 \cdot \dist(\U, \Uo_r)^2 \nonumber \\
&\leq  \|\gradf(\X) -\gradf( \Xo) \|_F \cdot \dist(\U, \Uo_r)^2 \nonumber \\
&\stackrel{(ii)}{\leq} M \left(\|\X -\Xo_r\|_F +\|\Xo -\Xo_r\|_F\right) \cdot \dist(\U, \Uo_r)^2 \nonumber \\
&\stackrel{(iii)}{\leq} M  (2+\rho') \cdot \rho' \cdot \|\Uo_r\|_2 \cdot \sigma_r(\Uo_r) \cdot \dist(\U, \Uo_r)^2 +  M\|\Xo -\Xo_r\|_F \cdot  \dist(\U, \Uo_r)^2  \nonumber \\
&\stackrel{(iv)}{\leq}  M  (2+\rho') \|\Uo_r\|_2  \tfrac{1}{100\kappa \tau(\Uo_r)} \sigma_r(\Uo_r) \cdot \dist(\U, \Uo_r)^2  +    M\|\Xo -\Xo_r\|_F \cdot  \dist(\U, \Uo_r)^2  \nonumber \\
&\leq \tfrac{m  \sigma_{r}(\Xo)}{40}   \dist(\U, \Uo_r)^2 +  M\|\Xo -\Xo_r\|_F \cdot \dist(\U, \Uo_r)^2, \label{proofsr1:eq_13}
\end{align}
where $(i)$ follows from $\gradf(\Xo)\Xo =0$, $(ii)$ is due to smoothness of $\f$ and $(iii)$ follows from Lemma~\ref{lem:X_Xr_bound}. 
Finally $(iv)$ follows from  $\dist(\U, \Uo_r) \leq \rho' \sigma_{r}(\Uo_r)$ and substituting $\rho' =  \frac{1}{100 \kappa \tau(\Uo_r)}$.

Substituting \eqref{proofsr1:eq_12}, \eqref{proofsr1:eq_13} in \eqref{proofsr1:eq_11}, we get:
\begin{small}
\begin{align*}
\ip{\gradf(X) }{ \Delta \Delta^\top} \geq - \left(\tfrac{2\weta}{25 } \|\gradf(\X) U \|_F^2 + \tfrac{m\sigma_{r}(\Xo)}{20} \cdot \dist(\U, \Uo_r)^2   + M\|\Xo -\Xo_r\|_F \cdot \dist(\U, \Uo_r)^2\right)
\end{align*}
\end{small}
This completes the proof.

\subsection{Proof of Lemma~\ref{lem:gradU,U-U_r_ bound}}\label{sec:supp_sc3}
Recall $\Up =U -\eta\gradf(X)U$. First we rewrite the inner product as shown below.
\begin{align}
\frac{1}{\eta}\ip{U -\Up}{\U -\Uorr} &= \ip{\gradf(X) \U}{\U -\Uorr} \nonumber \\
&= \ip{\gradf(X) }{\X -\Uorr \U^\top} \nonumber \\
&=\frac{1}{2}\ip{\gradf(X) }{\X -\Xo_r}  + \ip{\gradf(X) }{\frac{1}{2}(\X + \Xo_r) -\Uorr \U^\top} \nonumber \\
&=\frac{1}{2}\ip{\gradf(X) }{\X -\Xo_r}  + \frac{1}{2} \ip{\gradf(X) }{\Delta \Delta^\top}, \label{proofsr1:eq_09}
\end{align} which follows by adding and subtracting $\tfrac{1}{2}\Xo_r$. 

Let, $\weta =\tfrac{1}{16 (M\|\X\|_2 + \|\gradf(\X)Q_U Q_U^\top\|_2)}$. 
Using Lemmas \ref{lem:gradX,X-X_r_ bound_sc} and \ref{lem:DD_bound_sc}, we have:
\begin{align*}
& \ip{\gradf(X) \U}{\U -\Uorr} \\ 
&\geq \tfrac{9\weta}{10} \cdot  \|\gradf(X) U\|_F^2+\tfrac{m}{4}\norm{\X -\Xor}_F^2 - \tfrac{M}{4}\norm{\Xo -\Xor}_F^2 \\
& \quad \quad \quad- \tfrac{1}{2}\left(  \tfrac{2\weta}{25 } \cdot \|\gradf(\X) U \|_F^2 + \tfrac{ m\sigma_{r}(\Xo)}{20}  \cdot \dist(\U, \Uo_r)^2   + M\|\Xo -\Xo_r\|_F \cdot  \dist(\U, \Uo_r)^2 \right) \nonumber \\
&=  \left(\tfrac{9}{10}- \tfrac{1}{25}\right) \cdot \weta  \|\gradf(X) U\|_F^2 - \tfrac{M}{4}\norm{\Xo -\Xor}_F^2 \nonumber \\
& \quad \quad \quad + \tfrac{m}{4}\left(\norm{\X -\Xor}_F^2 -  \tfrac{4\sigma_{r}(\Xo) }{25} \cdot \dist(\U, \Uo_r)^2 - 2 \kappa \cdot \|\Xo -\Xo_r\|_F \cdot \dist(\U, \Uo_r)^2 \right) \nonumber
\end{align*}
\begin{align*}
&\stackrel{(i)}{\geq}  \tfrac{4\weta}{5} \cdot \|\gradf(X) U\|_F^2 - \tfrac{M}{4}\norm{\Xo -\Xor}_F^2 \nonumber \\
& \quad \quad \quad + \tfrac{m}{4}\left(\norm{\X -\Xor}_F^2 -  \tfrac{4\sigma_{r}(\Xo) }{25} \cdot \dist(\U, \Uo_r)^2 - \tfrac{\sigma_{r}(\Xo) }{50}\cdot \dist(\U, \Uo_r)^2 \right) \nonumber \\
&\stackrel{(ii)}{\geq } \tfrac{4\weta}{5} \cdot \|\gradf(X) U\|_F^2 + \tfrac{3m}{20} \cdot  \sigma_{r}(\Xo) \cdot \dist(\U, \Uo_r)^2 - \tfrac{M}{4}\norm{\Xo -\Xor}_F^2
\end{align*}
where $(i)$ follows from  $\|\Xo -\Xo_r\| \leq \frac{\sigma_r(\Xo)}{100 \kappa^{1.5}}\frac{\sigma_r(\Xo)}{\sigma_1(\Xo)} \leq \frac{\sigma_r(\Xo)}{100 \kappa^{1.5}} \leq \frac{\sigma_r(\Xo)}{100 \kappa}$ and $(ii)$ follows from Lemma~\ref{lem:lower_bound_X_Xr}. Finally the result follows from $\weta \geq \tfrac{5}{6}\eta$ from Lemma~\ref{lem:diff_eta}.

\section{Main lemmas for the smooth case}{\label{sec:smooth_case_proof}}
In this section, we present the main lemmas, used in the proof of Theorem~\ref{thm:smooth_inexact} in Section~\ref{sec:proofs}. 
First, we present a lemma bounding the error term $\left\langle \gradf(\X), \Delta \Delta^\top \right\rangle$, that appears in eq. \eqref{proofs_eq:jc2}. 

\begin{lemma}\label{lem:DD_bound}
Let $f$ be $M$-smooth and $\X = \U\U^\top$; also, define $\Delta := \U - \Uorr$. Then, for $\dist(\U, \Uo_r) \leq \rho \sigma_r\left(\Uo_r\right)$ and $\rho=\tfrac{1}{100} \frac{\sigma_r(\Xo)}{\sigma_1(\Xo)}$, the following bound holds true:
\begin{align*}
\left\langle \gradf(\X), \Delta \Delta^\top \right\rangle \leq \tfrac{1}{40} \|\gradf(X)U\|_2 \cdot \dist(\U, \Uo_r).
\end{align*} 
\end{lemma}

\begin{proof}
By the Von Neumann's trace inequality for PSD matrices, we know that $\trace(AB) \leq \trace(A) \cdot \|B\|_2$, for $A$ PSD matrix. 
In our context, we then have:
\begin{align}
\ip{\gradf(X) }{\Delta \Delta^\top} &\leq \|\gradf(X)Q_{\Delta} Q_{\Delta}^\top \|_2 \cdot \trace(\Delta \Delta^\top) \nonumber \\
&\stackrel{(i)}{\leq} \left(\|\gradf(X)Q_U Q_U^\top\|_2 + \|\gradf(X)Q_{\Uo_r} Q_{\Uo_r}^\top\|_2 \right) \cdot \dist(\U, \Uo_r)^2, \label{proofsjc:eq_07}
\end{align} 
where, 
$(i)$ is because $\Delta$ can be decomposed into the column span of $\U$ and $\Uo_r$, and the orthogonality of the rotational matrix $R_{U_r^\star}$. In sequence, we further bound the term $ \|\gradf(X)Q_{\Uo_r} Q_{\Uo_r}^\top\|_2 $ as follows:
\begin{align*}
\|\gradf(X)\Uo_r\|_2 &\stackrel{(i)}{\leq} \|\gradf(X)\U\|_2 +\|\gradf(X)\Delta\|_2 \\
&\stackrel{(ii)}{\leq} \|\gradf(X)\U\|_2 +\|\gradf(X)Q_{\Delta}Q_{\Delta}^\top\|_2\|\Delta\|_2 \\
&\stackrel{(iii)}{\leq} \|\gradf(X)\U\|_2 +\left( \|\gradf(X)Q_{U}Q_{U}^\top\|_2 +\|\gradf(X)Q_{\Uo_r}Q_{\Uo_r}^\top\|_2 \right) \|\Delta\|_2 \\
&\stackrel{(iv)}{\leq} \|\gradf(X)\U\|_2 +\left( \|\gradf(X)Q_{U}Q_{U}^\top\|_2 +\|\gradf(X)Q_{\Uo_r}Q_{\Uo_r}^\top\|_2 \right)\tfrac{1}{100} \sigma_{r}(\Uo_r) \\
&\stackrel{(v)}{\leq} \|\gradf(X)\U\|_2 +  \tfrac{1}{\left(1-\frac{1}{100}\right)} \cdot \tfrac{1}{100}\|\gradf(X)U\|_2 + \tfrac{1}{100} \|\gradf(X)\Uo_r\|_2 \\
&\leq \tfrac{102}{100}\|\gradf(X)\U\|_2 + \tfrac{1}{100} \|\gradf(X)\Uo_r\|_2.
\end{align*} 
where 
$(i)$ is due to triangle inequality on $\Uorr = \U - \Delta$, 
$(ii)$ is due to generalized Cauchy-Schwarz inequality; we denote as $\Q_\Delta\Q_\Delta$ the projection matrix on the column span of $\Delta$ matrix, 
$(iii)$ is due to triangle inequality and the fact that the column span of $\Delta$ can be decomposed into the column span of $\U$ and $\Uo_r$, by construction of $\Delta$, 
$(iv)$ is due to $$\|\Delta\|_2 \leq \dist(\U, \Uo_r) \leq \tfrac{1}{100} \tfrac{\sigma_r(\Xo)}{\sigma_1(\Xo)} \cdot \sigma_r(\Uo_r) \leq \tfrac{1}{100} \cdot \sigma_r(\Uo_r).$$ 
Finally, $(v)$ is due to the facts: $$\|\gradf(\X) \Uo_r\|_2 = \|\gradf(\X) \Q_{\Uo_r}\Q_{\Uo_r}^\top \Uo_r\|_2 \geq \|\gradf(\X) \Q_{\Uo_r}\Q_{\Uo_r}^\top\|_2 \cdot \sigma_r(\Uo_r),$$ and 
\begin{align*}
\|\gradf(\X) \U\|_2 &= \|\gradf(\X) \Q_{\U}\Q_{\U}^\top \U\|_2 \geq \|\gradf(\X) \Q_{\U}\Q_{\U}^\top\|_2 \cdot \sigma_r(\U) \nonumber \\ 
&\geq \|\gradf(\X) \Q_{\U}\Q_{\U}^\top\|_2 \cdot \left(1 - \tfrac{1}{100}\right) \cdot \sigma_r(\Uo_r),
\end{align*} by Lemma \ref{lem:sigma_bounds}. Thus:
\begin{align}\label{proofsjc:eq_067}
 \|\gradf(X)Q_{\Uo_r} Q_{\Uo_r}^\top\|_2  &\leq \tfrac{1}{\sigma_r(\Uo_r)}\|\gradf(X)\Uo_r\|_2 \nonumber \\ 
 &\leq \tfrac{1}{\sigma_r(\Uo_r)}\tfrac{102}{99} \|\gradf(X)\U\|_2 \nonumber \\ 
 &\leq \tfrac{101 \sigma_1(\Uo_r)}{100 \sigma_r(\Uo_r)}\tfrac{102}{99} \|\gradf(X)Q_U Q_U^\top\|_2, 
\end{align} and, combining with \eqref{proofsjc:eq_07}, we get
\begin{align*}
\ip{\gradf(X) }{\Delta \Delta^\top} &\leq \left(\tfrac{102 \cdot 101}{100 \cdot 99}+1\right) \cdot  \tfrac{\sigma_1(\Uo_r)}{\sigma_r(\Uo_r)} \cdot \|\gradf(X)Q_U Q_U^\top\|_2 \cdot \dist(\U, \Uo_r)^2 \\ &\leq \tfrac{1}{40} \|\gradf(X)U\|_2 \cdot \dist(\U, \Uo_r).
\end{align*}
The last inequality follows from $ \dist(\U, \Uo_r) \leq \tfrac{1}{100} \tfrac{\sigma_r(\Xo)}{\sigma_1(\Xo)} \cdot \sigma_r(\Uo_r) $. 
This completes the proof. 
\end{proof}

The following lemma lower bounds the term $\left\langle \gradf(\X), \Delta \Delta^\top \right\rangle$;
this result is used later in the proof of Lemma \ref{lem:ipbound_case1}.

\begin{lemma}{\label{lem:new_lemma_smooth}}
Let $\X = \U\U^\top$ and define $\Delta := \U - \Uorr$. 
Let $f(\Xp) \geq f(\Xor)$, where $\Xor$ is the optimum of the problem~\eqref{intro:eq_00}. 
Then, for $\dist(\U, \Uo_r) \leq \rho \sigma_r\left(\Uo_r\right)$, where $\rho=\tfrac{1}{100} \frac{\sigma_r(\Xo)}{\sigma_1(\Xo)}$, and $f$ being a $M$-smooth convex function, the following lower bound holds:
\begin{align*}
\ip{\gradf(X)}{\Delta \Delta^\top} \geq - \tfrac{\sqrt{2}}{\sqrt{2} - \tfrac{1}{100}} \cdot \tfrac{1}{100} \cdot \left|\ip{\gradf(X)}{\X -\Xor}\right|.
\end{align*}
\end{lemma}

\begin{proof}
Let the QR factorization of the matrix $\left[U ~~~\Uorr\right]_{n \times 2r}$ be $Q\cdot R$, where $Q$ is a $n \times 2r$ orthonormal matrix and $R$ is a $2r \times 2r$ invertible matrix (since $\left[U ~~~\Uorr\right]$ is assumed to be rank-$2r$). 
Further, let $\left[U ~~~\Uorr\right]^\dagger_{2r \times n}$ where $C^\dagger$ denotes the pseudo-inverse of matrix $C$.
It is obvious that $\left[U ~~~\Uorr\right]^\top \cdot \left(\left[U ~~~\Uorr\right]^\dagger\right)^\top = I_{2r \times 2r}$. 

Given the above, let us re-define some quantities w.r.t. $\left[U ~~~\Uorr\right]$, as follows
\begin{align*}
\Delta = \U - \Uorr = \left[U ~~~\Uorr\right]_{n \times 2r} \cdot \begin{bmatrix}
I_{r \times r} \\ -I_{r \times r}
\end{bmatrix}_{2r \times r}.
\end{align*}
Moreover, it is straightforward to justify that:
\begin{align*}
\X - \Xor = \left[U ~~~\Uorr\right]_{n \times 2r} \cdot 
\begin{bmatrix}
I_{r \times r} & 0_{r \times r} \\ 0_{r \times r} & -I_{r \times r}
\end{bmatrix} \cdot \left[U ~~~\Uorr\right]^\top_{2r \times n}
\end{align*}
Then, from the above, the two quantities $\X - \Xor$ and $\Delta$ are connected as follows: 
\begin{small}
\begin{align}
\left(\X - \Xor\right) \cdot \left(\left[U ~~~\Uorr\right]^\dagger\right)^\top \cdot 
\begin{bmatrix}
I \\ I
\end{bmatrix} &= \left[U ~~~\Uorr\right] \cdot 
\begin{bmatrix}
I & 0 \\ 0 & -I
\end{bmatrix} \cdot \underbrace{\left[U ~~~\Uorr\right]^\top \cdot \left(\left[U ~~~\Uorr\right]^\dagger\right)^\top}_{= I} \cdot 
\begin{bmatrix}
I \\ I
\end{bmatrix} \nonumber 
\end{align}
\end{small} 
which is equal to $\Delta$.
Then, the following sequence of (in)equalities holds true:
\begin{align}
    \ip{\gradf(X)}{\Delta \Delta^\top} 
    &\stackrel{(i)}{=} \ip{\gradf(X)}{\left(\X - \Xor\right) \cdot \left(\left[U ~~~\Uorr\right]^\dagger\right)^\top \cdot 
\begin{bmatrix}
I \\ I
\end{bmatrix} \cdot \Delta^\top} \nonumber \\
	&\stackrel{(ii)}{\geq} - \left| \trace\left( \underbrace{\gradf(X) \cdot \left(\X - \Xor\right)}_{=A} \cdot \underbrace{\left(\left[U ~~~\Uorr\right]^\dagger\right)^\top \cdot 
\begin{bmatrix}
I \\ I
\end{bmatrix} \cdot \Delta^\top}_{=B}\right) \right| \nonumber \\
	&\stackrel{(iii)}{\geq} - \left| \trace\left(\gradf(X) \cdot \left(\X - \Xor\right) \right) \right| \cdot \left\|\left(\left[U ~~~\Uorr\right]^\dagger\right)^\top \cdot 
\begin{bmatrix}
I \\ I
\end{bmatrix} \cdot \Delta^\top \right\|_2 \nonumber \\
	&\stackrel{(iv)}{\geq} - \left| \ip{\gradf(X)}{\X - \Xor} \right| \cdot \left\|\left(\left[U ~~~\Uorr\right]^\dagger\right)^\top \right\|_2 \cdot \left\|
\begin{bmatrix}
I \\ I
\end{bmatrix} \right\|_2 \cdot \|\Delta \|_2 \nonumber \\
	&\stackrel{(v)}{\geq} - \sqrt{2} \cdot \left| \ip{\gradf(X)}{\X - \Xor} \right| \cdot \left\|\left(\left[U ~~~\Uorr\right]^\dagger\right)^\top \right\|_2 \cdot \tfrac{1}{100} \cdot \sigma_r(\Uo_r) \nonumber \\
	&\stackrel{(vi)}{\geq} - \sqrt{2} \cdot \left| \ip{\gradf(X)}{\X - \Xor} \right| \cdot \tfrac{1}{\sqrt{2} - \tfrac{1}{100}} \cdot \tfrac{1}{\sigma_r(\Uo_r)} \cdot \tfrac{1}{100} \cdot \sigma_r(\Uo_r) \nonumber \\ 
	&=- \tfrac{\sqrt{2}}{\sqrt{2} - \tfrac{1}{100}} \cdot \tfrac{1}{100} \cdot \left| \ip{\gradf(X)}{\X - \Xor} \right|, \label{eq:paok_00}
\end{align}
where, $(i)$ follows by substituting $\Delta$, according to the discussion above, 
$(ii)$ follows from symmetry of $\gradf(X)$,
$(iii)$ follows from the Von Neumann trace inequality $\trace(AB) \leq \trace(A)||B||_2$, for a PSD matrix $A$; next, we show that $y^\top A y \geq 0$, $\forall y$ and $A := \gradf(X) \cdot \left(\X - \Xor\right)$, 
$(iv)$ is due to successive application of the Cauchy-Schwarz inequality, 
$(v)$ is due to $\left\|
\begin{bmatrix}
I & I
\end{bmatrix}^\top \right\|_2 = \sqrt{2}$ and $\|\Delta\|_2 \leq \dist(\U, \Uo_r) \leq \rho \cdot \sigma_r(\Uo_r) \leq \tfrac{1}{100} \cdot \sigma_r(\Uo_r)$,
$(vi)$ follows from the the following fact:
\begin{align*}
    \tfrac{1}{\left\|\left[U ~~~\Uorr\right]^\dagger\right\|_2} &=\sigma_r\left(\left[U ~~~\Uorr\right]\right) \\
    &=\sigma_r\left(\left[U ~~~\Uorr\right] - \left[\Uorr ~~~\Uorr\right] + \left[\Uorr ~~~\Uorr\right]\right) \\
    &=\sigma_r\left(\left[U - \Uorr ~~~0\right] + \left[\Uorr ~~~\Uorr\right]\right) \\
    &\stackrel{(i)}{\geq}  \sigma_r\left(\left[\Uorr ~~~\Uorr\right]\right) - \left\|\U - \Uorr\right\|_2 \\
    &\stackrel{(ii)}{=}  \sqrt{2} \cdot \sigma_r\left(\Uor\right) - \left\|\U - \Uorr\right\|_2 \\
    &\stackrel{(iii)}{\geq} \left(\sqrt{2} - \tfrac{1}{100}\right) \cdot \sigma_r(\Uo_r),
\end{align*}
where, $(i)$ follows from a variant of Weyl's inequality,
$(ii)$ is due to $\sigma_r\left(\left[\Uorr ~~~\Uorr\right]\right) = \sqrt{2} \cdot \sigma_r\left(\Uo_r\right)$, 
$(iii)$ follows from the assumption that  $\left\|\U - \Uorr\right\|_2 \leq \dist(\U, \Uo_r) \leq \tfrac{1}{100} \cdot \sigma_r\left(\Uo_r\right)$.
The above lead to the inequality:
\begin{align*}
\left\|\left(\left[U ~~~\Uorr\right]^\dagger\right)^\top\right\|_2 = \left\|\left[U ~~~\Uorr\right]^\dagger\right\|_2 \leq \tfrac{1}{\sqrt{2} - \tfrac{1}{100}} \cdot \tfrac{1}{\sigma_r(\Uo_r)}.
\end{align*}

In the above inequalities \eqref{eq:paok_00}, we used the fact that symmetric version of A is a PSD matrix, where $A:= \gradf(X)\Delta (\U+\Uorr)^\top = \gradf(X)\cdot (X-\Xor) $ is a PSD matrix, \emph{i.e.}, given a vector $y$, $y^\top \gradf(X)\cdot (X-\Xor) y \geq 0$. 
To show this, let $g(t) =f(X+t yy^\top)$ be a function from $\R \to \R$. 
Hence, $\nabla g(t) =\ip{\gradf(X+t yy^\top)}{yy^\top}$. 
Now, consider $g$ restricted to the level set $\{ t: f(X+t yy^\top) \leq f(\X)\}$. 
Note that, since $f$ is convex, this set is convex and further $X$ belongs to this set from the hypothesis of the lemma. Also $f(\Xor) \leq f(X+t yy^\top)$, for $t$ in this set from the optimality of $\Xor$.
Let $t^*$ be the minimizer of $g(t)$ over this set. 
Then, by convexity of $g$, $$\ip{\gradf(X)}{yy^\top}\cdot -t^*=\nabla g(0) \cdot (0-t^*) \geq g(0) -g(t^*) \geq 0.$$ 
Further, since $g(t^*) =f(X+t^* yy^\top) \geq f(\Xor)$, $X+t^* yy^\top -\Xor$ is orthogonal to $y$. 
Hence, $(X+t^* yy^\top -\Xor)y =0$. 
Combining this with the above inequality gives, $\ip{\gradf(X)}{(\X -\Xor)yy^\top} \geq 0$. This completes the proof.
\end{proof}

We next present a lemma for lower bounding the term $\ip{\gradf(\X)}{\X -\Xor}$. This result is used in the following Lemma \ref{lem:ipbound_case1}, where we bound the term $\ip{\gradf(X)U}{\U-\Uorr}$. 

\begin{lemma}\label{lem:gradX,X-X_r_ bound}
Let $f$ be a $M$-smooth convex function with optimum point $\Xo_r$. 
 Then, under the assumption that $f(\Xp) \geq f(\Xor)$, the following holds:
\begin{align*}
\ip{\gradf(\X)}{\X-\Xor}  \geq \tfrac{18\weta}{10} \|\gradf(X) U\|_F^2. 
\end{align*}
\end{lemma}

\begin{proof}
The proof follows much like the proof of the Lemma for strong convex case (Lemma~\ref{lem:gradX,X-X_r_ bound_sc}), except for the arguments used to bound equation~\eqref{eq:gradX,X-X_r_00}. 
For completeness, we here highlight the differences; in particular, we again have by smoothness of $f$:
\begin{align*}
\f(\X) &\geq \f(\Xp) -\ip{\gradf(\X)}{\Xp -\X} - \tfrac{M}{2} \norm{\Xp -\X}_F^2,
\end{align*}
where we consider the same notation with Lemma \ref{lem:gradX,X-X_r_ bound_sc}. 
By the assumptions of the Lemma, we have  $f(\Xp) \geq f(\Xor)$ and, thus, the above translates into:
\begin{align*}
\f(\X) &\geq \f(\Xor) -\ip{\gradf(\X)}{\Xp -\X} - \tfrac{M}{2} \norm{\Xp -\X}_F^2,
\end{align*}
hence eliminating the need for equation~\eqref{eq:gradX,X-X_r_01}. 
Combining the above and assuming just smoothness (\emph{i.e.}, the restricted strong convexity parameter is $m=0$), we obtain a simpler version of eq. \eqref{eq:gradX,X-X_r_03}:
\begin{align} 
\ip{\gradf(\X)}{\X-\Xor} \geq  \ip{\gradf(\X)}{\X -\Xp} -\tfrac{M}{2} \norm{\Xp -\X}_F^2.
\end{align} 
Then, the result easily follows by the same steps in Lemma \ref{lem:gradX,X-X_r_ bound_sc}.
\end{proof}


Next, we state an important result, relating the gradient step in the factored space $\Up -U$ to the direction to the optimum $U -\Uo$. 
The result borrows the outcome of Lemmas \ref{lem:DD_bound}-\ref{lem:gradX,X-X_r_ bound}.

\begin{lemma}\label{lem:ipbound_case1}
Let $\X = \U\U^\top$ and define $\Delta := \U - \Uorr$. 
Assume $f(\Xp) \geq f(\Xor)$ and $\dist(\U, \Uo_r) \leq \rho \sigma_r\left(\Uo_r\right)$, where $\rho=\tfrac{1}{100} \frac{\sigma_r(\Xo)}{\sigma_1(\Xo)}$.
For $f$ being a $M$-smooth convex function, the following descent condition holds for the $\U$-space:
\begin{align*}
\ip{\gradf(X)U}{\U-\Uorr} \geq \tfrac{\eta}{2} \cdot \|\gradf(X)U\|_F^2.
\end{align*}
\end{lemma}

\begin{proof} 
Expanding the term $\ip{\gradf(X)U}{\U-\Uorr}$, we obtain the equivalent characterization:
\begin{align}
\ip{\gradf(X) \U}{\U -\Uorr} &= \ip{\gradf(X) }{\X -\Uorr \U^\top} \nonumber \\
&=\tfrac{1}{2}\ip{\gradf(X) }{\X -\Xo_r}  + \ip{\gradf(X) }{\tfrac{1}{2}(\X + \Xo_r) -\Uorr \U^\top} \nonumber \\ 
&=\tfrac{1}{2}\ip{\gradf(X) }{\X -\Xo_r}  + \tfrac{1}{2} \ip{\gradf(X) }{\Delta \Delta^\top} \label{proofsjc:eq_06}
\end{align}
which follows by the definition of $\X$ and adding and subtracting $\tfrac{1}{2}\Xo_r$ term. 
By Lemma \ref{lem:gradX,X-X_r_ bound}, we can bound the first term on the right hand side as:
\begin{align}\label{eq:lemma:jpjc2}
\tfrac{1}{2}\ip{\gradf(\X)}{\X-\Xor}  \geq \tfrac{18\weta}{20} \cdot \|\gradf(X) U\|_F^2.
\end{align} 
Observe that $\ip{\gradf(\X)}{\X-\Xor} \geq 0$.
By Lemma \ref{lem:new_lemma_smooth}, we can lower bound the last term on the right hand side of \eqref{proofsjc:eq_06} as:
\begin{align}{\label{eq:lemma:jpjc22}}
\tfrac{1}{2} \ip{\gradf(X) }{\Delta \Delta^\top} \geq -\tfrac{\sqrt{2}}{\sqrt{2} - \tfrac{1}{100}} \cdot \tfrac{1}{200} \left|\ip{\gradf(\X)}{\X-\Xor} \right| = -\tfrac{\sqrt{2}}{\sqrt{2} - \tfrac{1}{100}} \cdot \tfrac{1}{200} \ip{\gradf(\X)}{\X-\Xor}.
\end{align}
Combining \eqref{eq:lemma:jpjc2} and \eqref{eq:lemma:jpjc22} in \eqref{proofsjc:eq_06}, we get:
\begin{align*}
\ip{\gradf(X) \U}{\U -\Uorr}  &\geq \tfrac{1}{2}\ip{\gradf(X) }{\X -\Xo_r}  -\tfrac{\sqrt{2}}{\sqrt{2} - \tfrac{1}{100}} \cdot \tfrac{1}{200} \ip{\gradf(\X)}{\X-\Xor} \\
									  &\geq \left(1 -\tfrac{\sqrt{2}}{\sqrt{2} - \tfrac{1}{100}} \cdot \tfrac{1}{100}\right) \cdot \tfrac{1}{2} \ip{\gradf(X) }{\X -\Xo_r} \\									  
									  &\geq \tfrac{98}{100} \cdot \tfrac{18\weta}{20} \cdot \|\gradf(X) U\|_F^2 \\
									  &\stackrel{(i)}{\geq} \tfrac{98}{100} \cdot \tfrac{18}{20} \cdot \tfrac{5\eta}{6} \cdot \|\gradf(X) U\|_F^2 \\
									  &\geq \tfrac{7\eta}{10} \cdot \|\gradf(X) U\|_F^2 \geq \tfrac{\eta}{2} \cdot \|\gradf(X) U\|_F^2,
\end{align*}  
where $(i)$ follows from $\weta \geq \tfrac{5}{6}\eta$ in Lemma~\ref{lem:diff_eta}. This completes the proof.
\end{proof}


We conclude this section with a lemma that proves that the distance $\dist(\U, \Uo_r)$ is non-increasing per iteration of \algo.
This lemma is used in the proof of sublinear convergence of \algo (Theorem \ref{thm:smooth_inexact}), in Section \ref{sec:proofs}.

\begin{lemma}\label{lem:gradf,D_bound1}
Let $\X = \U\U^\top$ and $\Xp = \Up\left(\Up\right)^\top$ be the current and next estimate of \algo. 
Assume $f$ is a $M$-smooth convex function such that $f(\Xp) \geq f(\Xor)$. 
Moreover, define $\Delta := \U - \Uorr$ and  $\dist(\U, \Uo_r) \leq \rho \sigma_r\left(\Uo_r\right)$, where $\rho=\tfrac{1}{100} \frac{\sigma_r(\Xo)}{\sigma_1(\Xo)}$. 
Then, the following inequality holds:
\begin{align*}
\dist(\Up, \Uo_r) \leq \dist(\U, \Uo_r).
\end{align*} This further implies $\dist(\U, \Uo_r) \leq \dist(\U^0, \Uo_r)$ for any estimate $\U$ of \algo.
\end{lemma}

\begin{proof}
Let $R_U^\star = \arg\min_{R \in \mathcal{O}} \|\U - \Uo_r R\|_F^2.$ Expanding $\dist(\Up, \Uo_r)^2$, we obtain:
\begin{align}
\dist(\Up, \Uo_r)^2 &= \min_{R \in \mathcal{O}} \|\Up - \Uo_r R\|_F^2 \\ 
&\leq \norm{\Up -\Uorr }_F^2 \nonumber \\&= \norm{\Up -\U + \U -\Uorr }_F^2 \nonumber \\
&= \norm{\Up -\U}_F^2 + \norm{\U -\Uorr}_F^2 -2\ip{\Up -\U}{\Uorr -\U} \nonumber \\
&= \eta^2 \|\gradf(\X)U\|_F^2 + \dist(\U, \Uo_r)^2 - 2\eta\ip{\gradf(X)U}{\U -\Uorr }  \nonumber \\
&\leq \dist(\U, \Uo_r)^2,
\end{align} 
where last inequality is due to Lemma~\ref{lem:ipbound_case1}. 
\end{proof}

\section{Initialization proofs}{\label{sec:init_proofs}}

\subsection{Proof of Lemma~\ref{lem:switch}}
The proof borrows results from standard projected gradient descent. In particular, we know from Theorem 3.6 in \cite{bubeck2014theory} that, for consecutive estimates $\Xp,\X$ and optimal point $\Xo$, projected gradient descent satisfies:
\begin{align*}
\|\Xp - \Xo\|_F^2 \leq \left(1 - \tfrac{1}{\kappa}\right) \cdot \|\X - \Xo\|_F^2.
\end{align*} By taking square root of the above inequality, we further have:
\begin{align}\label{eq:init_help}
\|\Xp - \Xo\|_F \leq \sqrt{1 - \tfrac{1}{\kappa}} \cdot \|\X - \Xo\|_F \leq \left(1 - \tfrac{1}{2\kappa}\right) \cdot \|\X - \Xo\|_F,
\end{align} since $\sqrt{1 - \tfrac{1}{\kappa}} \leq 1 - \tfrac{1}{2\kappa}$, for all values of $\kappa > 1$. 

Given the above, the following (in)equalities hold true:
\begin{align*}
\norm{\X -\Xp}_F &= \norm{\X -\Xo +\Xo - \Xp}_F \\
&\stackrel{(i)}{\geq} \norm{\X -\Xo}_F -\norm{\Xp -\Xo}_F \\
&\stackrel{(ii)}{\geq} \frac{1}{2\kappa}\norm{\X -\Xo}_F \Rightarrow \\
\norm{\X -\Xo}_F &\leq 2\kappa \cdot \norm{\X -\Xp}_F,
\end{align*} where $(i)$ is due to the lower bound on triangle inequality and $(ii)$ is due to \eqref{eq:init_help}. Under the assumptions of the lemma, if $\norm{\X -\Xp}_F \leq \tfrac{c}{\kappa \sqrt{r} \tau(X_r)} \sigma_r(X)$, the above inequality translates into:
\begin{align*}
\norm{\X -\Xo}_F &\leq \tfrac{2c}{\sqrt{r}\tau(X_r)} \sigma_r(X).
\end{align*} By construction, both $\X$ and $\Xo$ are PSD matrices; moreover, $\X$ can be a matrix with $\text{rank}(\X) > r$. 
Hence, $$||X_r -\Xo||_F \leq \sqrt{r} ||X_r -\Xo||_2 \leq 2\sqrt{r}\|X-\Xo\|_2 \leq 2\sqrt{r}\|X-\Xo\|_F,$$ using Weyl's inequalities. Thus, $\norm{\X_r -\Xo}_F \leq \tfrac{4c}{\tau(X_r)} \sigma_r(X).$
Define $\X_r = \U_r\U_r^\top$ and $\Xo = \Uo_r \left(\Uo_r\right)^\top$. Further, by Lemma 5.4 of \cite{tu2015low}, we have:
\begin{align*}
\|\X_r - \Xo\|_F \geq \sqrt{2 \left(\sqrt{2} - 1\right)} \sigma_r(\Uo_r) \cdot \dist(\U_r, \Uo_r).
\end{align*} The above lead to:
\begin{align*}
 \sqrt{2 \left(\sqrt{2} - 1\right)} \sigma_r(\Uo_r) \cdot \dist(\U_r, \Uo_r) \leq \tfrac{4c}{\tau(X_r)} \sigma_r(X).
\end{align*} Recall that $\sigma_r(\X) = \sigma_r^2(\U_r)$; then, by Lemma \ref{lem:sigma_bounds}, there is constant $c'' > 0$ such that $\tfrac{4c}{\tau(\X_r)}\sigma_r^2(\U) \leq \tfrac{c''}{\tau(\Xo_r)}\sigma_r^2(\Uo_r)$. Combining all the above, we conclude that there is constant $c' > 0$ such that:
\begin{align*}
\dist(\U_r, \Uo_r) \leq \tfrac{c'}{ \tau(\Xo_r)} \sigma_r(\Uo_r).
\end{align*}

\subsection{Proof of Theorem \ref{thm:scinit}}\label{sec:init_prof}

Recall $\X^0 = \mathcal{P}_+ \left ( \frac{-\gradf(0)}{\| \gradf(0)-\gradf(e_1 e_1')\|_F} \right )$. Here, we remind that $ \mathcal{P}_+(\cdot)$ is the projection operator onto the PSD cone and $ \mathcal{P}_-(\cdot)$ is the projection operator onto the negative semi-definite cone. 


To bound $\|\X^0 -\Xo\|_F,$ we will bound each individual term in its squared expansion 
\begin{align*}
\|\X^0 -\Xo\|_F^2 = \|\X^0\|_F^2 + \| \Xo\|_F^2 -2\ip{\X^0}{\Xo}.
\end{align*}

From the smoothness of $\f$, we get the following:
\begin{align*}
M \norm{\Xo}_F &\geq   \norm{\gradf(0) - \gradf(\Xo)}_F \stackrel{(i)}{\geq} \norm{ \mathcal{P}_- (\gradf(0)) - \mathcal{P}_- ( \gradf(\Xo))}_F \stackrel{(ii)}{=} \norm{ \mathcal{P}_-(\gradf(0)) }_F.
\end{align*} where $(i)$ follows from non-expansiveness of projection operator and $(ii)$ follows from  the fact that  $\gradf(\Xo)$ is PSD and hence $\mathcal{P}_- ( \gradf(\Xo)) = 0$. Finally, observe that  $ \mathcal{P}_-(\gradf(0)) = \mathcal{P}_+(-\gradf(0))$. The above combined imply:
\begin{align*}
\norm{ \mathcal{P}_+(-\gradf(0)) }_F \leq M \norm{\Xo}_F \quad \Longrightarrow \quad \norm{\X^0}_F \leq \frac{M}{\| \gradf(0)-\gradf(e_1 e_1')\|_F} \cdot \norm{\Xo}_F \leq \kappa \norm{\Xo}_F
\end{align*}  where we used the fact that  $m \leq \norm{\gradf(0)-\gradf(e_1e_1^\top)}_F \leq M$ and $\kappa = \sfrac{M}{m}$. Hence $\|\X^0\|_F^2 \leq \kappa^2 \norm{\Xo}_F^2$.

Using the strong convexity of $\f$ around $\Xo$, we observe
\begin{align*} 
\f(0) \geq \f(\Xo) + \ip{\gradf(\Xo)}{0 - \Xo} + \frac{m}{2}\norm{\Xo}_F^2 \geq \f(\Xo) + \frac{m}{2}\norm{\Xo}_F^2,
\end{align*} where the last inequality follows from first order optimality of $\Xo$, $\ip{\gradf(\Xo)}{0 - \Xo} \geq 0$ and 0 is a feasible point for problem~\eqref{intro:eq_00}. Similarly, using strong convexity of $f$ around $0$, we have
\begin{align*}
\f(\Xo) \geq \f(0) + \ip{\gradf(0)}{\Xo} + \frac{m}{2}\norm{\Xo}_F^2 
\end{align*}
Combining the above two inequalities we get, $\ip{-\gradf(0)}{\Xo} \geq m\norm{\Xo}_F^2$. Moreover:
\begin{align*}
\ip{-\gradf(0)}{\Xo} &= \ip{\mathcal{P}_+\left(-\gradf(0)\right) + \mathcal{P}_{-}\left(-\gradf(0)\right)}{\Xo} \\ &= \ip{\mathcal{P}_+\left(-\gradf(0)\right)}{\Xo} + \underbrace{\ip{\mathcal{P}_{-}\left(-\gradf(0)\right)}{\Xo}}_{\leq 0}
\end{align*} since $\Xo$ is PSD. Thus, $ \ip{\mathcal{P}_+(-\gradf(0))}{\Xo} \geq \ip{-\gradf(0)}{\Xo}$ and
\begin{align}
\ip{X^0}{\Xo} \geq \frac{m}{\| \gradf(0)-\gradf(e_1 e_1')\|_F}\norm{\Xo}_F^2 \geq \frac{1}{\kappa} \norm{\Xo}_F^2, \label{proofs:eq_000}
\end{align} where we used the fact that  $m \leq \norm{\gradf(0)-\gradf(e_1e_1^\top)}_F \leq M$. Given the above inequalities, we can now prove the following:
\begin{small}
\begin{align*}
\norm{\X^0 - \Xo}_F^2 &= \|\X^0\|_F^2 + \| \Xo\|_F^2 -2\ip{\X^0}{\Xo} \\ &\leq \kappa^2 \norm{\Xo}_F^2 + \norm{\Xo}_F^2 - \frac{2}{\kappa} \norm{\Xo}_F^2 \\ &= \left(\kappa^2 -\frac{2}{\kappa}  +1\right) \norm{\Xo}_F^2.
\end{align*} 
\end{small}

Now we know that $\norm{  \X^0 - \Xo }_F ~ \leq  ~ \sqrt{\kappa^2 -\sfrac{2}{\kappa} +1} \, \norm{\Xo}_F .$  Now, by triangle inequality  $\norm{  \X^0 - \Xo_r }_F ~ \leq   \sqrt{\kappa^2 -\sfrac{2}{\kappa} +1} \norm{\Xo}_F + \norm{\Xo -\Xo_r}_F $. By $||.||_2 \leq ||.||_F$ and Weyl's inequality for perturbation of singular values (Theorem 3.3.16 \cite{horn37topics}) we get,
$$ \norm{\X^0_r -\Xo_r}_2 \leq 2\sqrt{\kappa^2 -\sfrac{2}{\kappa} +1} \norm{\Xo}_F + 2 \norm{\Xo -\Xo_r}_F .$$ By the assumptions of the theorem, we have $\norm{\Xo -\Xo_r}_F \leq \tilde{\rho} \norm{\Xo}_2$. Therefore,
$$ \norm{\X^0_r -\Xo_r}_F \leq 2\sqrt{2 r} \left( \sqrt{\kappa^2 -\sfrac{2}{\kappa} +1} \norm{\Xo}_F +   \tilde{\rho} \norm{\Xo}_2\right).$$ Now again using triangle inequality and substituting we get $\norm{\Xo}_F \leq \text{\texttt{srank}}^{\sfrac{1}{2}}\norm{\Xo}_2 + \tilde{\rho}  \norm{\Xo}_2$. Finally combining this with Lemma~\ref{lem:lower_bound_X_Xr} gives the result.



\section{Dependence on condition number in linear convergence rate}\label{sec:discuss}
It is known that the convergence rate of classic gradient descent schemes depends only on the condition number $\kappa = \tfrac{M}{m}$ of the function $f$. However, in the case of \algo, we notice that convergence rate also depends on condition number $\tau(\Xo_r) = \tfrac{\sigma_1(\Xo)}{\sigma_r(\Xo)}$, as well as $\|\gradf(\Xo)\|_2$.

To elaborate more on this dependence, let us recall the update rule of \algo, as presented in Section \ref{sec:algo}. In particular, one can observe that the gradient direction has an extra factor $U$, multiplying $\gradf(UU^\top)$, as compared to the standard gradient descent on $\X$. One way to reveal how this extra factor affects the condition number of the Hessian of $f$, 
we consider the special case of separable functions; see the definition of separable functions in the next lemma. Next, we show that the condition number of the Hessian -- for this special case -- has indeed a dependence on both $\tau(\Xo_r)$ and $\|\gradf(\Xo)\|_2$, a scaling similar to the one appearing in the convergence rate $\alpha$ of \algo.

\begin{lemma}[Dependence of Hessian on $\tau(\Xo_r)$ and $\|\gradf(\Xo)\|_2$]\label{lem:hessian}
Let $f$ be a smooth, twice differentiable function over the PSD cone. Further, assume $f$ is a separable function over the matrix entries, such that  $f(\X) = \sum_{(i, j)} \varphi_{ij}(\X_{ij})$, where $(i, j) \in [n] \times [n]$, and let $\varphi_{ij}$'s be $M$-smooth and $m$-strongly convex functions, $\forall i, j$. Finally, let $\Xo = \Uo (\Uo)^\top$ be rank-$r$ and let $\nabla_{\U R}^2 f(\X)$ denote the Hessian of $f$ w.r.t. the $\U$ factor and up to rotations, for some rotation matrix $R \in \mathcal{O}$. Then, 
\begin{align*}
 \sigma_{1}\left(\nabla_{\Uo}^2 f(\Xo)\right) \leq C \cdot (M\|\Xo\|_2 +\|\gradf(\Xo)\|_2),
\end{align*}
 for constant $C$. Further, for any unit vector $y \in \R^{nr \times 1}$ such that columns of ${\rm \texttt{mat}}(y) \in \R^{n \times r}$ are orthogonal to $\Uo$, i.e., ${\rm \texttt{mat}}(y)^\top \Uo = 0$, we further have: 
\begin{align*}
y^\top \nabla_{\Uo R}^2 f(\Xo) y  \geq c \cdot m\sigma_{r}(\Xo),
\end{align*} for some constant $c$.
\end{lemma}


\begin{proof}
By the definition of gradient, we know that $\nabla_U f(\U\U^\top) =  (\gradf(\U \U^\top) + \gradf(\U \U^\top)^\top )\U$; for simplicity, we assume $\gradf(\U\U^T)$ be symmetric. Since $X$ is symmetric, with $\gradf(\U\U^T)_{ij} = \varphi'_{ij}(X_{ij})$ and  $\varphi'_{ij}(X_{ij}) = \varphi'_{ji}(X_{ji})$. By the definition of Hessian, the entries of $\nabla_U^2 f(UU^T)$ are given by:
\begin{align*}
\left(\nabla_U^2 f(UU^T)\right)_{ij, kl} = \frac{\partial}{\partial U_{kl}}\sum_{p=1}^n \varphi'_{ip}(X_{ip}) U_{pj} = \underbrace{\sum_{p=1}^n \frac{\partial \varphi'_{ip}(X_{ip})}{\partial U_{kl}} U_{pj}}_{:=T_1} + \underbrace{\sum_{p=1}^n \varphi'_{ip}(X_{ip}) \frac{\partial  U_{pj} }{\partial U_{kl}}}_{:=T_2}.
\end{align*} In particular, for $T_1$ we observe the following cases:
\begin{align*}
T_1 =
\left\{
	\begin{array}{ll}
		\varphi_{ik}''(X_{ik}) U_{il}U_{kj}  & \mbox{if } i \neq k, \\
		\sum_p \varphi_{ip}''(X_{ip}) U_{pl}U_{pj} + \varphi_{ii}''(X_{ii}) U_{il}U_{kj} & \mbox{if } i = k.
	\end{array}
\right.
\end{align*} while, for $T_2$ we further have:
\begin{align*}
T_2 =
\left\{
	\begin{array}{ll}
		0  & \mbox{if } j \neq l, \\
		\varphi_{ik}'(X_{ik}) & \mbox{if } j = l.
	\end{array}
\right.
\end{align*}
Consider now the case where gradient and Hessian information is calculated at the optimal point $\Xo$. Based on the above, the Hessian of $f$ w.r.t $\Uo$ turns out to be a sum of three PSD $nr \times nr $ matrices, as follows: 
\begin{align*}
\nabla_{\Uo}^2 f(\Xo) =A + B +C,
\end{align*} where
\begin{itemize}
\item [$(i)$] $A =(\widehat{\Uo})^T G \widehat{\Uo}$, where $G$ is a $n^2 \times n^2$ diagonal matrix with diagonal elements $\varphi''_{ij}(\Xo_{ij})$ and $\widehat{\Uo}$ is a $n^2 \times nr$ matrix with $\Uo$ repeated $n$ times on the diagonal. It is easy to see that 
\begin{align*}
\|A\|_2 \leq \|\varphi_{ij}''\|_{\infty} \sigma_{\max} (\Uo)^2 =M \|\Xo\|_2.
\end{align*} 
Similarly, we have $\sigma_{nr}(A) \geq  \min{\varphi_{ij}''} \cdot \sigma_{\min} (\Uo)^2 =m \sigma_{\min}(\Xo).$
\item [$(ii)$] $B$ is a $nr \times nr$ matrix, with $B_{ij, kl} =  \varphi_{ik}''(\Xo_{ik}) \Uo_{il} \Uo_{kj}$. Again, it is easy to verify that  $\|B\|_2 \leq M \|\Xo\|_2$. Now for $y$ perpendicular to $\Uo$, notice that $y^\top B y =0$, since the columns of $B$ are concatenation of scaled columns of $\Uo$.

\item [$(iii)$] $C$ is a $nr \times nr$ diagonal block-matrix, with $n \times n$ blocks $\gradf(\Xo)$ repeated $r$ times. It is again easy to see that $\|C\|_2 \leq \|\gradf(\Xo)\|_2$, since $C$ is a block diagonal matrix. Moreover, by KKT optimality condition $\gradf(\Xo) \Xo =0$, $\text{rank}(\gradf(\Xo)) \leq n-r$ and thus, $\sigma_{nr} (C) =0$.
\end{itemize} Combining the above results and observing that all the three matrices are PSD, we conclude that $\sigma_1\left(\nabla_{\Uo}^2 f(\Xo)\right) \leq C\cdot \left(M\|\Xo\|_2 + \|\gradf{\Xo}\|_2\right)$. Regarding the lower bound on $\sigma_{nr}\left(\nabla_{\Uo}^2 f(\Xo)\right)$, we observe the following: due to $\U\U^\top$ factorization and for $\Uo$ optimum, we know that also $\Uo R$ is optimum, where gradient $\gradf(\Uo (\Uo)^\top) = \gradf(\Uo R R^\top (\Uo)^\top) = 0$. This further indicates that the hessian of $f$ is zero along directions corresponding to columns of $\Uo$, and thus $\sigma_{nr}\left(\nabla_{\Uo}^2 f(\Xo)\right) = 0$ along these directions; see figure \ref{fig:tau_1} (right panel) for an example. However, for any other directions orthogonal to $\Uo$, we have $y^\top\left(\nabla_{\Uo R}^2 f(\Xo)\right)y  \geq c \cdot m \sigma_{\min}(\Xo)$, for some constant $c$. This completes the proof.

\end{proof}

To show this dependence in practice, we present some simulation results in Figure~\ref{fig:tau_1}. We observe that the convergence rate does indeed depend on $\tau(\Xo_r)$.

\begin{figure}[!ht]
	\begin{center}
		\includegraphics[width=0.5\textwidth]{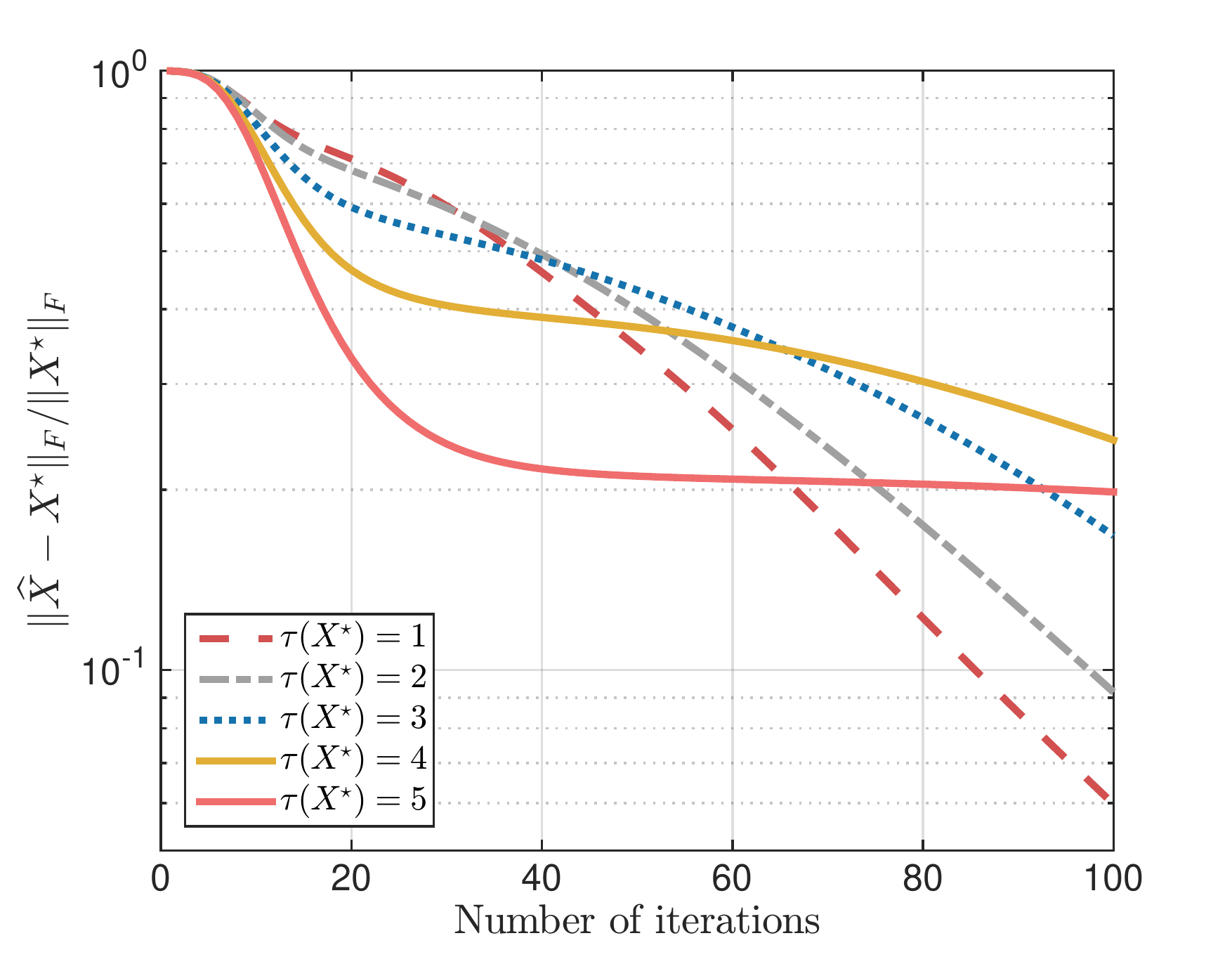}
		\includegraphics[width=0.45\textwidth]{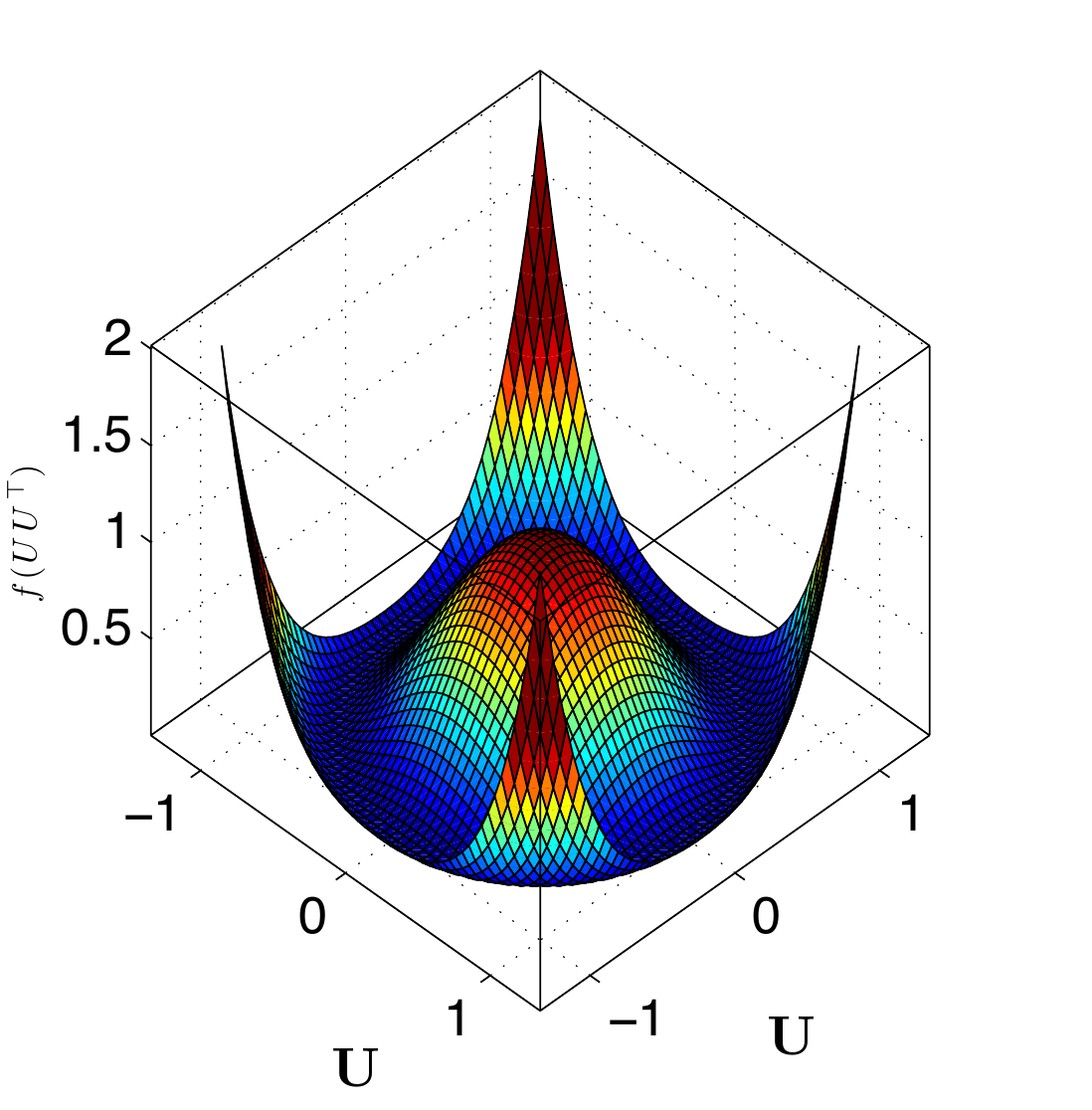}
	\end{center}
	\caption{Left panel: Assume dimension $n = 50$. We consider the matrix sensing setup \cite{recht2010guaranteed} and generate $m = \lceil 2 n \log n \rceil$ Gaussian linear measurements of $n \times n$ matrices $\Xo$ of rank $r = 2$, with varying condition number $\tau(\Xo)$. We compute matrix $X =UU^\top$, $U$ is $n \times r$ tall  matrix, by minimizing the standard least squares lost function, using our scheme. In the plot, we show the log error versus total number of iterations. Observe that, varying the condition number of $\Xo$, higher $\tau(\Xo)$ leads to slower convergence. Right panel: Contour of function $(u_1^2+u_2^2-1)^2$. Observe the ``ring'' of points $(u_1, u_2)$ where $f$ is minimized. This illustrates the existence of multiple points with zero gradient and, thus, directions where the hessian of the objective is zero. } 
{\label{fig:tau_1}}
\end{figure}

\section{Test case I: Matrix sensing problem}\label{sec:sensing}
In this section, we briefly describe and compare algorithms designed specifically for the \emph{matrix sensing} problem, using the variable parametrization $\X = \U\U^\top$. To accommodate the PSD constraint, we consider a variation of the matrix sensing problem where one desires to find $\X^\star$ that minimizes\footnote{This problem is a special case of affine rank minimization problem \cite{recht2010guaranteed}, where no PSD constraints are present.}:
\begin{equation}
\begin{aligned}
\underset{\X \in \R^{n \times n}}{\text{minimize}}
& & \tfrac{1}{2} \|b - \mathcal{A}\left(\X\right)\|_F^2 \quad \text{subject to}
& & \text{rank}(\X) \leq r,~\X \succeq 0.
\end{aligned} \label{related:eq_MS}
\end{equation} W.l.o.g., we assume $b = \mathcal{A}\left(\Xo\right)$ for some rank-$r$ $\Xo$. Here, $\mathcal{A}:\R^{n \times n} \rightarrow \R^p$ is the linear sensing mechanism, such that the $i$-th entry of $\mathcal{A}(\X)$ is given by $\left\langle A_i, ~\X \right\rangle$, for $A_i \in \R^{n \times n}$ sub-Gaussian independent measurement matrices. 

\cite{jain2013low} is one of the first works to propose a provable and efficient algorithm for \eqref{related:eq_MS}, operating in the $\U$-factor space, while \cite{sa2015global} solves \eqref{related:eq_MS} in the stochastic setting; see also \cite{zheng2015convergent, tu2015low, chen2015fast}. To guarantee convergence, most of these algorithms rely on \emph{restricted isometry assumptions}; see Definition \ref{def:RIP} below. 

To compare the above algorithms with \algo, Subsection \ref{sec:RSC1} further describes the notion of restricted strong convexity and its connection with the RIP. Then, Subsection \ref{sec:RSC2} provides explicit comparison results of the aforementioned algorithms, with respect to the convergence rate factor $\alpha$, as well as initialization conditions assumed, for each case.

\subsection{Restricted isometry property and restricted strong convexity}\label{sec:RSC1}
To shed some light on the notion of restricted strong convexity and how it relates to the RIP, consider the matrix sensing problem, as described above. According to \eqref{related:eq_MS}, we consider the quadratic loss function:
\begin{align*}
f(\X) = \tfrac{1}{2} \|b - \mathcal{A}(\X)\|_F^2.
\end{align*} Since the Hessian of $f$ is given by $\mathcal{A}^*\mathcal{A}$, restricted strong convexity suggests that \cite{negahban2012restricted}: 
\begin{align*}
\|\mathcal{A}(Z)\|_2^2\geq C \cdot \|Z\|_F^2, \quad Z \in \R^{n \times n},
\end{align*} for a restricted set of directions $Z$, where $C > 0 $ is a small constant. This bound implies that the quadratic loss function, as defined above, is strongly convex in such a restricted set of directions $Z$.\footnote{One can similarly define the notion of restricted smoothness condition, where $\|\mathcal{A}(Z)\|_2^2$ is upper bounded by $\|Z\|_F^2$. }

A similar but stricter notion is that of \emph{restricted isometry property} for low rank matrices \cite{candes2011tight, liu2011universal}:
\begin{definition}[Restricted Isometry Property (RIP)]\label{def:RIP}
A linear map $\mathcal{A}$ satisfies the $r$-RIP with constant $\delta_r$, if
\begin{align*}
(1 - \delta_r)\|\X\|_F^2 \leq \|\mathcal{A}(\X)\|_2^2 \leq (1 + \delta_r)\|\X\|_F^2,
\end{align*} is satisfied for all matrices $\X \in \R^{n \times n}$ such that $\text{rank}(\X) \leq r$. 
\end{definition} 

The correspondence of restricted strong convexity with the RIP is obvious: both lower bound the quantity $\|\mathcal{A}(\X)\|_2^2$, where $\X$ is drawn from a restricted set. It turns out that linear maps that satisfy the RIP for low rank matrices, also satisfy the restricted strong convexity; see Theorem 2 in \cite{chen2010general}.

By assuming RIP in \eqref{related:eq_MS}, the condition number of $f$ depends on the RIP constants of the linear map $\mathcal{A}$; in particular, one can show that $\kappa = \tfrac{M}{m} \propto \tfrac{1 + \delta}{1-\delta}$, since the eigenvalues of $\mathcal{A}^*\mathcal{A}$ lie between $1-\delta$ and $1 + \delta$, when restricted to low-rank matrices. For $\delta$ sufficiently small and dimension $n$ sufficiently large, $\kappa \approx 1$, which, with high probability, is the case for $\mathcal{A}$ drawn from a sub-Gaussian distribution. 

\subsection{Comparison}\label{sec:RSC2}
Given the above discussion, the following hold true for \algo, under RIP settings:
\begin{itemize}
\item [$(i)$] In the noiseless case, $b = \mathcal{A}(\Xo)$ and thus, $\|\gradf(\Xo)\|_2 = \| -2 \mathcal{A}^*\left(b - \mathcal{A}(\Xo)\right)\|_2 = 0$. Combined with the above discussion, this leads to convergence rate factor
\begin{align*}
\alpha \lesssim 1 - \tfrac{c_4}{\tau(\Uo_r)^2},
\end{align*} in \algo.
\item [$(ii)$] In the noisy case, $b = \mathcal{A}(\Xo) + e$ where $e$ is an additive noise term; for this case, we further assume that $\| \mathcal{A} \left( e \right)\|_2$ is bounded. Then, 
\begin{align*}
\alpha \lesssim 1 - \tfrac{c_4}{\tau(\Uo_r)^2 + \frac{\|\mathcal{A}(e)\|_2}{(1- \delta) \sigma_r(\Xo)}}.
\end{align*}
\end{itemize}

Table \ref{table:convergence_comp_summary} summarizes convergence rate factors $\alpha$ and initialization conditions of state-of-the-art approaches for the noiseless case.
\begin{table*}[!htb]
\centering
\begin{small}
\begin{tabular}{c c c c c}
  \toprule
  Reference & & $\dist\left(\Up, \Uo\right)^2 \leq \alpha \cdot \dist\left(\U, \Uo \right)^2$ & & $\dist\left(\U^0, \Uo\right) \leq \cdots$ \\ 
	  \cmidrule{1-1} \cmidrule{3-3} \cmidrule{5-5}
	  \cite{jain2013low} & & $\alpha = \tfrac{1}{16}^\dagger$ & & $\sqrt{6 \delta} \cdot \tfrac{\sigma_r(\Xo)}{\sigma_1(\Xo)}\sigma_r(\Uo_r)$  \\
	  \cite{tu2015low} & & $\alpha = 1 - \tfrac{c_1}{\tau(\Uo_r)^4}$ & & $\tfrac{1}{4}\sigma_r(\Uo_r)$ \\
      \cite{zheng2015convergent} & & $\alpha = 1 - \tfrac{c_2}{\tau(\Uo_r) \cdot r}$  & & $\sqrt{\tfrac{3}{16}} \cdot \sigma_r(\Uo_r)$ \\
      \cite{chen2015fast} & & $\alpha = 1 - \tfrac{c_3}{\tau(\Uo_r)^{10}}$  & & $\left(1 - \tau\right) \cdot \sigma_r(\Uo_r)$ \\      
  \midrule
      This work & & $\alpha = 1 - \tfrac{c_4}{\tau(\Uo_r)^2}$  & & $\tfrac{1}{100} \cdot \tfrac{\sigma_r(\Xo)}{\sigma_1(\Xo)} \sigma_r(\Uo_r)$ \\  
  \bottomrule
\end{tabular}
\end{small}
\caption{Comparison of related work for the matrix sensing problem. All methods use $\U\U^\top$ parametrization of the variable $\X$ and admit linear convergence. $\tau = \sqrt{12\delta}$ according to \cite{chen2015fast}. $c_i > 0, ~\forall i$ denote absolute constants. In \cite{jain2013low}, the proposed algorithm is designed to solve the rectangular case where $\X = \U\V^\top$; the reported factor $\alpha$ and initial conditions could be improved for the case of \eqref{intro:eq_01}. $^\dagger$ Note that this convergence is in terms of subspace distance. } \label{table:convergence_comp_summary} 
\end{table*}

\subsection{Empirical results}\label{sec:sims}
We start our discussion on empirical findings with respect to the convergence rate of the algorithm, how the step size and initialization affects its efficiency and some comparison plots with an efficient first-order projected gradient solver. We note that the experiments presented below are performed as a proof of concept and are not complete in the set of algorithms we could compare with.

\paragraph{Linear convergence rate and step size selection: } To show the convergence rate of the factored gradient descent in practice, we solve affine rank minimization problems instances with synthetic data. In particular, the ground truth $\Xo \in \mathbb{R}^{n \times n}$ is synthesized as a rank-$r$ matrix as $\X^\star = \U^\star \left(\U^\star\right)^\top$, where $\U^\star \in \mathbb{R}^{n \times r}$. In sequence, we sub-sample $\X^\star$ by observing $m = C_{\text{sam}} \cdot p \cdot r$ entries, according to:
\begin{align}
y = \mathcal{A}(\Xo) \in \R^m. \label{exp:eq_00}
\end{align} We use permuted and sub-sampled noiselets for the linear operator $\mathcal{A}: \R^{n \times n} \rightarrow \R^m $; for more information, see \cite{waters2011sparcs}. $\mathbf{y} \in \mathbb{R}^m$ contains the linear measurements of $\X^\star$ through $\mathcal{A}$ in vectorized form. We consider the noiseless case, for ease of exposition. Under this setting, we solve \eqref{intro:eq_01} with $f(\uut) := \sfrac{1}{2} \cdot \|y - \mathcal{A}\left(\uut\right)\|_2^2$. We use as a stopping criterion the condition $\|\Up\left(\Up\right)^\top - \U\U^\top\|_F < \text{\texttt{tol}} \cdot \|\Up\left(\Up\right)^\top\|_F$ where $\text{\texttt{tol}} := 5 \cdot 10^{-6}$.

\begin{figure}[!ht]
	\begin{center}
		\includegraphics[width=0.34\textwidth]{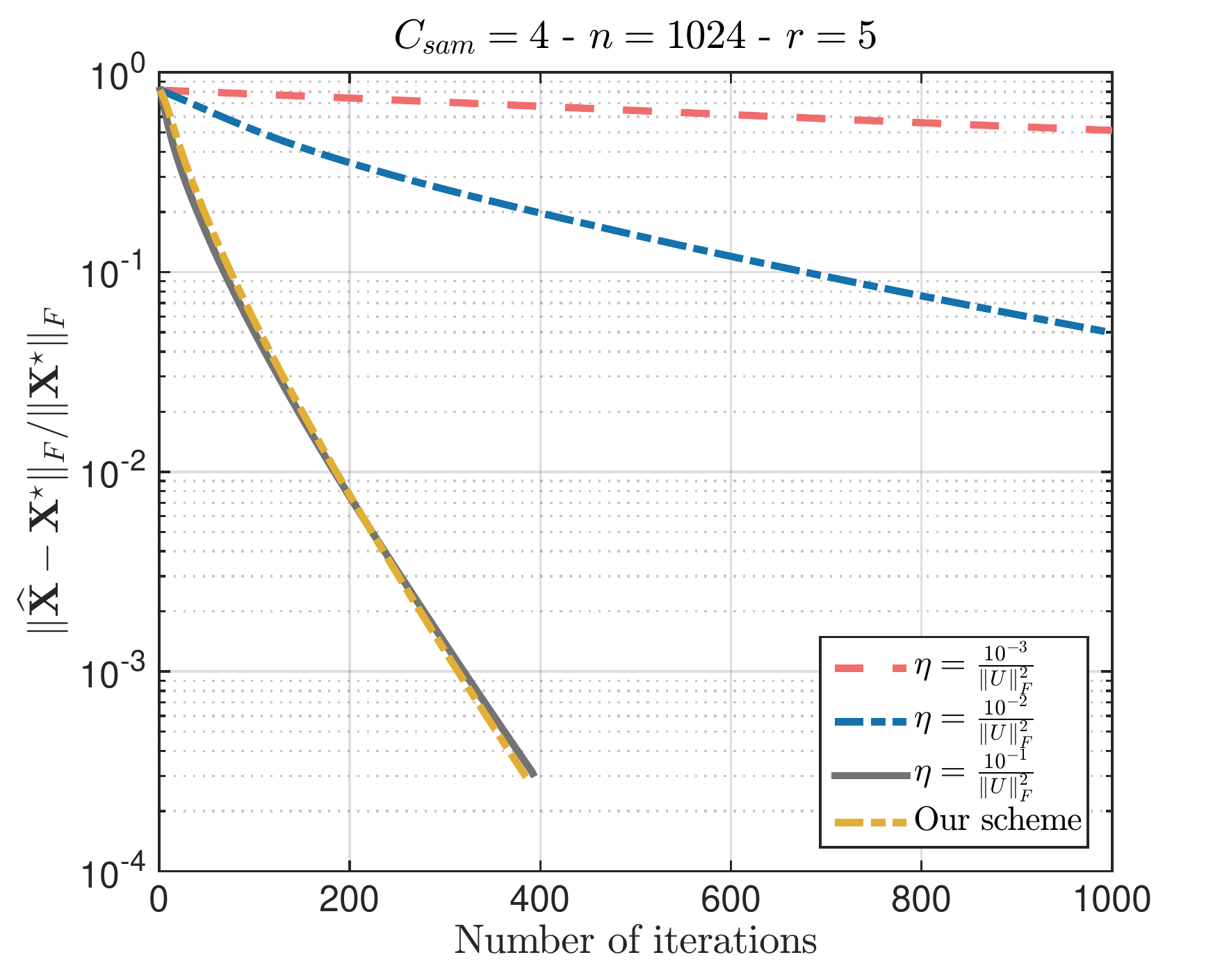} \hspace{-0.5cm}
		\includegraphics[width=0.34\textwidth]{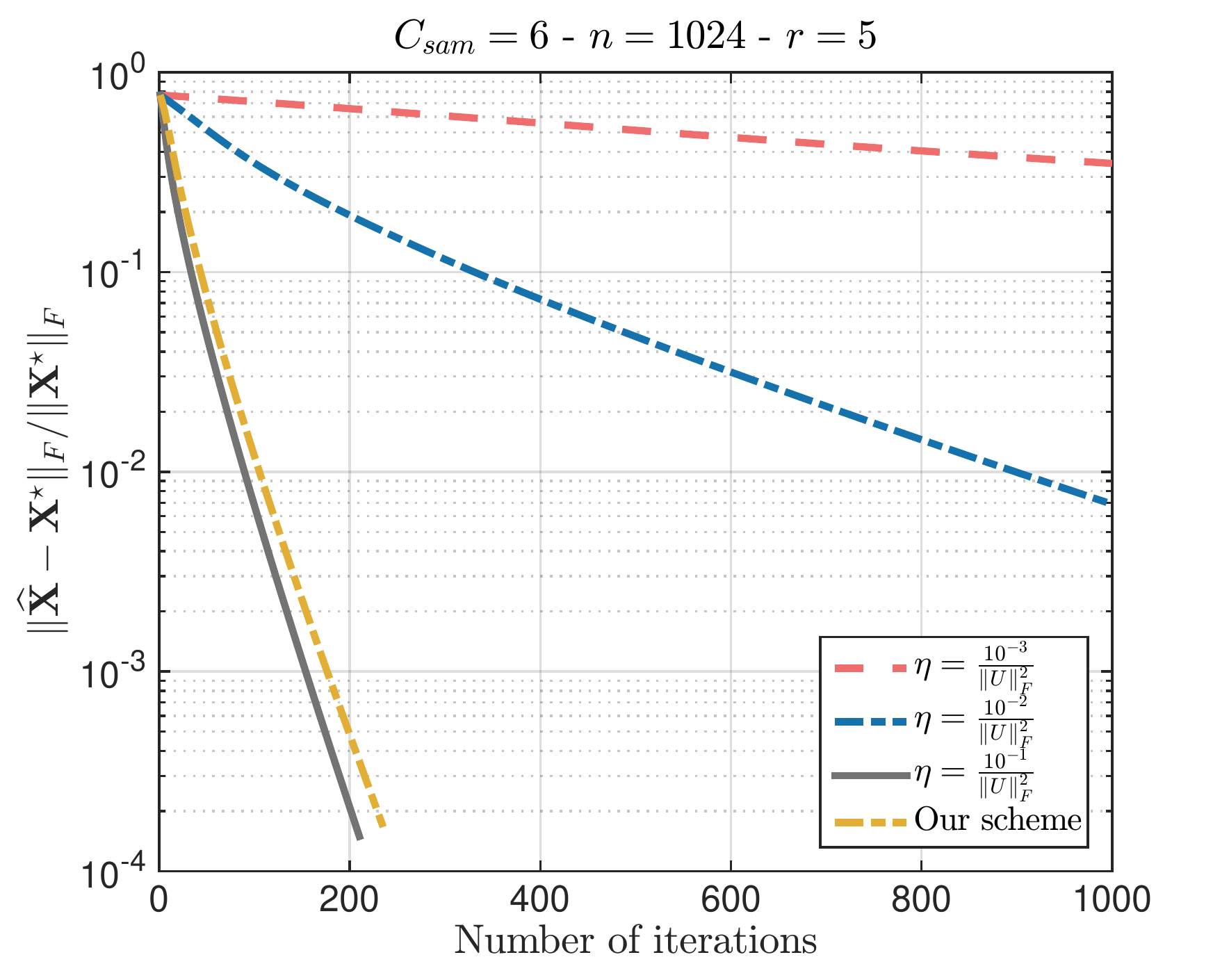} \hspace{-0.5cm}
		\includegraphics[width=0.34\textwidth]{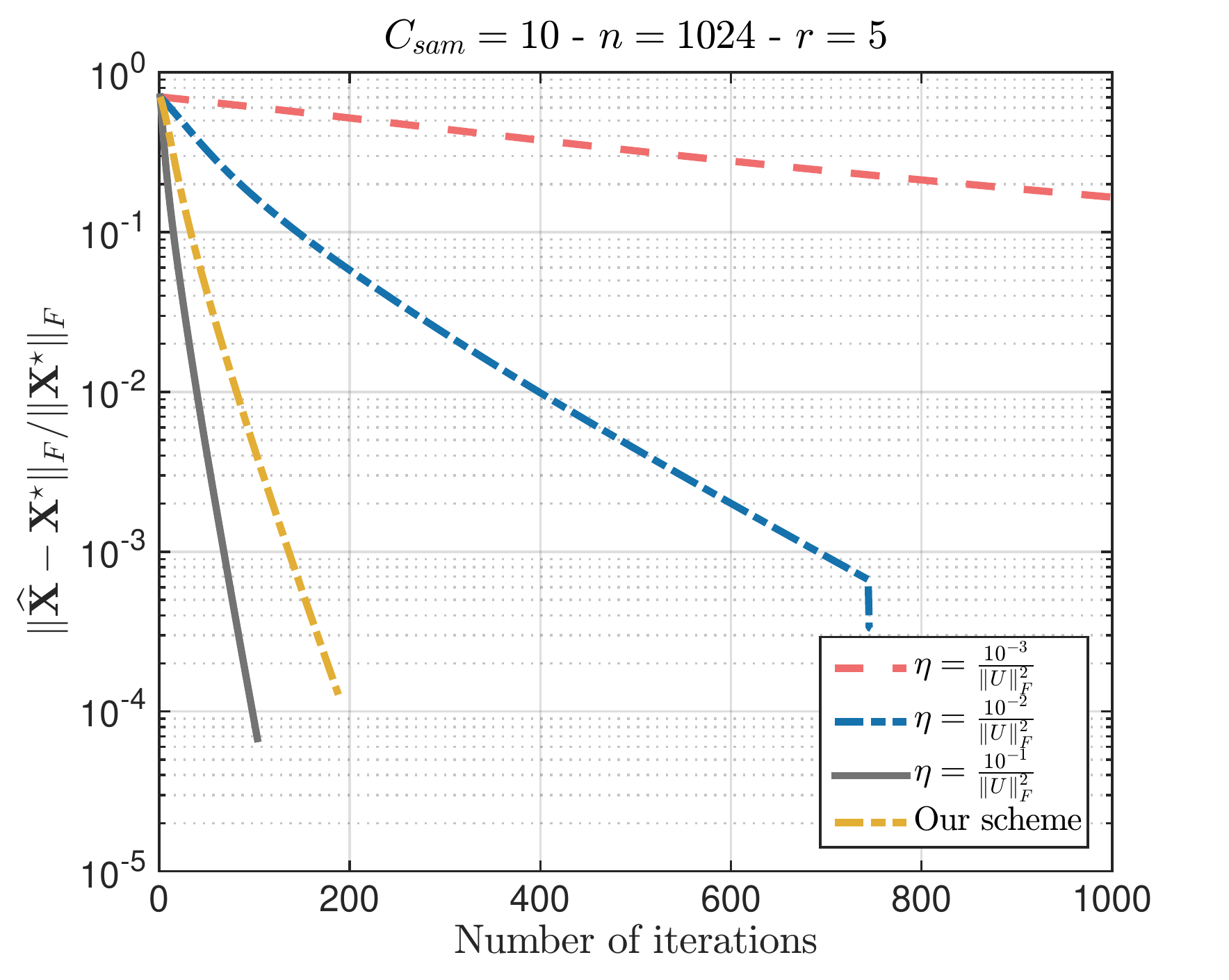}		
	\end{center}
	\caption{Median error per iteration of factored gradient descent algorithm for different step sizes, over 20 Monte Carlo iterations. The number of measurements is fixed to $C_{\text{sam}} \cdot n \cdot r$ for varying $C_{\text{sam}} \in \left\{4, 6, 10 \right\}$. Here, $n = 1204$ and rank $r = 5$. Curves show convergence behavior of factored gradient descent  as a function of the step size selection. One can observe that arbitrary step size selections can lead to slow convergence. Moreover, good constant step size selections -- for a specific problem configuration, do not necessarily translate into good performance for a different setting; \text{e.g.}, observe how the constant step size convergence rates worsen \textit{faster}, as we decrease the number of observations.}{\label{fig:00}}
\end{figure}

Figure \ref{fig:00} show the linear convergence of our approach as well as the efficiency of our step selection, as compared to other arbitrary constant step size selections. All instances use our initialization point. It is worth mentioning that the performance of our step size can be inferior to specific constant step size selections; however, finding such a good constant step size usually requires trial-and-error rounds and do not come with convergence guarantees. 
Moreover, we note that one can perform line search procedures to find the ``best'' step size per iteration; although, for more complicated $f$ instances, such step size selection might not be computationally desirable, even infeasible.


\paragraph{Impact of avoiding low-rank projections on the PSD cone:} In this experiment, we compare factored gradient descent with a variant of the Singular Value Projection (\texttt{SVP}) algorithm \cite{jain2010guaranteed, becker2013randomized}\footnote{\texttt{SVP} is a non-convex, first-order, projected gradient descent scheme for low rank recovery from linear measurements.}. For the purpose of this experiment, the \texttt{SVP} variant further projects on the PSD cone, along with the low rank projection. Its main difference is that it does not operate on the factor $\U$ space but requires projection over the (low-rank) positive semi-definite cone per iteration. In the discussion below, we refer to this variant as \texttt{SVP} (SDP).

\begin{figure}[!ht]
	\begin{center}
		\includegraphics[width=0.45\textwidth]{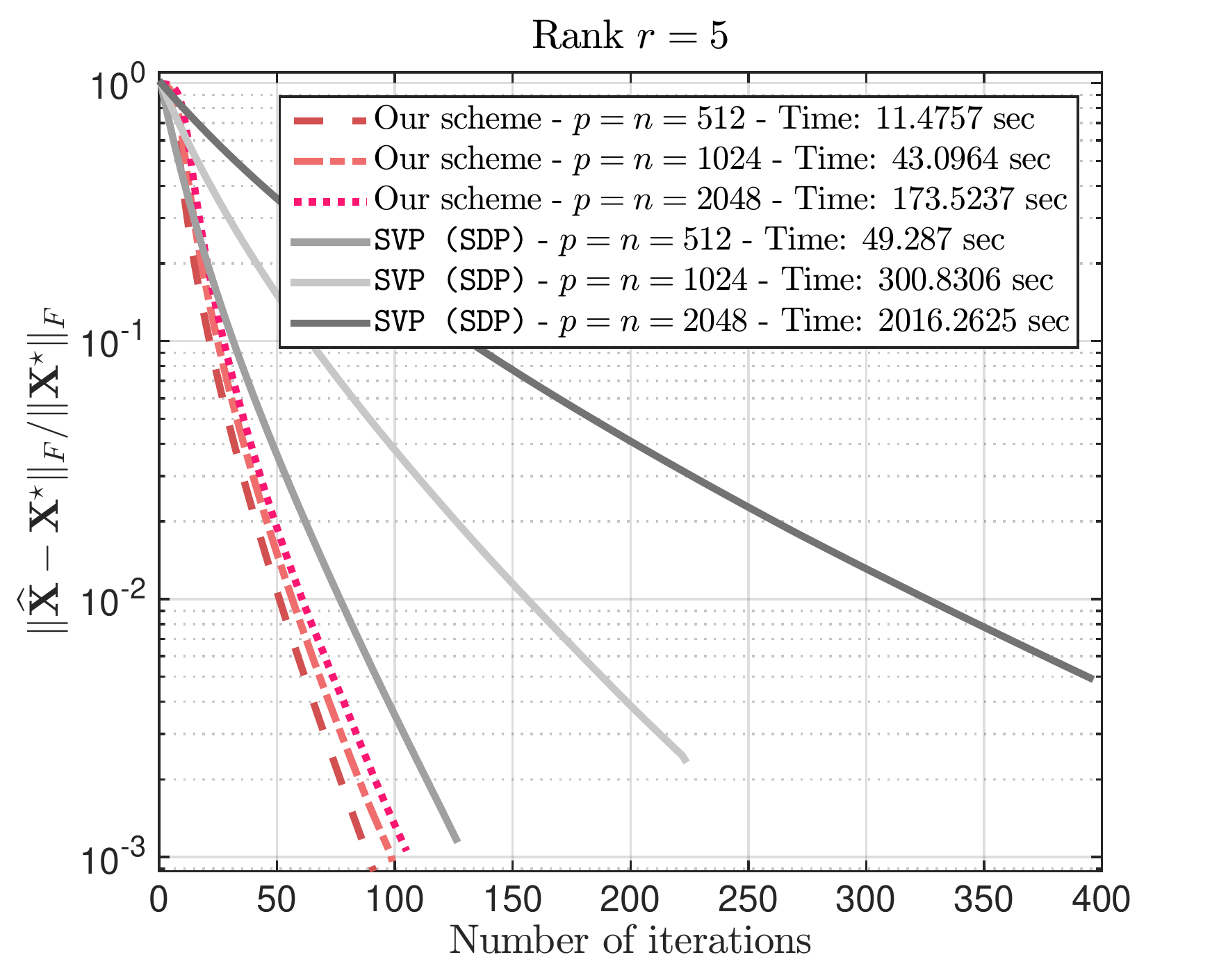}
		\includegraphics[width=0.45\textwidth]{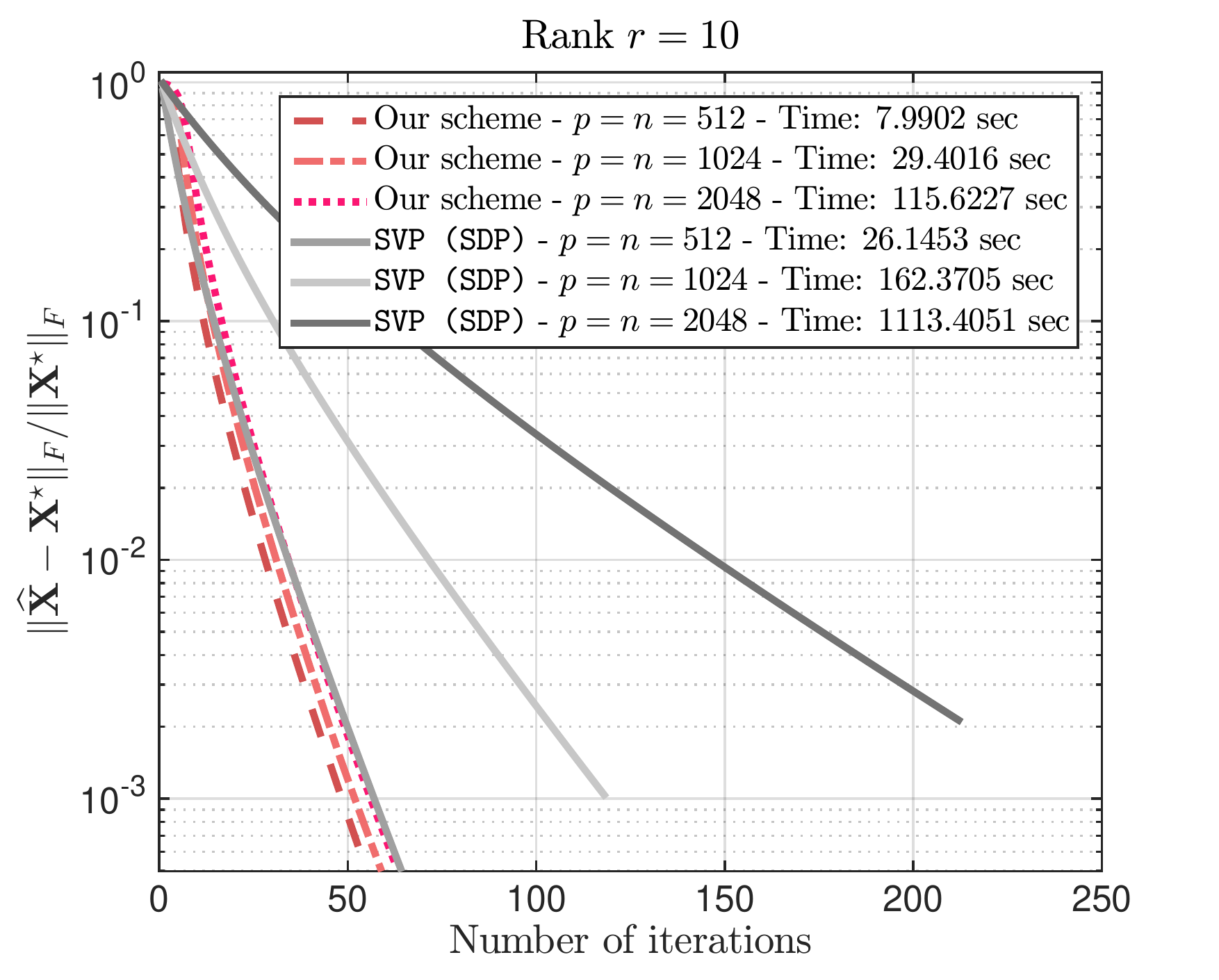}
	\end{center}
	\caption{Median error per iteration for factored gradient descent and \texttt{SVP} (SDP) algorithms, over 20 Monte Carlo iterations. For all cases, the number of measurements is fixed to $C_{\text{sam}} \cdot n \cdot r$ for $C_{\text{sam}} = 6.$ From left to right, we consider different rank configurations: $(i)$ $r = 5$ and $(ii)$ $r = 10$. Both schemes use the same initialization point. Both plots show better convergence rate performance in terms of iterations due to our step size selection. In addition, factored gradient descent avoids performing SVD operations per iteration, a fact that leads also to lower per iteration complexity; see also Table \ref{tbl:CompSVP}.}{\label{fig:01}}
\end{figure}

\begin{table}[!ht]
\centering
\ra{1.3}
\small
\begin{tabular}{ll c cc c cc} \toprule
\multicolumn{2}{c}{Model} & \phantom{a} & \multicolumn{2}{c}{$\|\widehat{\X} - \X^\star\|_F/\| \X^\star\|_F$} &  \phantom{a}  & \multicolumn{2}{c}{Time (sec)} \\
\cmidrule{1-2} \cmidrule{4-5} \cmidrule{7-8} 
\multicolumn{1}{c}{$n$} & \multicolumn{1}{c}{$r$} & \phantom{a} & \texttt{SVP} (SDP) & Our scheme & \phantom{a}  & \texttt{SVP} (SDP) & Our scheme  \\ \midrule
\multirow{3}{*}{$512$} & $5$ &  & 1.1339e-03  & \textcolor[rgb]{0.4,0.1,0}{\textbf{8.4793e-04}} &  & 36.9652  & \textcolor[rgb]{0.4,0.1,0}{\textbf{11.4757}} \\
 & $10$ &  & 4.6552e-04  & \textcolor[rgb]{0.4,0.1,0}{\textbf{4.4954e-04}} &  & 19.6089  & \textcolor[rgb]{0.4,0.1,0}{\textbf{7.9902}} \\
 & $20$ &  & \textcolor[rgb]{0.4,0.1,0}{\textbf{1.6541e-04}}  & 2.0571e-04 &  & 10.6052  & \textcolor[rgb]{0.4,0.1,0}{\textbf{6.4149}} \\
 \midrule 
\multirow{3}{*}{$1024$} & $5$ &  & 2.4224e-03  & \textcolor[rgb]{0.4,0.1,0}{\textbf{9.9180e-04}} &  & 225.6230  & \textcolor[rgb]{0.4,0.1,0}{\textbf{43.0964}} \\
 & $10$ &  & 1.0203e-03  & \textcolor[rgb]{0.4,0.1,0}{\textbf{4.5103e-04}} &  & 121.7779  & \textcolor[rgb]{0.4,0.1,0}{\textbf{29.4016}} \\
 & $20$ &  & 4.1149e-04  & \textcolor[rgb]{0.4,0.1,0}{\textbf{2.3442e-04}} &  & 67.6272  & \textcolor[rgb]{0.4,0.1,0}{\textbf{22.9616}} \\
 \midrule 
\multirow{3}{*}{$2048$} & $5$ &  & 4.8500e-03  & \textcolor[rgb]{0.4,0.1,0}{\textbf{1.0093e-03}} &  & 1512.1969  & \textcolor[rgb]{0.4,0.1,0}{\textbf{173.5237}} \\
 & $10$ &  & 2.0836e-03  & \textcolor[rgb]{0.4,0.1,0}{\textbf{4.6735e-04}} &  & 835.0538  & \textcolor[rgb]{0.4,0.1,0}{\textbf{115.6227}} \\
 & $20$ &  & 9.4893e-04  & \textcolor[rgb]{0.4,0.1,0}{\textbf{2.6417e-04}} &  & 458.8766  & \textcolor[rgb]{0.4,0.1,0}{\textbf{88.1960}} \\
\bottomrule
\end{tabular}
\caption{Summary of comparison results for reconstruction and efficiency. Observe that both our scheme and \texttt{SVP} (SDP) require more iterations to converge as $r$ radically decreases. This justifies the higher time-complexity observed; see also Figure \ref{fig:01} for comparison.} \label{tbl:CompSVP}
\end{table}
We perform two experiments. In the first experiment, we compare factored gradient descent with \texttt{SVP} (SDP), as designed in \cite{jain2010guaranteed}; \textit{i.e.}, while we use our initialization point for both schemes, step size selections are different. 
Figure \ref{fig:01} shows some convergence rate results: clearly our step size selection performs better in practice, in terms of the total number of iterations required for convergence.

In the second experiment, we would like to highlight the time bottleneck introduced by the projection operations: for this aim, \emph{we use the same initialization points and step sizes} for both the algorithms under comparison. Thus, the only difference lies in the SVD computations of \texttt{SVP} (SDP) to retain a PSD low rank estimate per iteration. Table \ref{tbl:CompSVP} presents reconstruction error and execution time results. It is obvious that projecting on the low-rank PSD code per iteration constitutes a computational bottleneck per iteration, which slows down (w.r.t. total time required) the convergence of \texttt{SVP} (SDP). 

\paragraph{Initialization.} Here, we evaluate the importance of our initialization point selection:
\begin{equation}
\X^0 ~ := ~ \mathcal{P}_+ \left ( \frac{-\gradf(0)}{\| \gradf(0)-\gradf(e_1 e_1')\|_F} \right )
\label{eq:X0a}
\end{equation} To do so, we consider the following settings: we compare random initializations against the rule \eqref{eq:X0a}, both for constant step size selections and our step size selection. In all cases, we work with the factored parametrization.

Figure \ref{fig:02} shows the results. Left panel presents results for constant step size selections where $\eta = \sfrac{0.1}{\|\U\|_F^2}$ and right panel uses our step size selection; again, note that the selection of the constant step size is after many trial-and-errors for best step size selection, based on the specific configuration. Both figures compare the performance of factored gradient descent when $(i)$ a random initialization point is selected and, $(ii)$ our initialization is performed, according to \eqref{eq:X0a}. All curves depict median reconstruction errors over 20 Monte Carlo iterations. For all cases, the number of measurements is fixed to $C_{\text{sam}} \cdot n \cdot r$ for $C_{\text{sam}} = 10,$ $n = 1024$ and rank $r = 20$.

\begin{figure}[!ht]
	\begin{center}
		\includegraphics[width=0.45\textwidth]{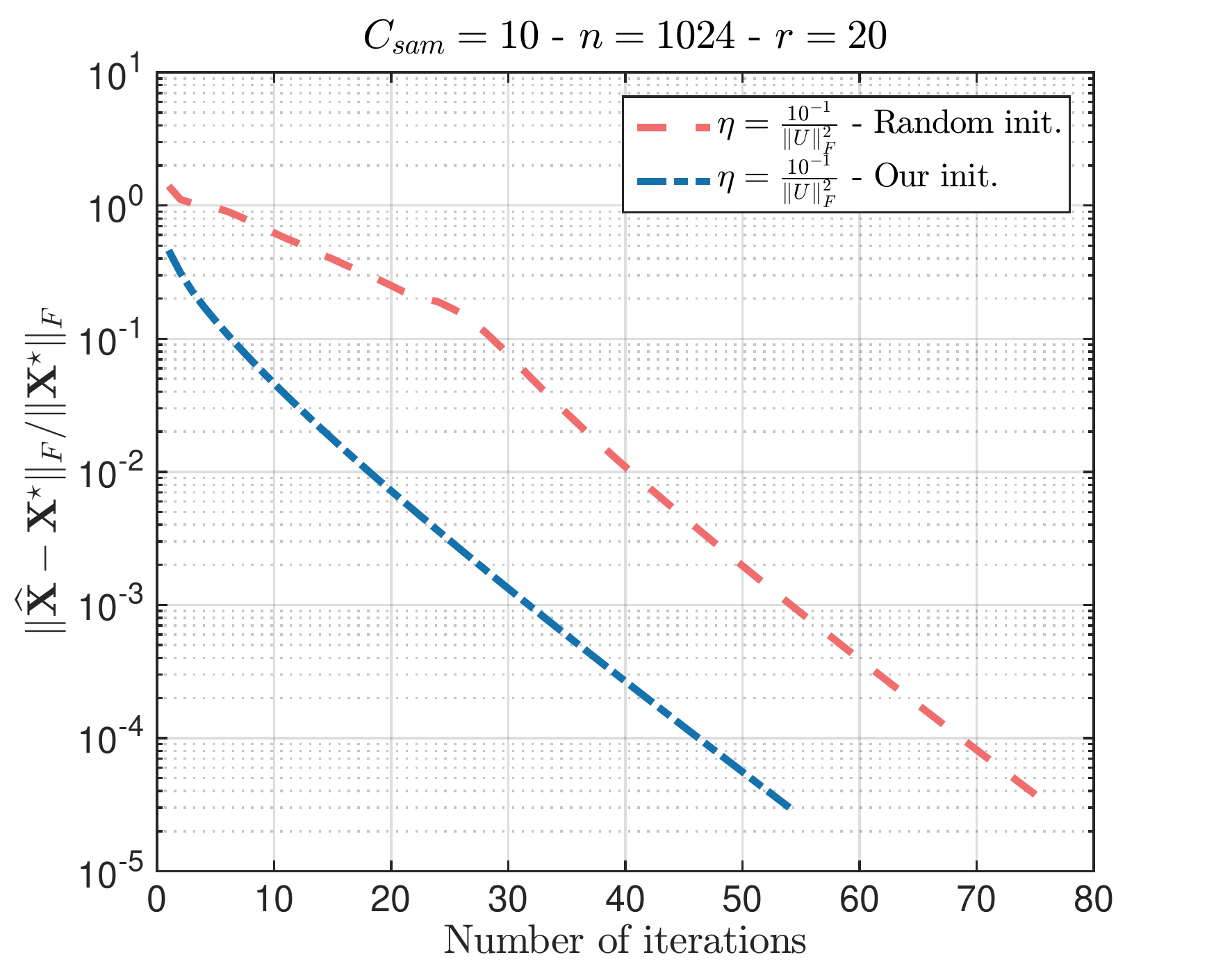}
		\includegraphics[width=0.45\textwidth]{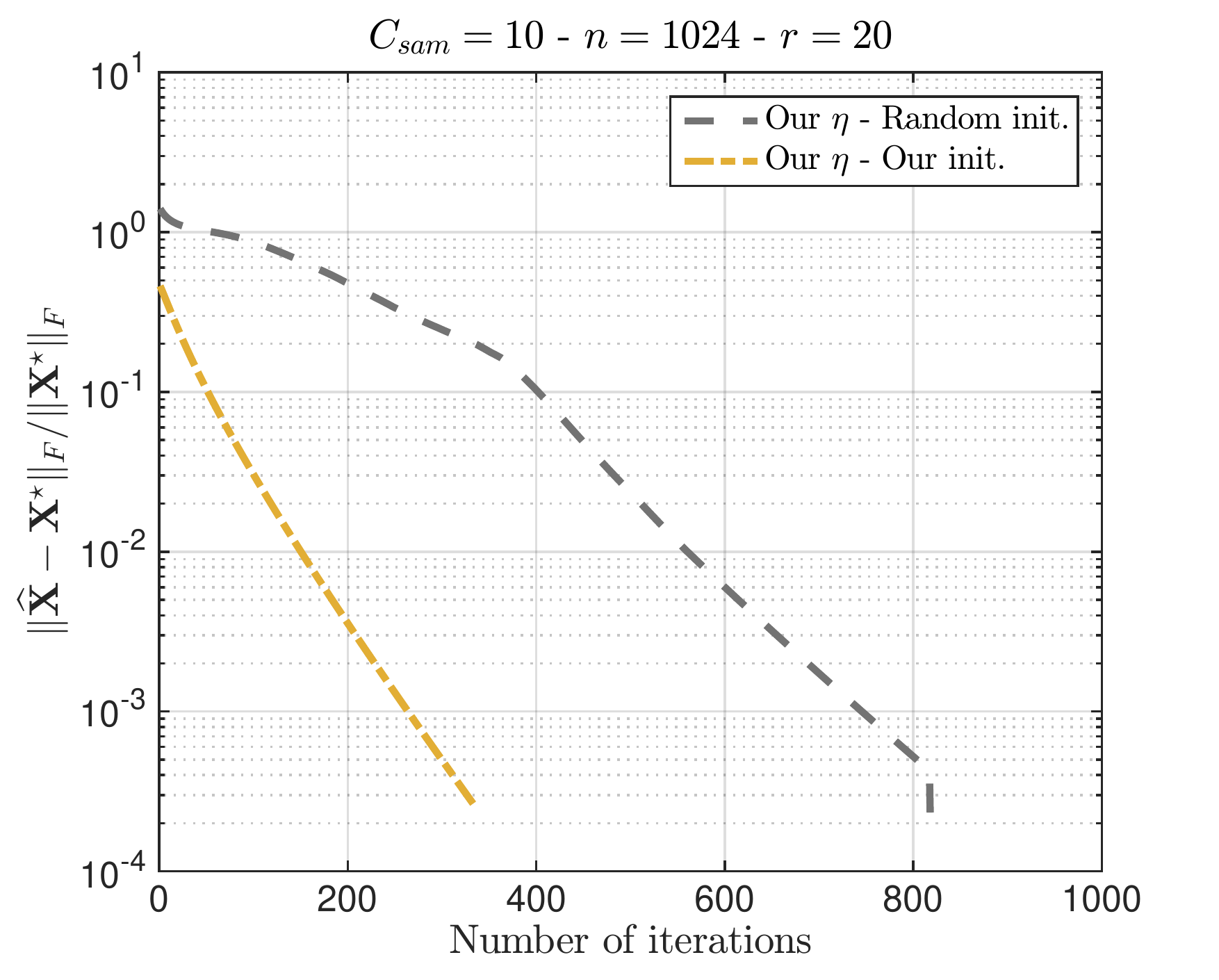}
	\end{center}
	\caption{Median error per iteration for different initialization set ups. Left panel presents results for constant step size selections where $\eta = \sfrac{0.1}{\|\U\|_F^2}$ and right panel uses our step size selection. Both figures compare the performance of factored gradient descent when $(i)$ a random initialization point is selected and, $(ii)$ our initialization is performed, according to \eqref{eq:X0a}. All curves depict median reconstruction errors over 20 Monte Carlo iterations. For all cases, the number of measurements is fixed to $C_{\text{sam}} \cdot n \cdot r$ for $C_{\text{sam}} = 10,$ $n = 1024$ and rank $r = 20$. }{\label{fig:02}}
\end{figure}

\paragraph{Dependence of $\alpha $ on $\frac{\sigma_1(\Xo)}{\sigma_r(\Xo)}$.} Here, we highlight the dependence of $\frac{\sigma_1(\Xo)}{\sigma_r(\Xo)}$ on the convergence rate of factored gradient descent. Consider the following matrix sensing toy example: let $\Xo := \Uo\left(\Uo\right)^\top \in \R^{n \times n}$ for $n = 50$ and assume $\text{rank}(\Xo) > r$. We desire to compute a (at most) rank-$r$ approximation of $\Xo$ by minimizing the simple least squares loss function:
\begin{equation} \label{exp:eq_01}
\begin{aligned}
& \underset{\X \in \R^{n \times n}}{\text{minimize}}
& & \frac{1}{2} \|\X - \Xo\|_F^2 \\
& \text{subject to}
& & \X \succeq 0, \quad \text{rank}(\X) \leq r
\end{aligned}
\end{equation} For this example, let us consider $r = 3$ and design $\Xo$ according to the following three scenarios: we fix $\sigma_1(\Xo) = \sigma_2(\Xo) = 100$ and vary $\sigma_3(\Xo) \in \left\{1, 10, 20\right\}$. This leads to condition numbers for these three cases as: $(i)$ $\frac{\sigma_1(\Xo)}{\sigma_3(\Xo)} = 100$, $(ii)$ $\frac{\sigma_1(\Xo)}{\sigma_3(\Xo)} = 10$ and, $(iii)$ $\frac{\sigma_1(\Xo)}{\sigma_3(\Xo)} = 5$. The convergence behavior is shown in Figure \ref{fig:03}(Left panel). It is obvious that factored gradient descent suffers -- w.r.t. convergence rate -- as the condition number $\frac{\sigma_1(\Xo)}{\sigma_3(\Xo)}$ get worse; especially, for the case where $\frac{\sigma_1(\Xo)}{\sigma_3(\Xo)} = 100$, factored gradient descent reaches a plateau after the $\sim$80-th iteration, where the steps towards solution become smaller. As the condition number improves, factored gradient descent enjoys faster convergence to the optimum, which shows the dependence of the algorithm on $\frac{\sigma_1(\Xo)}{\sigma_3(\Xo)}$ also in practice.

As a second setting, we fix $r = 2$, thereby computing a rank-$2$ approximation. As Figure \ref{fig:03}(Right panel) illustrates, for all values of $\sigma_3(\Xo)$, factored gradient descent performs similarly, enjoying fast convergence towards the optimum $\Xo$. Thus, while the condition number of original $\Xo$ varies to a large degree for $r = 3$, the convergence rate factor $\alpha$ only depends on $\frac{\sigma_1(\Xo)}{\sigma_2(\Xo)} = 1$, for $r = 2$. This leads to similar convergence behavior for all three scenarios described above.
\begin{figure}[!ht]
	\begin{center}
		\includegraphics[width=0.45\textwidth]{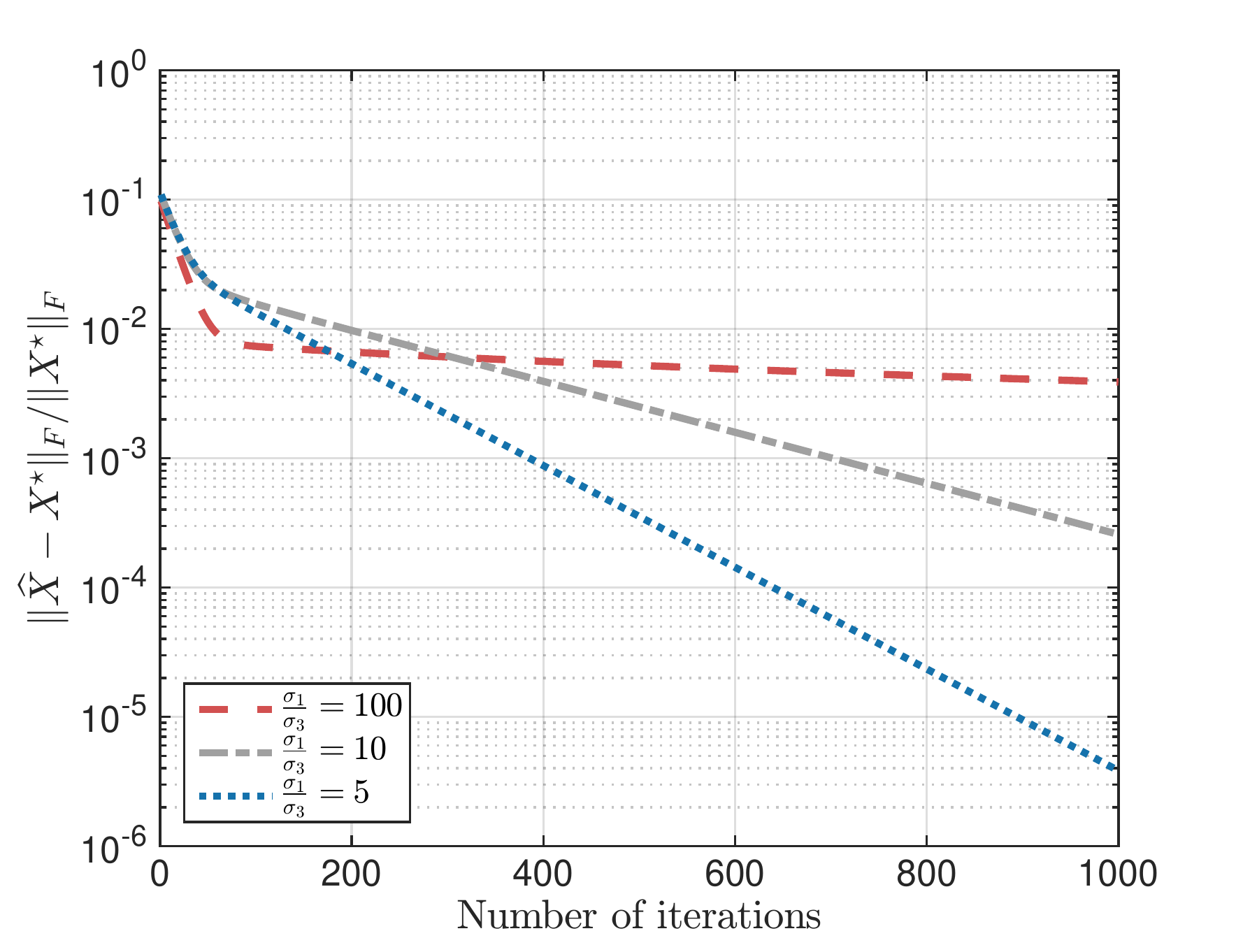}
		\includegraphics[width=0.45\textwidth]{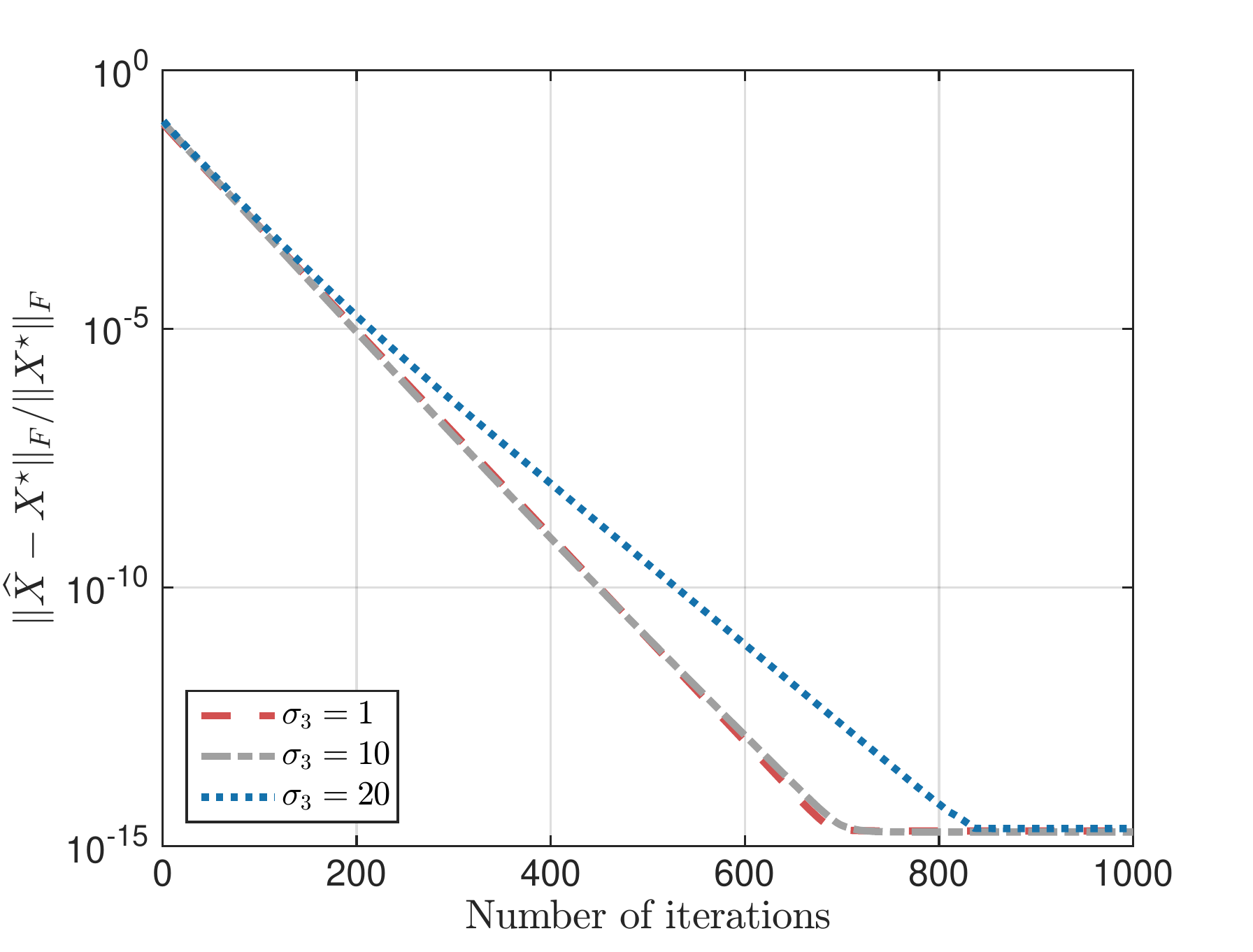}
	\end{center}
	\caption{Toy example on the dependence of $\alpha$ on the term $\frac{\sigma_1(\Xo)}{\sigma_r(\Xo)}$. Here, $\Xo := \Uo\left(\Uo\right)^\top \in \R^{n \times n}$ for $n = 50$. We use factored gradient descent to solve \eqref{exp:eq_01} for $r = 3$. Left panel: As condition number $\frac{\sigma_1(\Xo)}{\sigma_3(\Xo)}$ improves, factored gradient descent enjoys faster convergence in practice, as dictated by our theory. Right panel: convergence rate behavior of factored gradient descent when $r = 2$ in \eqref{exp:eq_01}.}{\label{fig:03}}
\end{figure}

\section{Test case II: Quantum State Tomography}{\label{sec:QST}}
As a second example, we consider the \emph{quantum state tomography} (QST) problem. QST can be described as follows:
\begin{equation}{\label{eq: orig QT}}
\begin{aligned}
	& \underset{\X \succeq 0}{\text{minimize}}
	& &  \|\linmap(\X) - y\|_F^2 \\
	& \text{subject to}
	& & \text{rank}(\X)= r, \;\trace(\X)\leq1.
\end{aligned}
\end{equation} 
In this problem, we look for a \emph{density matrix} $\X^{\star} \in \mathbb{C}^{n\times n}$ of a $q$-bit quantum system from a set of QST measurements $y \in \R^m,~m \ll n^2,$ that satisfy $y = \linmap(\X^{\star}) + \eta$.
Here, $(\linmap(\X^{\star}))_i = \trace(E_i \X^{\star})$ and $\eta_i$ could be modeled as zero-mean Gaussian. 
The operators $E_i \in  \R^{n \times n}$ are typically the tensor product of the $2\times 2$ Pauli matrices \cite{liu2011universal}. 
The density matrix is a priori known to be Hermitian, positive semi-definite matrix that satisfies $\text{rank}(\X^{\star}) = r$ and is normalized as $\trace(\X^{\star}) = 1$ \cite{gross2010quantum}; here, $n = 2^q$. 
Our task is to recover $\X^\star$. 

One can easily transform \eqref{eq: orig QT} into the following re-parameterized formulation:
\begin{equation}{\label{transform:eq_01}}
\begin{aligned}
	& \underset{\U \in \R^{n \times r}}{\text{{\rm minimize}}}
	& & \|\mathcal{A}(\U\U^\top) - y\|_2^2 \\
	& \text{{\rm subject to}}
	& &\|\U\|_F^2 \leq 1.
\end{aligned}
\end{equation} It is apparent that this problem formulation is not included in the cases \algo naturally solve, due to the Frobenius norm constraint on the factor $\U$. However, as a heuristic, one can alter \algo to include such projection: per iteration, each putative solution $\Up$ can be trivially projected onto the Frobenious norm ball $\|\U\|_F^2 \leq 1$; let us call this heuristic \texttt{projFGD}.
We compare such algorithm with state-of-the-art scheme for QST to show the merits of our approach.
The analysis of such constraint cases is an interesting and important extension of this paper and is left for future work.

\paragraph{State-of-the-art approaches.}
One of the first provable algorithmic solutions for the QST problem was through convexification \cite{fazel2002matrix, recht2010guaranteed, candes2009exact}:
this includes nuclear norm minimization approaches \cite{gross2010quantum} (using \textit{e.g.}, \cite{becker2011templates}), as well as proximal variants \cite{flammia2012quantum} (using \textit{e.g.}, \cite{cai2010singular}).
Recently, \cite{yurtsever2015universal} presented a universal primal-dual convex framework, which includes the QST problem as application and outperforms the above approaches, both in terms of recovery performance and execution time.

From a non-convex perspective, apart from Hazan's algorithm \cite{hazan2008sparse} -- see Related work,  \cite{becker2013randomized} propose Randomized Singular Value Projection (\texttt{RSVP}), a projected gradient descent algorithm for \eqref{eq: orig QT}, which merges gradient calculations with truncated SVDs via randomized approximations for computational efficiency. 

\paragraph{Experiments.}
Figure \ref{fig:exp1} (two-leftmost plots) illustrates the iteration and timing complexities of each algorithm under comparison, for a pure state density recovery setting ($r = 1$). Here, $q = 12$ which corresponds to a $\tfrac{n(n+1)}{2} = 8,390,656$ dimensional problem; moreover, we assume $C_{\rm sam} = 3$ and thus the number of measurements are $m = 12,288$. For initialization, we use the proposed initialization for all algorithms. It is apparent that \algo converges faster to a vicinity of $\X^\star$, as compared to the rest of the algorithms; observe also the sublinear rate of \texttt{SparseApproxSDP} in the inner plots, as reported in \cite{hazan2008sparse}.

\begin{figure*}[t!]
\centering
\includegraphics[width=0.33\textwidth]{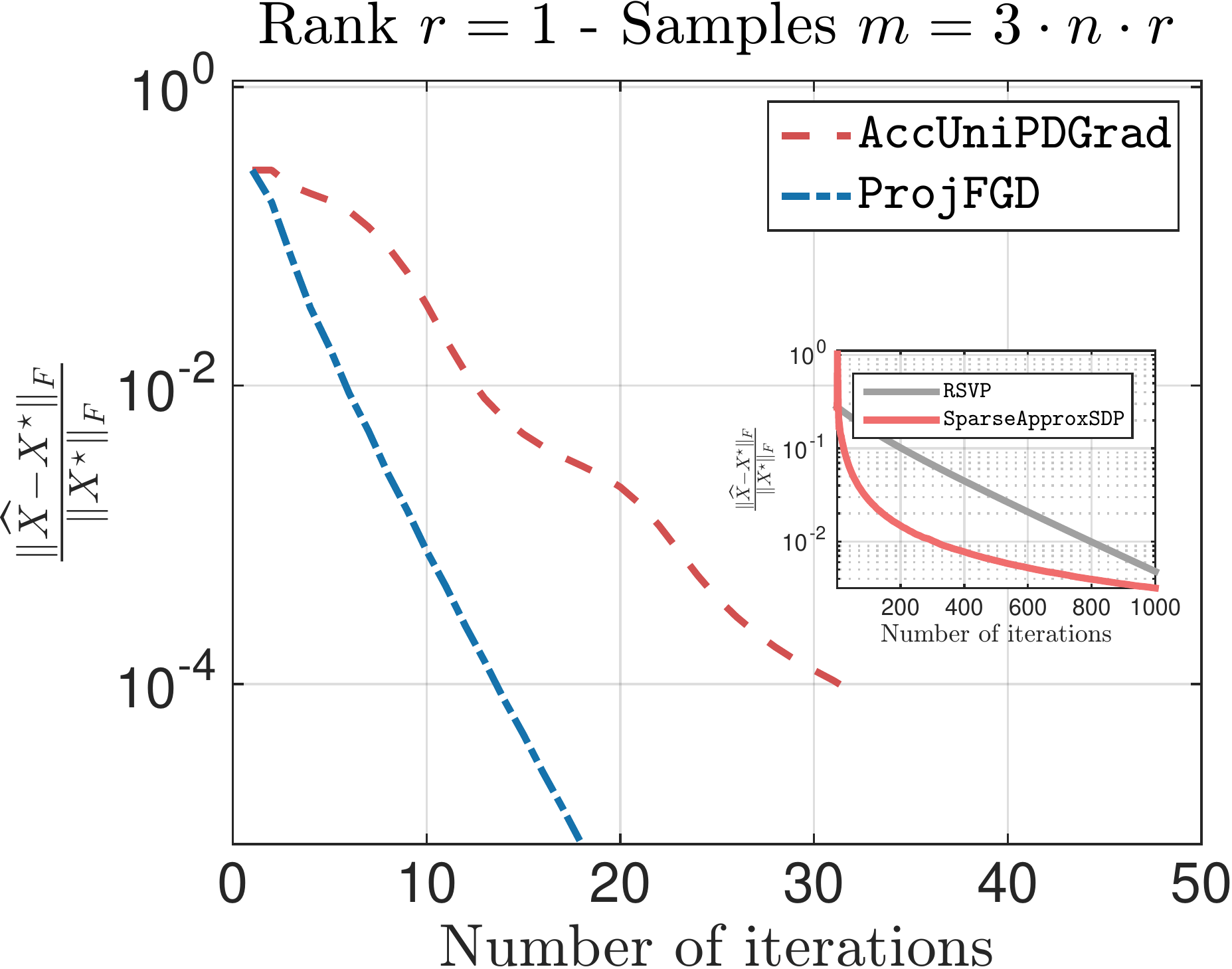} \hspace{-0.3cm}
\includegraphics[width=0.33\textwidth]{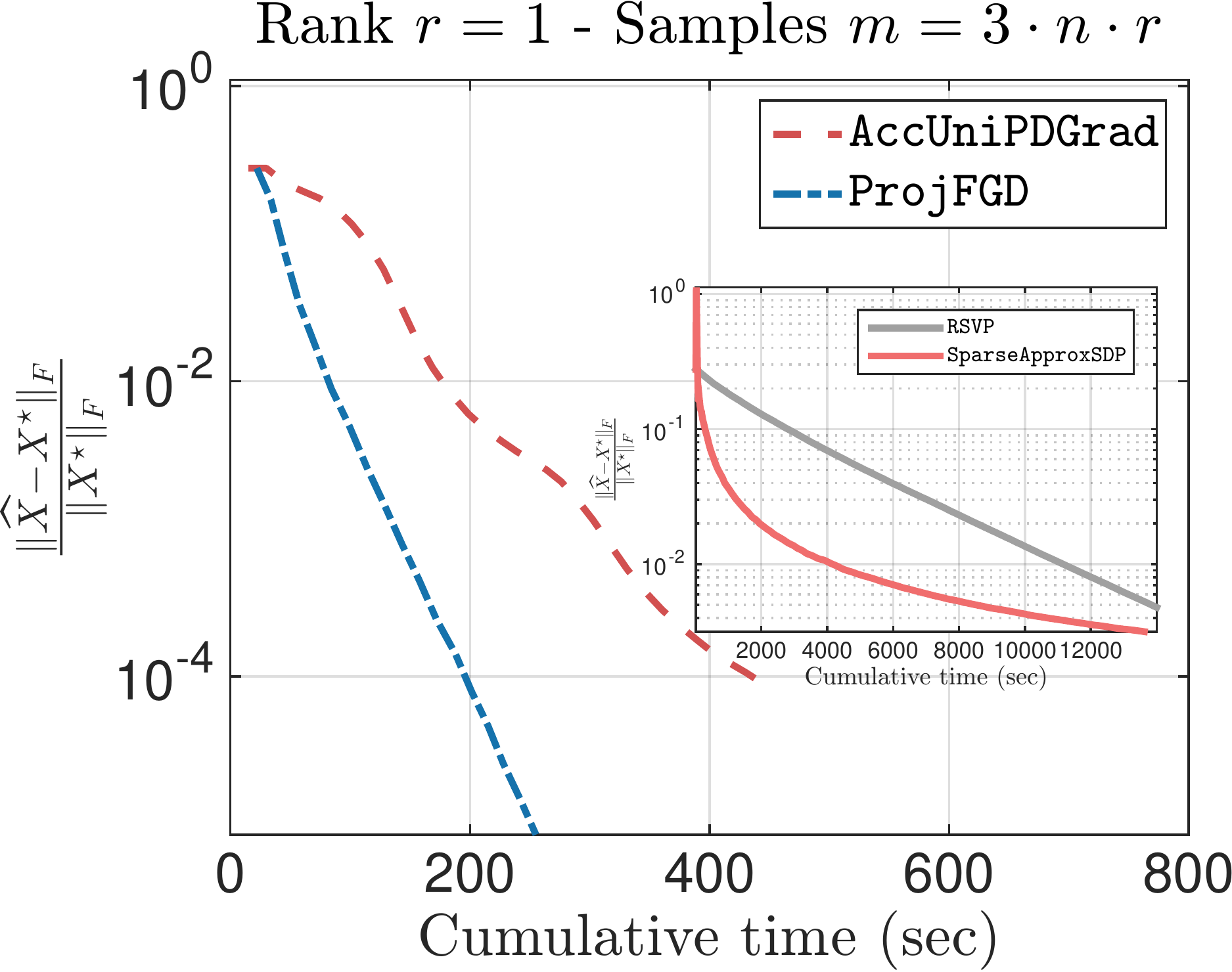} \hspace{-0.3cm}
\includegraphics[width=0.33\textwidth]{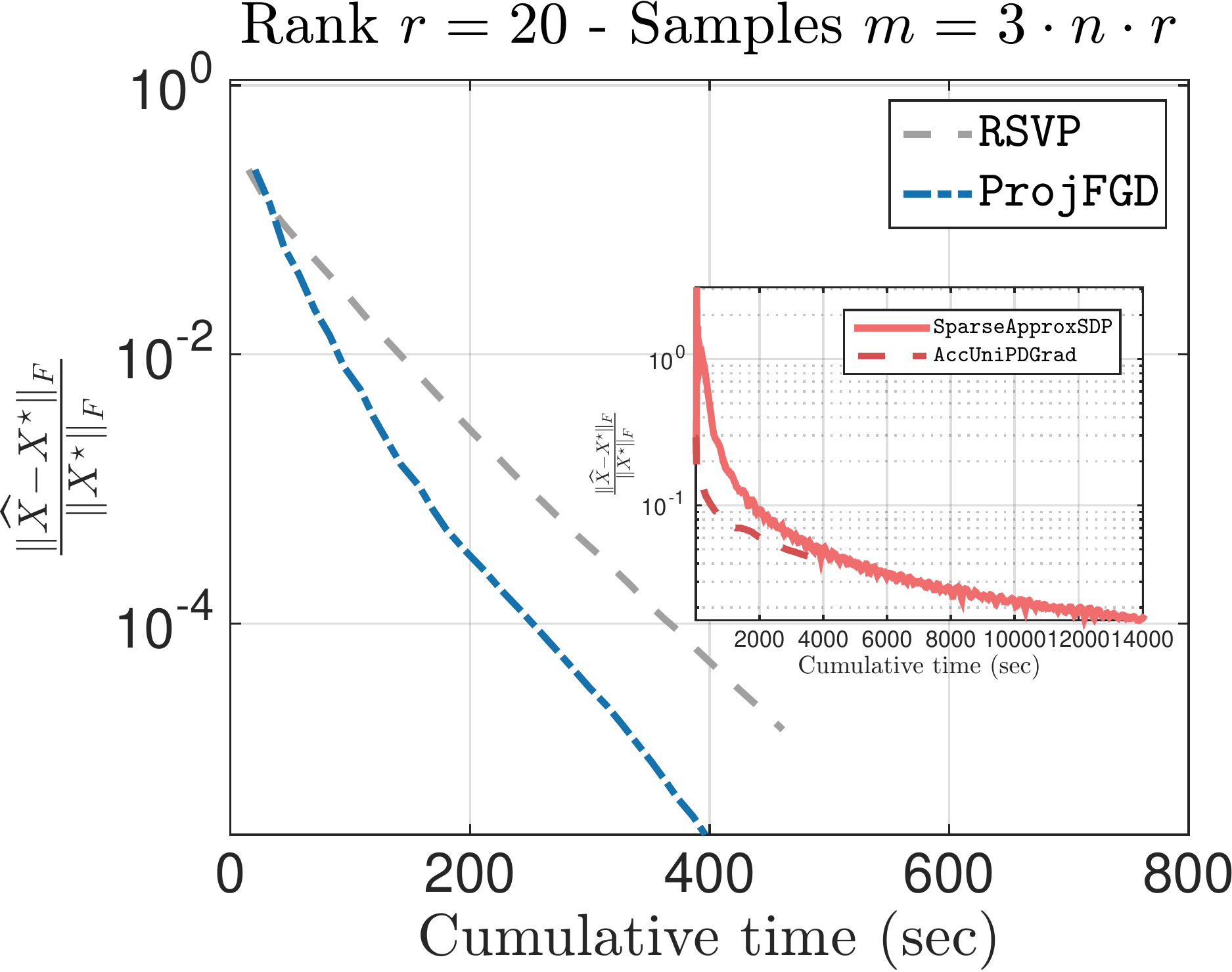} 
\caption{\textbf{Left and middle panels}: Convergence performance of algorithms under comparison w.r.t. $\tfrac{\|\widehat{X} - \X^\star\|_F}{\|\X^\star\|_F}$ vs. $(i)$ the total number of iterations (left) and $(ii)$ the total execution time. Both cases correspond to the case $C_{\rm sam} = 3$, $r = 1$ (pure state setting) and $q = 12$ (\textit{i.e.}, $n = 4096$). \textbf{Right panel}: Almost pure state ($r = 20$). Here, $C_{\rm sam} = 3$. 
}
\label{fig:exp1}
\end{figure*}
\section{Test case III: PSD problems with high-rank solutions}{\label{sec:high_rank}}

As a final example, we consider problems of the form:
\begin{equation*}
\begin{aligned}
\underset{\X \in \R^{n \times n}}{\text{minimize}}
& & f(\X) \quad \text{subject to}
& & \X \succeq 0,
\end{aligned} 
\end{equation*} 
where $\X^\star$ is the minimizer of the above problem and $\text{rank}(\X^\star) = O(n)$. 
In this particular case and assuming we are interested in finding high-ranked $\X^\star$, we can reparameterized the above problem as follows:
\begin{equation*}
\begin{aligned}
\underset{\U \in \R^{n \times O(n)}}{\text{minimize}}
& & f(\U\U^\top).
\end{aligned} 
\end{equation*} Observe that $\U$ is a square $n \times O(n)$ matrix. Under this setting, \algo performs the recursion:
\begin{align*}
\underbrace{\Up}_{n \times O(n)} = \underbrace{\U}_{n \times O(n)} - \eta \underbrace{\gradf(\U\U^\top)}_{n \times n}  \cdot \underbrace{\U}_{n \times O(n)}.
\end{align*} Due to the matrix-matrix multiplication, the per-iteration time complexity of \algo is $O(n^3)$, which is comparable to a SVD calculation of a $n \times n$ matrix. In this experiment, we study the performance of \algo in such high-rank cases and compare it with state-of-the-art approaches for PSD constrained problems. 

\begin{table}[!b]
\centering
\begin{footnotesize}
\begin{tabular}{c c c c c c c}
\toprule
Algorithm & & $\|\widehat{\X} - \X^\star\|_F / \|\X^\star\|_F$ & & Total time (sec) & & Time per iter. (sec - median) \\ 
\midrule
\multicolumn{7}{c}{Setting: $n = 1024$, $r = n/4$, $m = 2nr$.} \\ 
\midrule
\texttt{RSVP} & & 4.9579e-04 & & 1262.3719 & & 2.1644 \\ 
\texttt{SparseApproxSDP} & & 3.3329e-01 & & 895.9605 & & 2.1380e-01 \\ 
\texttt{FGD} & & 1.6763e-04 & & 57.8495 & & 2.1961e-01 \\ 
\midrule 
\multicolumn{7}{c}{Setting: $n = 2048$, $r = n/4$, $m = 2nr$.} \\ 
\midrule
\texttt{RSVP} & & 4.9537e-04 & & 8412.6981 & & 14.6811 \\ 
\texttt{SparseApproxSDP} & & 3.3526e-01 & & 26962.0379 & & 8.7761e-01 \\ 
\texttt{FGD} & & 1.6673e-04 & & 272.8102 & & 1.0040e+00 \\ 
\midrule 
\multicolumn{7}{c}{Setting: $n = 2048$, $r = n/8$, $m = 4nr$.} \\ 
\midrule
\texttt{RSVP} & & 2.4254e-04 & & 1945.6714 & & 5.9763 \\ 
\texttt{SparseApproxSDP} & & 9.6725e-02 & & 3506.8147 & & 8.6440e-01 \\ 
\texttt{FGD} & & 3.8917e-05 & & 68.5689 & & 9.2567e-01\\ 
 \midrule  
\end{tabular}
\end{footnotesize}
\caption{Comparison of related work in high-rank matrix sensing problems. We construct $\X^\star$ with $\trace(\X^\star) = 1$ such that \cite{hazan2008sparse} applies. It is apparent that avoiding SVDs helps in practice.} \label{table:high_rank_summary} 
\end{table}

For the purpose of this experiment, we only consider first-order solvers; \emph{i.e.}, second order methods such as interior point methods are excluded as, in high dimensions, it is prohibitively expensive the hessian of $f$. To this end, the algorithms to compare include: $(i)$ standard projected gradient descent approach \cite{kyrillidis2014matrix} and $(ii)$ Frank-Wolfe type of algorithms, such as the one in \cite{hazan2008sparse}. We note that this experiment can be seen as a proof of concept on how avoiding SVD calculations help in practice.\footnote{Here, we assume a standard Lanczos implementation of SVD, as the one provided in Matlab enviroment.}

\paragraph{Experiments.}
We consider the simple example of matrix sensing \cite{kyrillidis2014matrix}: we obtain a set of measurements $y \in \R^{m}$ according to the linear model:
\begin{align*}
y = \mathcal{A}(\X^\star).
\end{align*} Here, $\mathcal{A}: \R^{n \times n} \rightarrow \R^{m}$ is a sensing mechanism such that $\left(\mathcal{A}(\X)\right)_i = \trace(A_i \X)$ for some Gaussian random matrices $A_i, ~i = 1, \dots, m$. The ground truth matrix $\X^\star$ is design such that $\text{rank}(\X^\star) = \sfrac{n}{4}$ and $\trace(\X^\star) = 1$.\footnote{The reason we design $\X^\star$ such that $\trace(\X^\star)$ is such that the algorithm \texttt{SparseApproxSDP} \cite{hazan2008sparse} applies; this is due to the fact that \texttt{SparseApproxSDP} is designed for QST problems, where trace constraint is present in the optimization criterion---see also \ref{eq: orig QT}, without the rank constraint.}

Figure \ref{fig:exp2} and Table \ref{table:high_rank_summary} show some results for the following settings: $(i)$ $n = 1024$, $r  = n/4$ and $m = 2nr$, $(ii)$ $n = 2048$, $r  = n/4$ and $m = 2nr$, $(iii)$ $n = 2048$, $r  = n/8$ and $m = 4nr$. From our finding, we observe that, even for high rank cases---where $r = O(n)$---performing matrix factorization and optimizing over the factors results into a much faster  convergence, as compared to low-rank projection algorithms, such as \texttt{RSVP}  in \cite{becker2013randomized}. 
Furthermore, \algo performs better than \texttt{SparseApproxSDP} \cite{hazan2008sparse} in practice: while \texttt{SparseApproxSDP} is a Frank-Wolfe type-of algorithm (and thus, the per iteration complexity is low), it admits \emph{sublinear} convergence which leads to suboptimal performance, in terms of total execution time. However, \texttt{RSVP} and \texttt{SparseApproxSDP} algorithms do not assume specific initialization procedures to work in theory.

\begin{figure}[t!]
\centering
\includegraphics[width=0.32\textwidth]{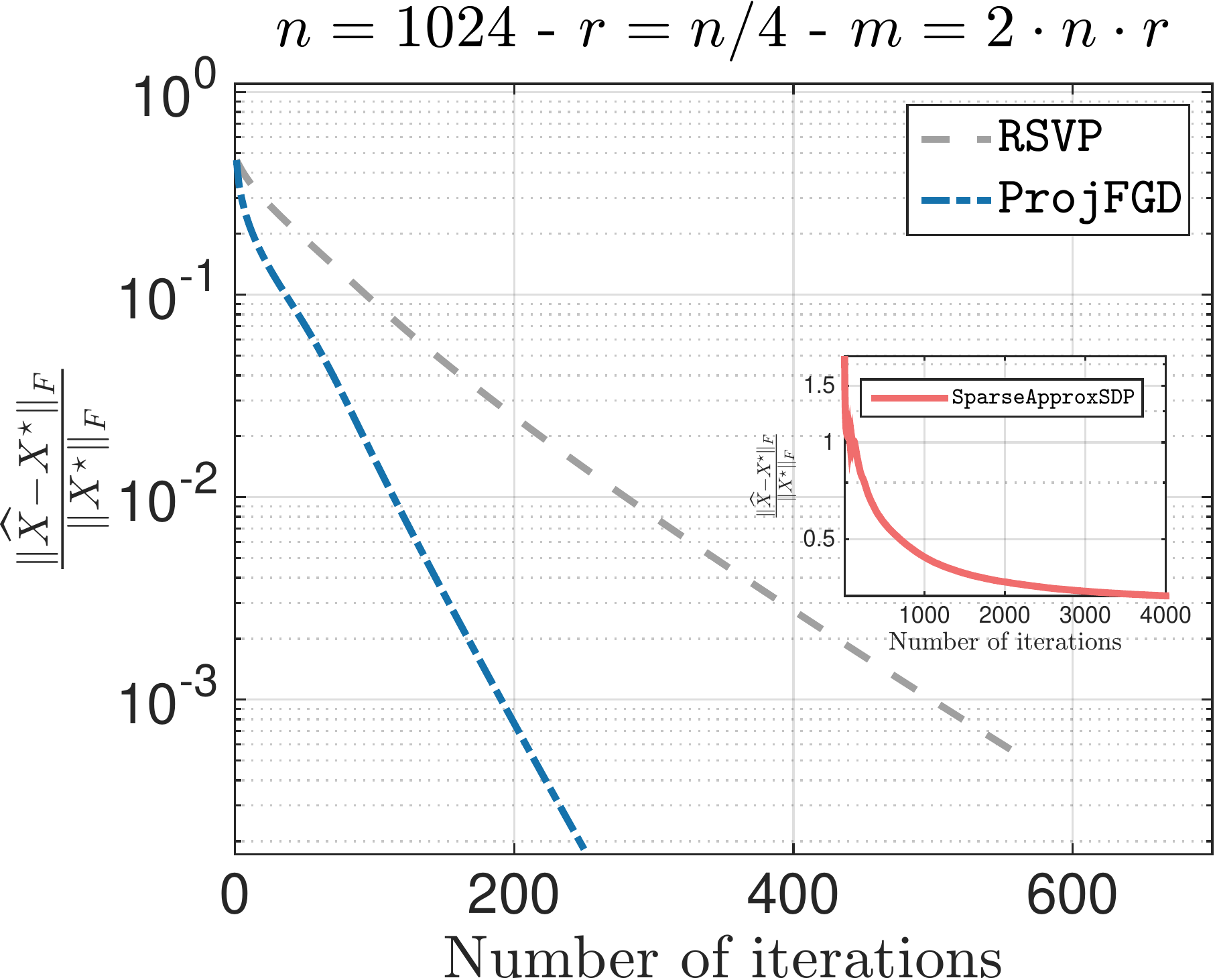}
\includegraphics[width=0.32\textwidth]{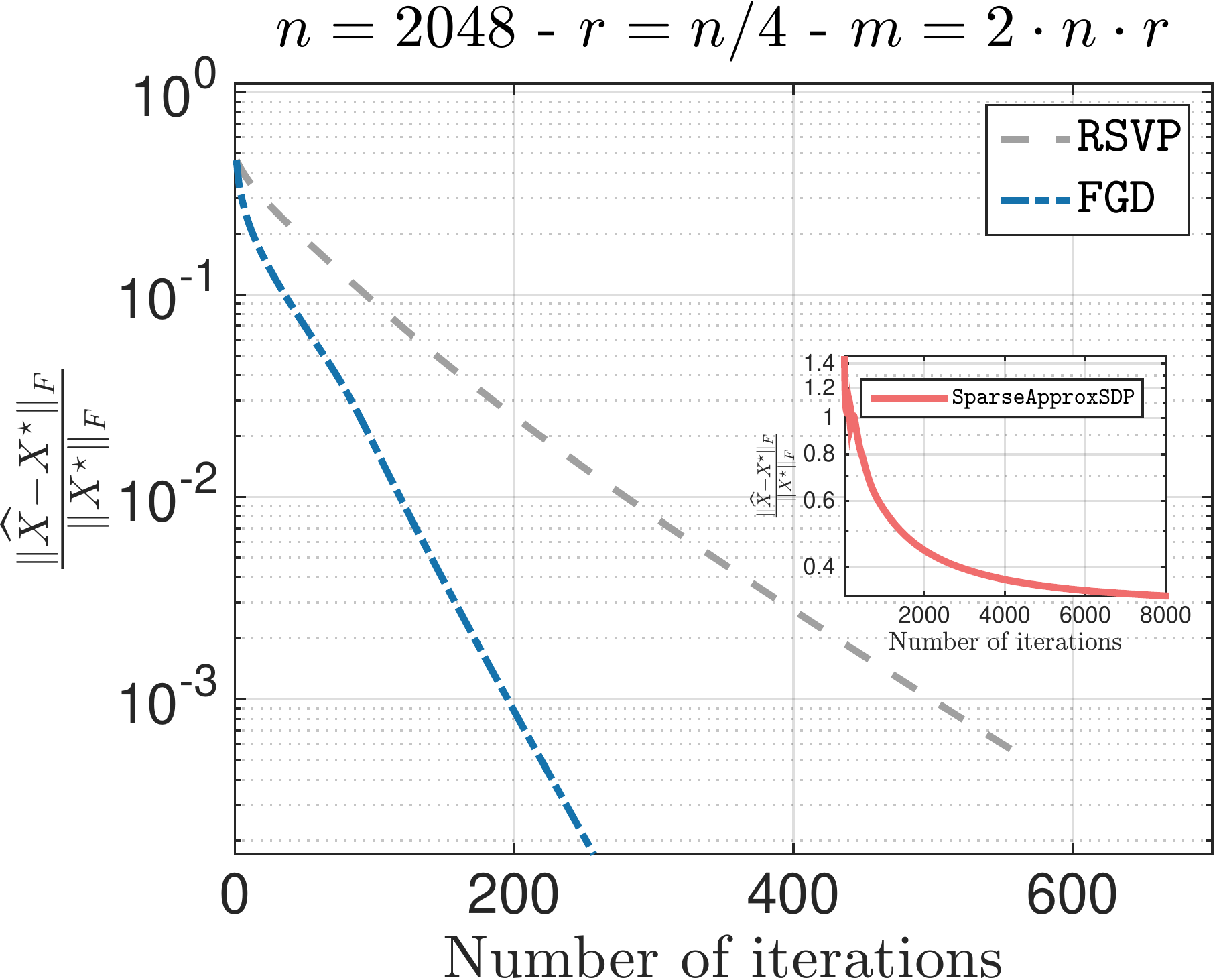}
\includegraphics[width=0.32\textwidth]{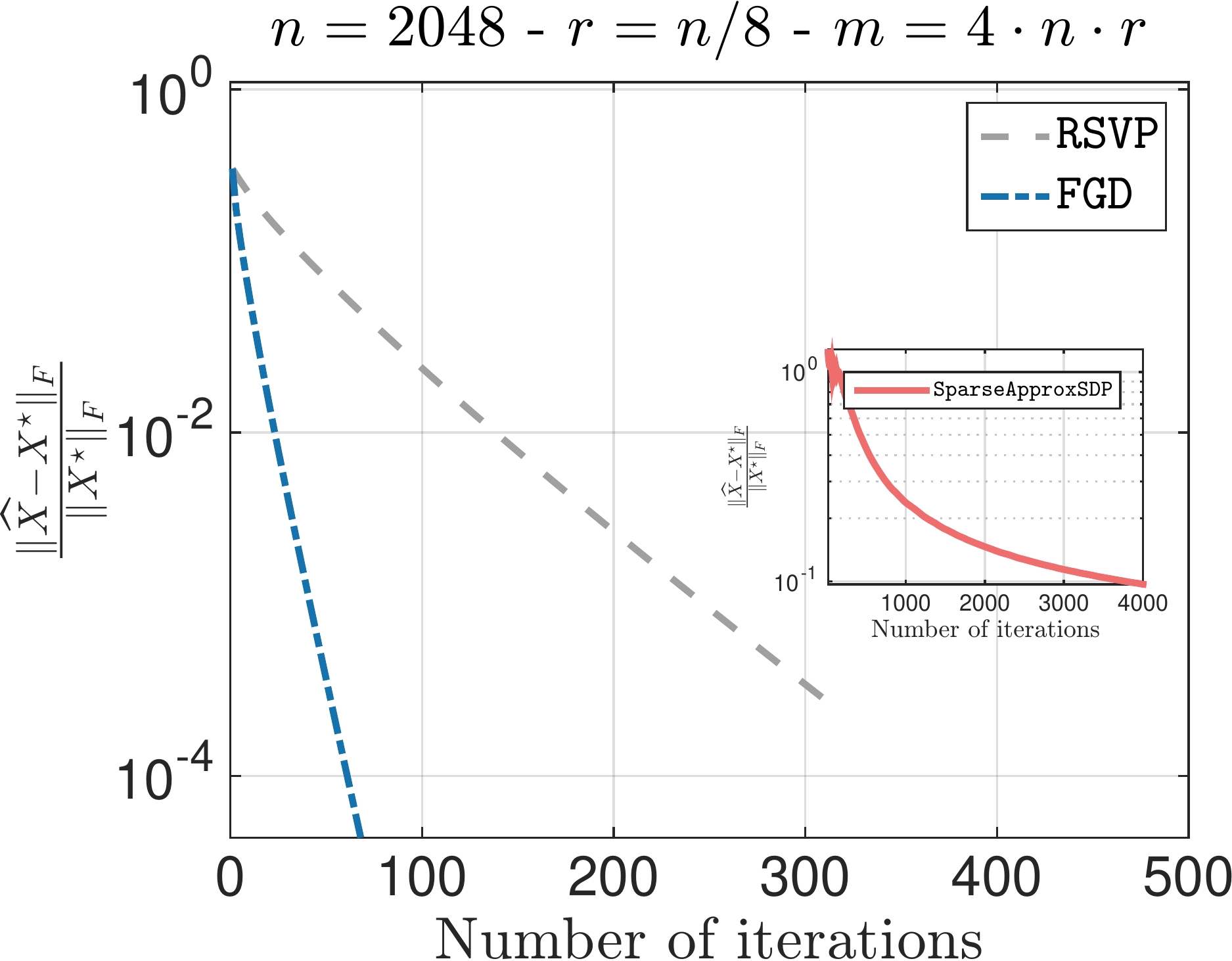}\\
\includegraphics[width=0.32\textwidth]{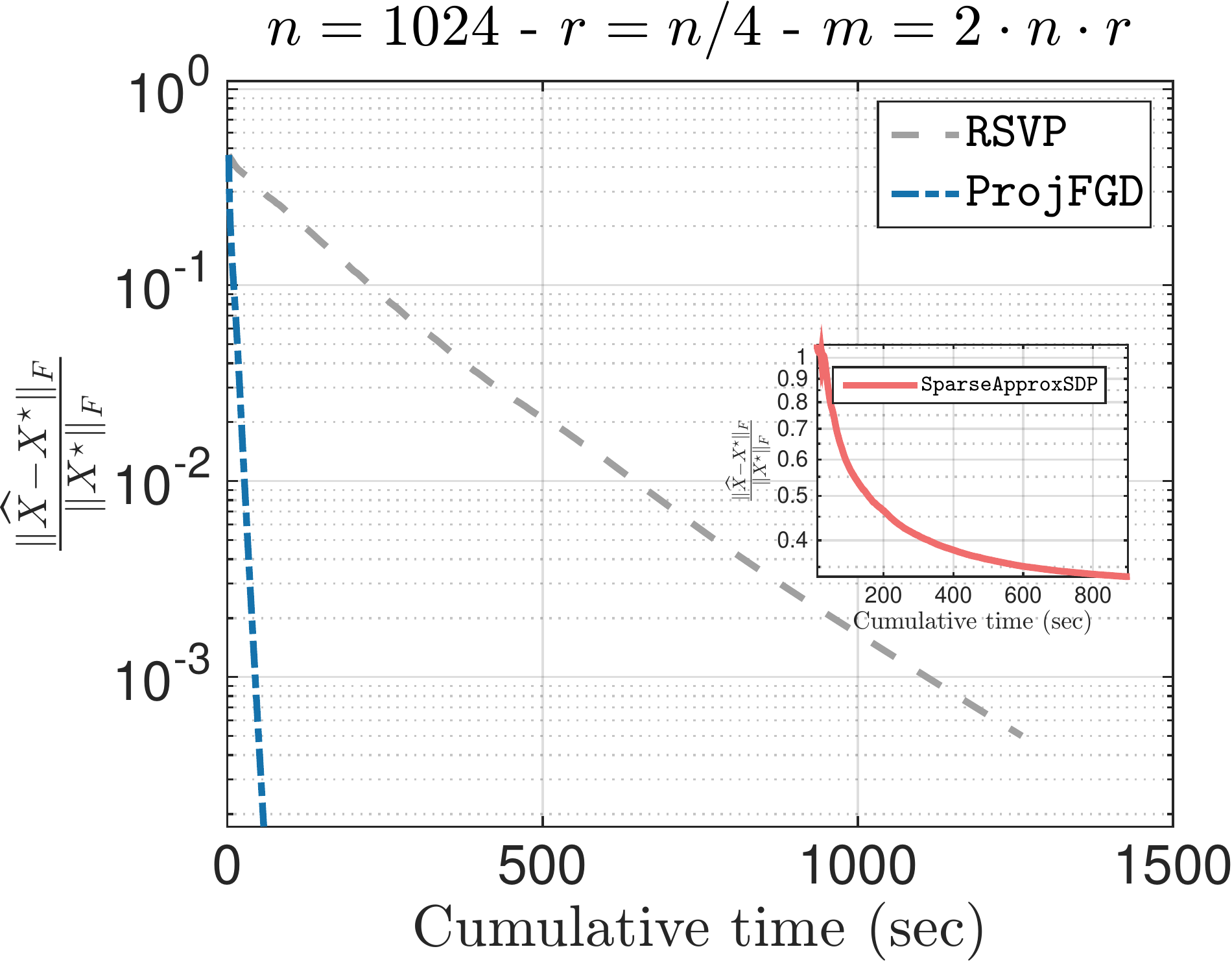}
\includegraphics[width=0.32\textwidth]{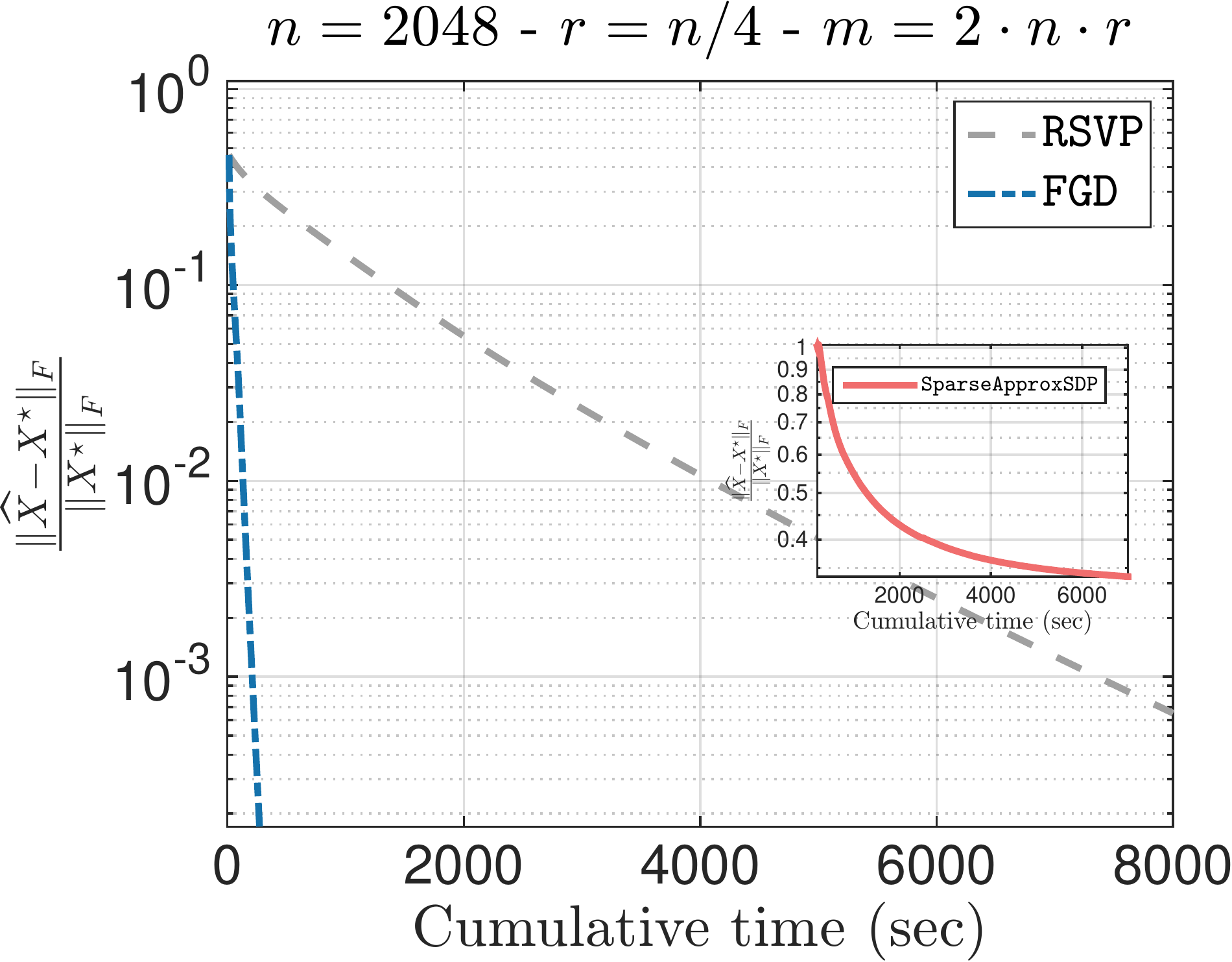}
\includegraphics[width=0.32\textwidth]{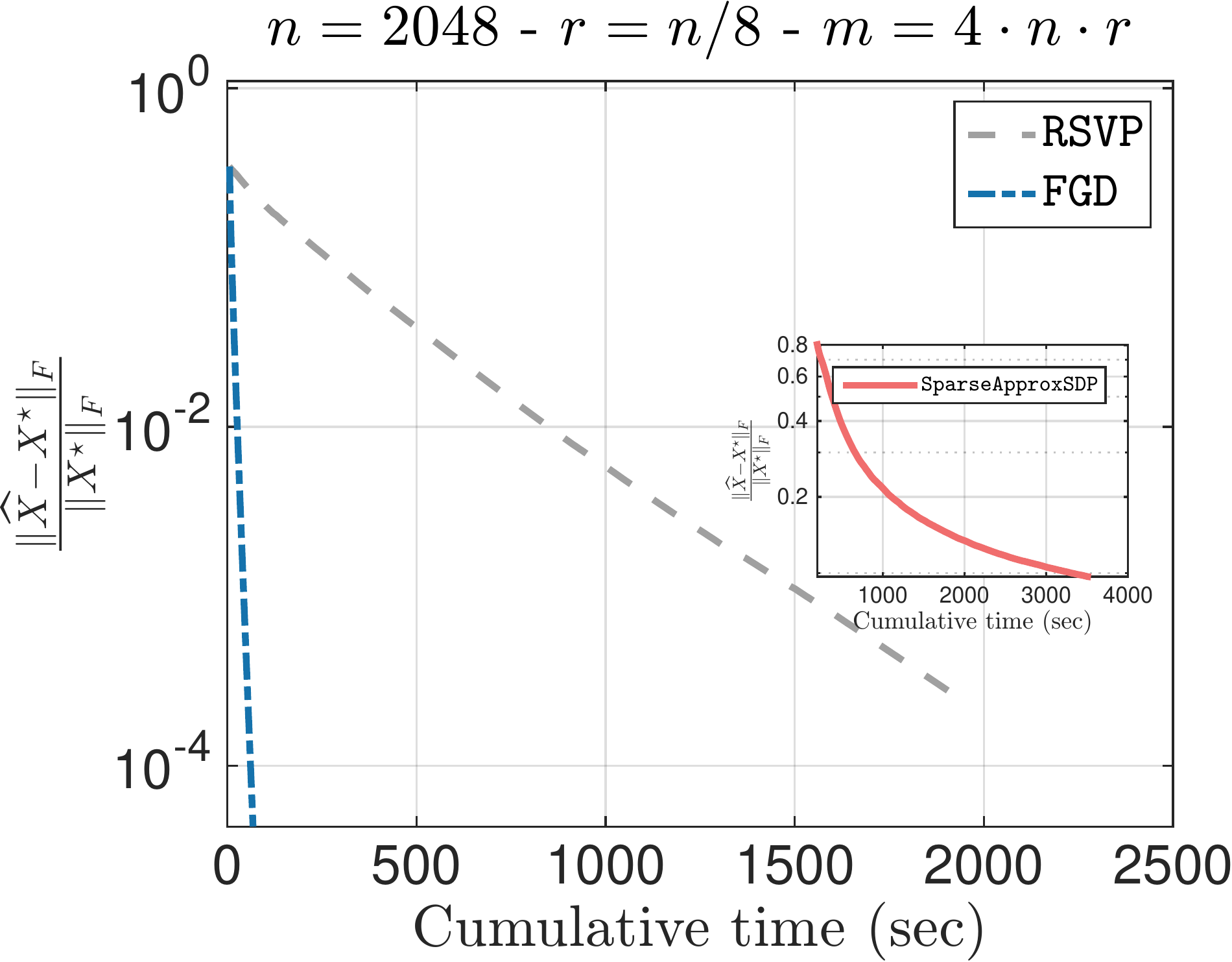}
\caption{Convergence performance of algorithms under comparison w.r.t. $\tfrac{\|\widehat{X} - \X^\star\|_F}{\|\X^\star\|_F}$ vs. $(i)$ the total number of iterations (top) and $(ii)$ the total execution time (bottom). 
}
\label{fig:exp2}
\end{figure}

\section{Convergence without tail bound assumptions}{\label{sec:no_tail}}
In this section, we show how assumptions (A3) and (A4) can be dropped by using a different step size $\eta$, where spectral norm calculation of two $n \times r$ matrices is required per iteration.
Here, we succinctly describe the main theorems and how they differ from the case where $\eta$ as in \eqref{eq:step_size} is used. We also focus only on the case of restricted strongly convex functions. Similar extension is possible without restricted strong convexity.

Our discussion is organized as follows: we first re-define key lemmas (\emph{e.g.}, descent lemmas, etc.) for a different step size; then, we state the main theorems and a sketch of their proof. In the analysis below we use as step size:
$$
\weta = \frac{1}{16 \, (M \norm{X}_2 + \norm{\gradf(X)Q_UQ_U^\top}_2)}
$$

\subsection{Key lemmas}
Next, we present the main descent lemma that is used for both sublinear and linear convergence rate guarantees of \algo. 

\begin{lemma}[Descent lemma]\label{lem:newgradU,U-U_r_ bound}
For $f$ being a $M$-smooth and $(m, r)$-strongly convex function and under assumptions $(A2)$ and $f(\Xp) \geq f(\Xor)$, the following inequality holds true:
\begin{align*}
\tfrac{1}{\weta}\ip{U -\Up}{\U -\Uorr} \geq  \tfrac{3}{5} \weta  \|\gradf(X) U\|_F^2 + \tfrac{3m}{20} \cdot  \sigma_{r}(\Xo)  \| \Delta\|_F^2.
\end{align*}
\end{lemma}

\begin{proof}[Proof of Lemma \ref{lem:newgradU,U-U_r_ bound}]
By \eqref{proofsr1:eq_09_main}, we have: 
\begin{align}
 \ip{\gradf(X) \U}{\U -\Uorr} = \frac{1}{2}\ip{\gradf(X) }{\X -\Xo_r}  + \frac{1}{2} \ip{\gradf(X) }{\Delta \Delta^\top}, \label{proofsr1:neweq_09_main}
\end{align} 

\noindent{\bf Step I: Bounding $\ip{\gradf(X) }{\X -\Xo_r}$.} For this term, we have a variant of Lemma \ref{lem:gradX,X-X_r_ bound_sc}, as follows:
\begin{lemma}\label{lem:newgradX,X-X_r_ bound_sc}
Let $f$ be a $M$-smooth and $(m, r)$-restricted strongly convex function with optimum point $\Xo$. Assume $f(\Xp) \geq f(\Xor)$. Let $\X = \U\U^\top$. Then,
\begin{align*}   
\ip{\gradf(\X)}{\X-\Xor}  \geq \tfrac{18\weta}{10} \|\gradf(X) U\|_F^2 +\tfrac{m}{2}\norm{\X - \Xor}_F^2,
\end{align*}
where  $\weta =\frac{1}{16 (M\|\X\|_2 + \|\gradf(\X)Q_U Q_U^\top\|_2)}$.
\end{lemma} The proof of this lemma is provided in Section~\ref{sec:newsupp_sc1}.\\

\noindent{\bf Step II: Bounding $\ip{\gradf(X) }{\Delta \Delta^\top}$.}
For the second term, we have the following variant of Lemma \ref{lem:DD_bound_sc}. 
\begin{lemma}\label{lem:newDD_bound_sc}
Let $f$ be $M$-smooth and $(m, r)$-restricted strongly convex. Then, under assumptions $(A2)$ and $f(\Xp) \geq f(\Xo_r)$, the following bound holds true:
\begin{align*}
\ip{\gradf(X) }{ \Delta \Delta^\top} \geq - \tfrac{6\weta}{25 } \|\gradf(\X) U \|_F^2 - \tfrac{3m\sigma_{r}(\Xo)}{40}  \cdot \| \Delta\|_F^2.
\end{align*} 
\end{lemma} Proof of this lemma can be found in Section~\ref{sec:newsupp_sc2}.\\

\noindent{\bf Step III: Combining the bounds in equation~\eqref{proofsr1:neweq_09_main}.} The rest of the proof is similar to that of Lemma \ref{lem:gradU,U-U_r_ bound}. 

\end{proof}

\subsection{Proof of linear convergence}\label{sec:newproofs_2}
For the case of (restricted) strongly convex functions $f$, we have the following revised theorem:

\begin{theorem}[Convergence rate for restricted strongly convex $f$]\label{thm:convergence_main_new}
Let current iterate be $\U$ and $\X = \U\U^\top$. Assume $ \dist(\U, \Uo_r) \leq \rho' \sigma_{r}(\Uo_r)$ and let the step size be $\weta =\frac{1}{16 \, (M \norm{X}_2 + \norm{\gradf(X)Q_UQ_U^\top}_2)} $. Then under assumptions $(A2)$ and $f(\Xp) \geq f(\Xor)$, the new estimate $\Up= U -\eta \gradf(X) \cdot U$ satisfies
\begin{equation}
\dist(\Up, \Uo_r)^2 \leq \alpha \cdot \dist(\U, \Uo_r)^2, \label{conv:neweq_00}
\end{equation}
where $\alpha = 1 -\frac{m \sigma_{r}(\Xo)}{64 (M\|\Xo\|_2 + \|\gradf(\Xor)\|_2)}$. Furthermore, $\Up$ satisfies $ \dist(\Up, \Uo_r) \leq \rho' \sigma_{r}(\Uo_r). $
\end{theorem}
 
The proof follows the same motions with that of theorem \ref{thm:convergence_main}, except from the fact Lemmas \ref{lem:newgradX,X-X_r_ bound_sc} and \ref{lem:newDD_bound_sc} are used.
\section{Main lemmas for convergence proof without tail bound assumptions}

\subsection{Proof of Lemma \ref{lem:newgradX,X-X_r_ bound_sc}}\label{sec:newsupp_sc1}
Let $\Up = U -\weta \gradf(\X) U$ and $\Xp =\Up (\Up)^\top$.  By smoothness of $\f$, we get:
\begin{align}
\f(\X) &\geq \f(\Xp) -\ip{\gradf(\X)}{\Xp -\X} - \tfrac{M}{2} \norm{\Xp -\X}_F^2 \nonumber \\ 
&\stackrel{(i)}{\geq} \f(\Xor) -\ip{\gradf(\X)}{\Xp -\X} - \tfrac{M}{2} \norm{\Xp -\X}_F^2 , \label{eq:newgradX,X-X_r_00}
 \end{align}
where $(i)$ follows from hypothesis of the lemma and since $\Xp$ is a feasible point $(\Xp \succeq 0)$ for problem~\eqref{intro:eq_00}. 
 Finally, since $\text{rank}(\Xor) = r$, by the $(\rscm, r)$-restricted strong convexity of $\f$, we get, 
\begin{align} 
\f(\Xor) \geq \f(\X) +\ip{\gradf(\X)}{\Xor-\X} +\tfrac{m}{2}\norm{\Xor -\X}_F^2 .\label{eq:newgradX,X-X_r_02}
\end{align}
Combining equations~\eqref{eq:newgradX,X-X_r_00} and~\eqref{eq:newgradX,X-X_r_02}, we obtain: 
\begin{small}
\begin{align} 
\ip{\gradf(\X)}{\X-\Xor} \geq  \ip{\gradf(\X)}{\X -\Xp} -\tfrac{M}{2} \norm{\Xp -\X}_F^2+\tfrac{m}{2}\norm{\Xor -\X}_F^2, \label{eq:newgradX,X-X_r_03} 
\end{align}
\end{small} instead of \eqref{eq:gradX,X-X_r_03} in the proof where $\eta$ is used. 
The rest of the proof follows the same motions with that of Lemma \ref{lem:gradX,X-X_r_ bound_sc} and we get:
\begin{align*}   
\ip{\gradf(\X)}{\X-\Xor}  \geq \tfrac{18 \weta}{10} \|\gradf(X) U\|_F^2 +\tfrac{m}{2}\norm{\Xor -\X}_F^2.
\end{align*} Moreover, for the case where $f$ is just $M$-smooth and $\Xo \equiv \Xor$, the above bound becomes:
 \begin{align*}   
\ip{\gradf(\X)}{\X-\Xor}  \geq \tfrac{18 \weta}{10} \|\gradf(X) U\|_F^2.
\end{align*} This completes the proof.

\subsection{Proof of Lemma \ref{lem:newDD_bound_sc}}\label{sec:newsupp_sc2}
Similar to Lemma \ref{lem:DD_bound_sc}, we have:
\begin{small}
\begin{align}
\|Q_{U} Q_{U}^\top\gradf(\X)\|_2 \cdot \| \Delta\|_F^2 =  \weta \left( 16 \underbrace{M\|X\|_2 \|Q_{U} Q_{U}^\top\gradf(\X)\|_2 \| \Delta\|_F^2}_{:=A} + 16 \|Q_{U} Q_{U}^\top\gradf(\X)\|_2^2 \cdot  \| \Delta\|_F^2  \right) \nonumber
\end{align} 
\end{small}
At this point, we desire to introduce strong convexity parameter $m$ and condition number $\kappa$ in our bound. In particular, to bound term $A$, we observe that $\|Q_{U} Q_{U}^\top\gradf(\X)\|_2 \leq \tfrac{m \sigma_r(\X)}{40\tau(\Uo_r)}$ or $\|Q_{U} Q_{U}^\top\gradf(\X)\|_2 \geq \tfrac{m \sigma_r(\X)}{40 \tau(\Uo_r)}$.
This results into bounding $A$ as follows:
\begin{small}
\begin{align*}
M\|X\|_2 &\|Q_{U} Q_{U}^\top\gradf(\X)\|_2 \| \Delta\|_F^2 \\ &\leq \max \left\{\tfrac{16\cdot \weta \cdot M \|X\|_2 \cdot m \sigma_r(\X)}{40\tau(\Uo_r)} \cdot \| \Delta\|_F^2 ,~ \weta \cdot 16 \cdot 40\tau(\Uo_r) \kappa \tau(\X) \|Q_{U} Q_{U}^\top\gradf(\X)\|_2^2 \cdot  \| \Delta\|_F^2 \right\} \\
																				 &\leq \tfrac{16\cdot \weta \cdot M \|X\|_2 \cdot m \sigma_r(\X)}{40\tau(\Uo_r)} \cdot \| \Delta\|_F^2 + \weta \cdot 16 \cdot 40 \kappa \tau(\X)\tau(\Uo_r) \|Q_{U} Q_{U}^\top\gradf(\X)\|_2^2 \cdot  \| \Delta\|_F^2.
\end{align*}
\end{small}
Combining the above inequalities, we obtain:
\begin{align}
\|Q_{U} Q_{U}^\top\gradf(\X)\|_2 &\| \Delta\|_F^2\stackrel{(i)}{\leq}  \tfrac{m \sigma_{r}(\X)}{40\tau(\Uo_r)} \cdot \| \Delta\|_F^2 + (40 \kappa \tau(\X)\tau(\Uo_r)+1) \cdot  16 \cdot \weta \|Q_{U} Q_{U}^\top\gradf(\X)\|_2^2 \cdot \|\Delta\|_F^2  \nonumber \\
											   &\stackrel{(ii)}{\leq}  \tfrac{m \sigma_{r}(\X)}{40\tau(\Uo_r)} \cdot \| \Delta\|_F^2 + (41 \kappa \tau(\Xo_r)\tau(\Uo_r)+1) \cdot  16 \cdot \weta \|Q_{U} Q_{U}^\top\gradf(\X)\|_2^2 \cdot (\rho')^2 \sigma_{r}(\Xo_r)  \nonumber \\
 											   &\stackrel{(iii)}{\leq}  \tfrac{m \sigma_{r}(\X)}{40\tau(\Uo_r)} \cdot \| \Delta\|_F^2 + 16 \cdot 42 \cdot \weta \cdot \kappa \tau(\Xo_r)\tau(\Uo_r) \cdot \|\gradf(\X)U\|_F^2 \cdot \tfrac{11(\rho')^2}{10}   \nonumber \\ 
 											   &\stackrel{(iv)}{\leq} \tfrac{m \sigma_{r}(\X)}{40\tau(\Uo_r)} \cdot \| \Delta\|_F^2 + \tfrac{2\weta}{25\tau(\Uo_r)} \cdot \|\gradf(\X)U\|_F^2,   \label{proofsr1:neweq_12}
\end{align}
where $(i)$ follows from $\weta \leq \tfrac{1}{16 M \|\X\|_2}$, $(ii)$ is due to Lemma~\ref{lem:sigma_bounds} and bounding $\| \Delta\|_F \leq \rho' \sigma_{r}(\Uo_r)$ by the hypothesis of the lemma, $(iii)$ is due to $\sigma_{r}(\Xo) \leq 1.1 \sigma_{r}(\X)$ by Lemma~\ref{lem:sigma_bounds}, $\sigma_{r}(\X)\|Q_{U} Q_{U}^\top\gradf(X)\|_2^2 \leq  \|U^\top\gradf(X)\|_F^2$ and $(41 \kappa \tau(\Xo_r)+1) \leq 42 \kappa \tau(\Xo_r)$. Finally, $(iv)$ follows from substituting $\rho'$ and using Lemma~\ref{lem:sigma_bounds}.

From Lemma \ref{lem:DD_bound_sc}, we also have the following bound:
\begin{align}
\|Q_{\Uor} Q_{\Uor}^\top\gradf(X)\|_2 \leq \frac{102 \cdot 101\tau(\Uo_r)}{99 \cdot 100 } \|Q_{U} Q_{U}^\top\gradf(\X)\|_2.
\label{proofsr1:neweq_13}
\end{align}
This follows from equation~\eqref{proofsjc:eq_067}. Then, the proof completes when we combine the above two inequalities to obtain:
\begin{align*}
\ip{\gradf(X) }{ \Delta \Delta^\top} \geq - \left(\tfrac{6\weta}{25 } \|\gradf(\X) U \|_F^2 + \tfrac{3m\sigma_{r}(\Xo)}{40} \cdot \| \Delta\|_F^2  \right)
\end{align*}

\end{document}